\documentclass[twoside]{article}
\RequirePackage{etex}

\usepackage[accepted]{aistats2024}

\makeatletter

\usepackage[round]{natbib}

\usepackage{float}
\usepackage{ifthen}
\usepackage{amsbsy}
\usepackage{amssymb}
\usepackage{amsmath}
\usepackage{amsthm}
\usepackage{bbm}
\usepackage{graphicx}
\usepackage{algorithm}
\usepackage{algpseudocode}
\usepackage{setspace}
\usepackage{esint}
\usepackage{comment}
\usepackage{mathtools}
\usepackage{xcolor}
\let\Pr\undefined

\DeclareMathOperator{\Pr}{Pr}

\DeclareMathOperator*{\argmin}{arg\,min} %

\def\ddefloop#1{\ifx\ddefloop#1\else\ddef{#1}\expandafter\ddefloop\fi}
\def\ddef#1{\expandafter\def\csname bb#1\endcsname{\ensuremath{\mathbb{#1}}}}
\ddefloop ABCDEFGHIJKLMNOPQRSTUVWXYZ\ddefloop
\def\ddefloop#1{\ifx\ddefloop#1\else\ddef{#1}\expandafter\ddefloop\fi}
\def\ddef#1{\expandafter\def\csname b#1\endcsname{\ensuremath{\mathbf{#1}}}}
\ddefloop ABCDEFGHIJKLMNOPQRSTUVWXYZ\ddefloop
\def\ddef#1{\expandafter\def\csname c#1\endcsname{\ensuremath{\mathcal{#1}}}}
\ddefloop ABCDEFGHIJKLMNOPQRSTUVWXYZ\ddefloop
\def\ddef#1{\expandafter\def\csname h#1\endcsname{\ensuremath{\widehat{#1}}}}
\ddefloop ABCDEFGHIJKLMNOPQRSTUVWXYZ\ddefloop
\def\ddef#1{\expandafter\def\csname hc#1\endcsname{\ensuremath{\widehat{\mathcal{#1}}}}}
\ddefloop ABCDEFGHIJKLMNOPQRSTUVWXYZ\ddefloop
\def\ddef#1{\expandafter\def\csname t#1\endcsname{\ensuremath{\widetilde{#1}}}}
\ddefloop ABCDEFGHIJKLMNOPQRSTUVWXYZ\ddefloop
\def\ddef#1{\expandafter\def\csname tc#1\endcsname{\ensuremath{\widetilde{\mathcal{#1}}}}}
\ddefloop ABCDEFGHIJKLMNOPQRSTUVWXYZ\ddefloop

\usepackage{prettyref}

\usepackage{nicefrac}
\usepackage{subcaption}
\usepackage{booktabs}

\newcommand{\ci}{\perp \!\!\! \perp}

\usepackage{color-edits}
\addauthor{vs}{blue}
\addauthor{hl}{red}

\usepackage{makecell}

\newcommand{\shull}{\ensuremath{\text{star}}}

\makeatletter
\newtheorem*{rep@theorem}{\rep@title}
\newcommand{\newreptheorem}[2]{%
\newenvironment{rep#1}[1]{%
 \def\rep@title{#2 \ref{##1}}%
 \begin{rep@theorem}}%
 {\end{rep@theorem}}}
\makeatother

\usepackage{url}
\usepackage[breaklinks]{hyperref}
\usepackage{cleveref}
\newtheorem{theorem}{Theorem}

\newreptheorem{theorem}{Theorem}
\newtheorem{corollary}[theorem]{Corollary}
\newreptheorem{corollary}{Corollary}
\newtheorem{lemma}[theorem]{Lemma}

\newtheorem{remark}{Remark}

\crefname{equation}{}{}
\crefname{proposition}{Proposition}{Propositions}
\crefname{appendix}{Appendix}{Appendices}
\crefname{assumption}{Assumption}{Assumptions}

\usepackage{autonum}

\newcommand{\kibitz}[2]{\ifnum\Comments=1{\color{#1}{#2}}\fi}

\newcommand\norm[1]{\left\lVert#1\right\rVert}
\newcommand{\E}{\mathbb{E}}

\newcommand{\R}{\mathbb{R}}

\renewcommand{\Pr}{\ensuremath{\mathrm{Pr}}}

\newcommand{\bias}{\ensuremath{\textsc{Bias}}}
\newcommand{\calS}{\ensuremath{{\cal S}}}
\newcommand{\Hcal}{\ensuremath{{\cal H}}}

\newcommand{\mcF}{{\mathcal F}}

\newcommand{\ba}{\begin{array}}
\newcommand{\ea}{\end{array}}
\newcommand{\bs}{\begin{align}\begin{split}\nonumber}
\newcommand{\bsnumber}{\begin{align}\begin{split}}
\newcommand{\es}{\end{split}\end{align}}

\newtheorem{assumption}{ASSUMPTION}

\newcommand{\co}{\ensuremath{\mathrm{co}}}
\newcommand{\error}{\ensuremath{\textsc{Error}}}
\newcommand{\calF}{\ensuremath{{\cal F}}}

\def\balign#1\ealign{\begin{align}#1\end{align}}
\def\balignat#1\ealign{\begin{alignat}#1\end{alignat}}
\def\bitemize#1\eitemize{\begin{itemize}#1\end{itemize}}
\def\benumerate#1\eenumerate{\begin{enumerate}#1\end{enumerate}}
\newenvironment{talign}
 {\csname align\endcsname}
 {\endalign}
\def\balignt#1\ealignt{\begin{talign}#1\end{talign}}%

\begin{document}
\twocolumn[

\aistatstitle{Causal Q-Aggregation for CATE Model Selection}

\aistatsauthor{ Hui Lan \And Vasilis Syrgkanis* }

\aistatsaddress{
Stanford University\\
\texttt{huilan@stanford.edu} 
\And 
Stanford University\\
\texttt{vsyrgk@stanford.edu}} ]

\begin{abstract}
Accurate estimation of conditional average treatment effects (CATE) is at the core of personalized decision making. While there is a plethora of models for CATE estimation, model selection is a nontrivial task, due to the fundamental problem of causal inference. Recent empirical work provides evidence in favor of proxy loss metrics with double robust properties and in favor of model ensembling. However, theoretical understanding is lacking. Direct application of prior theoretical work leads to suboptimal oracle model selection rates due to the non-convexity of the model selection problem. We provide regret rates for the major existing CATE ensembling approaches and propose a new CATE model ensembling approach based on Q-aggregation using the doubly robust loss. Our main result shows that causal Q-aggregation achieves statistically optimal oracle model selection regret rates of $\frac{\log(M)}{n}$ (with $M$ models and $n$ samples), with the addition of higher-order estimation error terms related to products of errors in the nuisance functions. Crucially, our regret rate does not require that any of the candidate CATE models be close to the truth. We validate our new method on many semi-synthetic datasets and also provide extensions of our work to CATE model selection with instrumental variables and unobserved confounding.
\end{abstract}

\section{Introduction}
Identifying optimal decisions requires understanding the causal effect of an action on an outcome of interest. With the emergence of rich and large datasets in many application domains such as digital experimentation, precision medicine, and digital marketing, identifying optimal personalized decisions has emerged as a mainstream topic in the literature. Identifying optimal personalized decisions requires understanding how the causal effect changes with observable characteristics of the treated unit. For this reason, many recent works have studied the estimation of conditional average treatment effects (CATE):
\begin{align}
    \tau_0(X) := \E[Y(1) - Y(0)\mid X]
\end{align}
where $X$ are observable features and $Y(d)$ is the potential outcome of interest under treatment $d\in \{0,1\}$.

This has led to a surge of many different methods for CATE estimation using machine learning techniques, such as deep learning \citep{shalit2017estimating, shi2019adapting}, lasso \citep{lasso}, random forests \citep{wager2018estimation,oprescu2019orthogonal}, Bayesian regression trees \citep{hahn2020bayesian}, as well as model-agnostic frameworks such as meta-learners \citep{kunzel2019metalearners} and double machine learning \citep{kennedy2020towards, foster2023orthogonal,nie2021quasi}. Machine learning CATE estimation has been considered both under the assumption of unconfoundedness, i.e., that outcomes $Y(1),Y(0)$ are independent of the assigned treatment $D$, conditional on the observed features $X$, as well as when there is unobserved confounding and access to instrumental variables \citep{IV,Tripadvisor}.

However, the effectiveness of each approach depends on the learnability of the various causal mechanisms at play and the structure of the CATE function. 
This has led to the pursuit of automated and data-driven model selection approaches for CATE estimation \citep{schuler2018comparison,mahajan2022empirical,curth2023search,alaa2019validating}. Due to the fundamental problem of causal inference, i.e., that we do not observe both counterfactual outcomes $Y(1),Y(0)$ for each unit, there does not exist a perfect analogue of out-of-sample model selection based on mean squared error, as is typically invoked in regression problems. For this reason, many works have proposed proxy loss metrics \cite{nie2021quasi,foster2023orthogonal,kennedy2020towards,alaa2019validating}, analogous to the mean square error, which, under assumptions, can be used for data-driven model evaluation. In particular, \cite{mahajan2022empirical}, highlights that proxy metrics that incorporate double robustness properties (that we either have a good estimate of the outcome regression function $\E[Y\mid D,X]$ or the propensity function $\E[D\mid X]$), tend to perform well in many practical scenarios. 

Moreover, ensembling of CATE models tends to typically increase performance, as compared to simply choosing the single best model based on the loss. The empirical success of model aggregation in causal inference has also been demonstrated multiple times. In the 2016 Atlantic Causal Inference Conference Competition, 3 out of the 5 best performing models to estimate the average treatment effect were ensemble models \citep{dorie2019automated}. Moreover, \cite{athey2019ensemble} also showed that ensemble methods can outperform single models at predicting counterfactuals. 

\paragraph{Contributions and prior work.} Motivated by these empirical findings and by the success of ensemble methods in the regression and supervised machine learning literature, we consider the problem of CATE model selection: given a set of $M$ candidate CATE models $\cF=\{f_1,\ldots,f_M\}$ and $n$ samples, can we construct an ensemble $\hat{f}=\sum_j \theta_j f_j$ that competes with the best model:
\begin{align}
    \E\left[(f_\theta(X) - \tau_0(X))^2\right] \leq \min_{j=1}^M \E\left[(f_j(X)-\tau_0(X))^2\right] + \epsilon
\end{align}
No rigorous theoretical treatment of this CATE model selection problem from observational data has been provided in the literature. Prior theoretical work considered only randomized trial data \cite{han2022ensemble}, or considered only CATE estimation over rich function spaces \cite{nie2021quasi,foster2023orthogonal}, which is technically very different from the goal of competing with the best of $M$ given models, or consider only guarantees for model evaluation \cite{alaa2019validating}, but not selection. For the case of regression problems, the optimal oracle selection error $\epsilon$ that is achievable is of the order of $\log(M)/n$. This rate allows selecting from an exponentially sized set of candidate models $M$ and decays fast with the sample size, as $n^{-1}$ instead of $n^{-1/2}$. 

Our main result is such an optimal oracle selection error for CATE model selection. We propose a novel CATE model ensembling method based on Q-aggregation \cite{Rigollet2014} with a doubly robust loss and prove that it attains the oracle selection error of $\log(M)/n$ plus second-order terms that depend on the regression and the propensity estimation errors. Moreover, even the $\log(M)$ term can be removed if one has a relatively good prior over which models will be the better performers, which can be quite natural in the context of CATE model selection.

In the regression setting, optimal selection error rates cannot be achieved using ERM approaches, i.e., by minimizing squared loss over some ensemble hypothesis space \cite{Rigollet2014}. Thus, prior approaches to CATE ensembles \cite{nie2021quasi,han2022ensemble} that either select the best model based on a proxy loss, or optimize over the space of all possible convex combinations of models, will necessarily achieve suboptimal selection error rates (either $\sqrt{\log(M)/n}$ or $M/n$ in the worst case). As a set of side results, we provide oracle selection rates for both of these methods for the CATE problem (which are novel in the literature and require a generalization of the main statistical learning theorems with nuisance functions presented in \cite{foster2023orthogonal}). These results also highlight when one should expect ERM approaches to achieve close to optimal error. 

Moreover, a theoretical result on optimal oracle model selection necessarily needs to use techniques beyond what is outlined in the prior theoretical work of \cite{nie2021quasi,foster2023orthogonal,kennedy2020towards}, which is our main technical contribution. Notably, the general theorems of \cite{foster2023orthogonal} require a first-order optimality condition, which is not satisfied due to the non-convexity of the function space $\mcF$ with which we are trying to compete and cannot be applied to prove our main Q-aggregation theorem. Due to space constraints, we defer a detailed comparison to previous work in the Appendix~\ref{app:related-work}.

Finally, we provide an extensive empirical evaluation of our causal Q-aggregation ensemble method and of the ERM based ensemble approaches, on a broad spectrum of synthetic data as well as semi-synthetic data based on five real-world datasets from a multitude of application domains, such as political science \citep{welfare}, economics \citep{poverty,401k_1,401k_2}, education \cite{STAR} and digital advertising \citep{criteo}. We showcase how ensemble methods offer more robust solutions than any single meta-learning approach proposed in the literature and how Q-aggregation can many times lead to improvements in performance than ERM based approaches.

\section{Preliminaries}

We consider a setting where we are given a set of $M$ candidate CATE models $\{f_1, \ldots, f_M\}$ and a data set of $n$ samples. Each candidate model $f_i$ maps covariates $X$ to a CATE prediction and can be thought of as a CATE model that was fitted on a separate set of samples (e.g. fitted on training set and correspond to different hyperparameters, methods, and model spaces), and the $n$ samples that we are using are the validation set. Our goal is to select a model or an ensemble of models, i.e. $f_{\theta}=\sum_{j=1}^M \theta_j f_j$, that achieves mean-squared-error with respect to the true CATE function $\tau_0(x)=\E[Y(1)-Y(0)\mid X=x]$ that is comparable to the best model, i.e.:
\begin{align}
    \E\left[(f_\theta(X) - \tau_0(X))^2\right] \leq \min_{j=1}^M \E\left[(f_j(X)-\tau_0(X))^2\right] + \epsilon
\end{align}

As we will see in Section~\ref{sec:dr-loss}, the true CATE model $\tau_0$, can be expressed as the minimizer of a risk $\min_{f} \E[\ell(Z;f(X), g_0)]$, which, in addition to the CATE function $f$ also depends on other nuisance functions $g_0$, which are the solution to auxiliary statistical estimation problems and which also need to be estimated from data. Moreover, the loss function varies depending on the argument that is used to identify the CATE (e.g. conditional exogeneity, instrumental variables, etc.). 

For this reason, we will consider the more general problem of model selection in the presence of nuisance functions and present our main theorems at this level of generality. Subsequently in Section~\ref{sec:dr-loss} and Appendix~\ref{app:iv-cate}, we will instantiate the general theorems for particular CATE selection problems. 

\paragraph{Model Selection with Nuisance Functions} 
Given a set of candidate functions $\calF=\{f_1, \ldots, f_M\}$, we care about minimizing an expected loss function $\E[\ell(Z;f(X), g_0)]$, where $g_0$ is the ground truth nuisance function. In particular, our goal is to identify an ensemble of the candidate estimators, i.e. 
\begin{align}
f_{\theta}:=~& \sum_{j=1}^M \theta_j f_j, &
\theta\in \Theta :=~& \left\{\theta \in \R_{\geq 0}^M: \sum_{j=1}^M \theta_j=1\right\}
\end{align}
that controls the oracle selection error:
\begin{align}
    \calS(\theta) :=~& \E[\ell(Z;f_{\theta}(X), g_0)] - \min_{j=1}^{M} \E[\ell(Z;f_j(X), g_0)]
\end{align}
Our target is to prove an oracle inequality of the order:
$O\left(\frac{\log(M)}{n}\right)$, 
with access to $n$ samples of the variables $(X,Z)$, which is the optimal rate in the absence of nuisance functions.
Since, we also do not know $g_0$, we need to construct an estimate $\hat{g}$ from the data. Our goal is to derive an error that would have a second order dependence on the estimation error of $g_0$
\begin{align}
    \calS(\theta) \leq  O\left(\frac{\log(M)}{n}\right) + \error(\hat{g})
\end{align}
for some appropriately defined error function that would decay at least as fast or faster than $1/n$ under assumptions on the estimation error of $\hat{g}$. \emph{Unless otherwise stated, the estimate $\hat{g}$ will be assumed to be trained on a separate sample of size $n$ (e.g. the training sample used to fit the models $f_i$). }

\paragraph{Notation}
We use short-hand notation:
\begin{align}
P \ell(Z; f(X), g) =~& \E[\ell(Z;f(X), g)]\\
P_n\ell(Z;f(X), g) =~& \frac{1}{n} \sum_i \ell(Z_i;f(X_i), g),
\end{align}
where $Z$ denotes all the random variables and $X$ is some subset of them. 
Moreover, we denote with $\|f-f'\|^2 = P(f'(X) - f(X))^2$ and with $\|X\|_{L^p}=\E[X^p]^{1/p}$.
For any convex combination $\theta$ over a set of candidate functions $\{f_1, \ldots, f_M\}$, we define the notation:
\begin{align}
    \ell_{\theta,g}(Z) :=~& \ell(Z; f_{\theta}(X), g) &
    R(\theta,g) :=~& P\ell_{\theta,g}(Z)
\end{align}

We also denote with $e_j\in \Theta$ the vector that has $1$ at coordinate $j$ and zero otherwise (that is, $f_{e_j}=f_j$) and denote with $\co(\calF)$ the convex hull of space $\calF$. 

\section{ERM-Based Selection} \label{sec:ols_results}

We start our investigation by exploring what regret rates are achievable by ensembling or stacking approaches based on ERM, i.e. minimizing the empirical risk over the weight parameters $\theta\in \Theta$, for some constraint space $\Theta$. 
To simplify exposition, we restrict attention in this section to square losses:
\begin{align}\label{eqn:square-loss}
    \ell(Z; f(X), g) = (\hat{Y}(g) - f(X))^2
\end{align}
where $g$ is a vector of nuisance functions and $\hat{Y}(g)$ is some functional of these functions that satisfies:
\begin{align}\label{eqn:target-param}
    \tau_0(X) := \E[\hat{Y}(g_0)\mid X]
\end{align}
for some target function $\tau_0$.

Two major ensemble methods are choosing the single best model, ($\Theta_b=\{e_1, \ldots, e_M\}$), and choosing the best convex combination, 
\begin{align}
    \hat{\theta}_{b} :=~& \argmin_{\theta\in \Theta_b} \E_n[(\hat{Y}(\hat{g}) - f_\theta)^2] \tag{Best-ERM}\\
    \hat{\theta}_{c} :=~& \argmin_{\theta \in \Theta} \E_n[(\hat{Y}(\hat{g}) - f_{\theta})^2] \tag{Convex-ERM}
\end{align}

It is known that such ERM approaches lead to suboptimal regret rates without further assumptions, even in the absence of nuisance functions (i.e., when we know $g_0$) \cite{Rigollet2014}. As a side result of this paper, we provide oracle results for these estimators in the presence of nuisance functions and refine prior results in the regression literature, providing some insights as to when we should expect these methods to achieve near optimal performance in practice. 

\begin{theorem}[Gurantees for Best-ERM] \label{thm:best_erm}
Consider the case of the square loss and let:
\begin{align}
    \bias(X; \hat{g}) = \E\left[\hat{Y}(\hat{g}) - \hat{Y}(g_0)\mid X\right]
\end{align}
Assume that $\hat{Y}(g)$ is absolutely uniformly bounded by $U$ for all $g\in G$ and all functions $f_j$ are absolutely uniformly bounded by $U$. Let 
\begin{align}
\bar{f}_* :=& \argmin_{f\in \co(\calF)} \|f - \tau_0\|, ~~&~~ f_* :=& \argmin_{f\in \calF} \|f - \tau_0\|
\end{align}{ and let:
\begin{align}
    \epsilon = ~& \|f_*-\tau_0\|^2 - \|\bar{f}_*-\tau_0\|^2
\end{align}
}
Then with probability $1-\delta$:
\begin{align}\label{eqn:bound-best}
    \calS(\hat{\theta}_b) \leq~& 
     O\left(\epsilon + \frac{\log((M\vee n)/\delta)}{n} + \E[\bias(X;\hat{g})^2]\right)
\end{align}
\end{theorem}
The result is enabled by a generalization of the main theorems in \cite{foster2023orthogonal}. Incorporating the bound $ \|f_*-\tau_0\|^2 - \|\bar{f}_*-\tau_0\|^2$ is a novel addition to the literature, even in the absence of nuisance functions. We see that the best selector achieves the target rate only when $\epsilon$ vanishes faster than $\log(M)/n$, which would happen only if the best model in $\calF$ is very close to the true model $\tau_0$ (c.f. $\min_{f\in \calF}\|f-\tau_0\|^2\leq \epsilon$), which is quite unlikely when $\tau_0$ is complex, or when the best model $\bar{f}_*$ in the convex hull ends up putting most of the support on only one model, i.e., there is a dominant model (c.f. $\min_{f\in\calF} \|f-\bar{f}_*\|\leq \epsilon$). Typical prior bounds depend on the larger quantity $\epsilon=\|f_*-\tau_0\|^2$, which is not close to zero if the optimizer in the convex hull is primarily supported on one model. 

Theorem~\ref{thm:best_erm} is also closely related to the results in \cite{van2003unified}. As we show in Appendix~\ref{app:vanderlaan}, our Theorem~\ref{thm:best_erm} improves upon this prior work in: i) only incurring a dependence on the nuisance error through the term $\E[\bias(X;\hat{g})^2]$, which possesses doubly robust properties in our main application of CATE model selection, unlike the bound in \cite{van2003unified}, which would also depend on a term of the form $\E[(\hat{Y}(\hat{g})-\hat{Y}(g_0))^2]$, which does not possess doubly robust properties , ii) only depending on the smaller quantity $\epsilon$, unlike the bound in \cite{van2003unified} which would depend on $\|f_*-\tau_0\|^2$; as we discussed $\epsilon$ can be much smaller than $\|f_*-\tau_0\|^2$, when there is a dominant model, but all models are far from $\tau_0$.
\begin{theorem}[Gurantees for Convex-ERM] \label{thm:convex_erm}
Consider the case of the square loss and assume that $\hat{Y}(g)$ is absolutely uniformly bounded by $U$ for all $g\in G$ and all functions $f_j$ are absolutely uniformly bounded by $U$. Then with probability $1-\delta$:
\begin{align}\label{eqn:bound-convex}
    \calS(\hat{\theta}_c) \leq~& 
     O\left(\frac{M\log(n/\delta)}{n} + \E[\bias(X; \hat{g})^2]\right)\\
     \calS(\hat{\theta}_c) \leq~& 
     O\left(\sqrt{\frac{\log(M/\delta)}{n}} + \E[|\bias(X; \hat{g})|]\right)
\end{align}
\end{theorem}

As we show in the appendix, the benchmark in the oracle selection guarantee in both the above theorems can be strengthened to be the best model in the convex hull of $\calF$ and not just the best in $\calF$. Moreover, we note that the slow rate result of $\sqrt{\log(M/\delta)/n}$ is closely related to the work of \cite{han2022ensemble}, who provided the exact same rate for convex-ERM ensembling using the doubly robust loss that we will introduce in Section~\ref{sec:dr-loss}, albeit only for the case of randomized trials, where we have $\E[|\bias(X;\hat{g})|]=0$, due to knowing the propensity function $\Pr(D=1\mid X)$. Thus this result is a direct generalization of their result to observational data. 
None of these aforementioned theorems achieves the optimal rate when we want to compete with the best in $\calF$. The inability of ERM based approaches to achieve the optimal oracle selection rates motivates the main result of this paper in the next section: the use of ensemble approaches based on Q-aggregation.

\section{Q-Aggregation with Nuisances}

We assume that the function $\ell$ is $\sigma$-strongly convex in $f$ in expectation, i.e., with $\ell^{(2)}$ being the derivative of $\ell$ with respect to $f(X)$:
\begin{multline}
    P\left(\ell(Z;f'(X), g_0) - \ell(Z;f(X), g_0)\right) \geq \\
    P \ell^{(2)}(Z; f(X), g_0)\, (f'(X)-f(X)) 
    + \frac{\sigma}{2} \|f-f'\|^2
\end{multline}
Consider a modified loss, which penalizes the weight of each model, based on individual model performance:
\begin{align}
    \tilde{\ell}_{\theta,g}(Z) =~& (1-\nu) \ell_{\theta,g}(Z) + \nu \sum_{j=1}^M \theta_j \ell_{e_j,g}(Z)
\end{align}
For any prior weights $\{\pi_1, \ldots, \pi_M\}$, the Q-aggregation ensemble is the solution to the convex problem:
\begin{align}
    \hat{\theta} =~& \arg\min_{\theta\in \Theta} P_n\tilde{\ell}_{\theta, \hat{g}}(Z) + \frac{\beta}{n} \sum_{j=1}^M \theta_j \log\left(\frac{1}{\pi_j}\right) \tag{Q-agg} 
\end{align}
\begin{theorem}[Main Theorem]\label{thm:main}
Assume that for some function $\error(\hat{g})$:
\begin{multline}
    P(\tilde{\ell}_{\theta,g_0}(Z) - \tilde{\ell}_{\theta',g_0}(Z)) - P(\tilde{\ell}_{\theta,\hat{g}}(Z) - \tilde{\ell}_{\theta',\hat{g}}(Z)) \\
    \leq \error(\hat{g})\, \|f_{\theta} - f_{\theta'}\|^{\gamma}
\end{multline}
for some $\gamma\in [0, 1]$ and $\ell(f, g_0)$ is $\sigma$-strongly convex in expectation with respect to $f$. Moreover, assume that the loss $\ell(Z; f(X), \hat{g})$ is $C_b(\hat{g})$-Lipszhitz with respect to $f(X)$, and both the loss and the candidate functions are uniformly bounded in $[-b,b]$ almost surely. For any $\mu \leq \min\{\nu, 1-\nu\} \frac{\sigma}{14}$  and $\beta \geq \{ \frac{8C_b^2(1-\nu)^2}{\mu}, 8\sqrt{3}bC_b(1-\nu),\frac{2C_b\nu(\nu C_b + \mu b)}{\mu}\}$, w.p. $1-\delta$:
\begin{align}
    R(\hat{\theta}, g_0) \leq~& \min_{j=1}^M\left( R(e_j, g_0)  +\frac{\beta}{n}\log(1/\pi_j)\right) \\
    ~&~ + O\left(  \frac{\log(1/\delta) \max\{\frac{1}{\mu} C_b(\hat{g})^2, b\, C_b(\hat{g}) \}}{n}\right) \\
    ~&+ \left(\frac{1}{\mu}\right)^{\frac{\gamma}{2-\gamma}} \error(\hat{g})^{\frac{2}{2-\gamma}}
\end{align}
\end{theorem}

The theorem gives a more refined version of oracle selection that is adaptive to prior beliefs on which model is best and can even save the $\log(M)$ term, if the prior beliefs end up being correct. In its simplest form, when prior weights are uniform, then $\beta$ is irrelevant in the optimization problem and can be taken to satisfy the condition of the theorem, as part of the analysis, i.e. the method does not have any hyperparameter $\beta$ to choose. Moreover, we can always take $\nu=1/2$, in which case the ensemble has the oracle selection error:
\begin{align}
    S(\hat{\theta}) \leq~& O\left(\frac{1}{\sigma}\frac{\log(M/\delta)}{n} + \left(\frac{1}{\sigma}\right)^{\frac{\gamma}{2-\gamma}}\error(\hat{g})^{\frac{2}{2-\gamma}}\right)
\end{align}

\begin{remark}[Computational Considerations]
The computation of the weights $\theta$ that optimize the Q-aggregation objective corresponds to a convex optimization problem, which can be solved efficiently by modern convex optimization solvers. Drawing from the results in \cite{greedy}, we also show in Appendix~\ref{app:greedy} that a greedy approximate solution to the Q-aggregation problem provides the same statistical guarantees as in Theorem~\ref{thm:main}. The greedy approximation requires only finding the best single model to the Q objective, as well as the best convex combination between that best single model and any other model (see Algorithm~\ref{alg:greedy}). This will always return a convex combination of at most two baseline models and can be implemented by a simple line search over the scalar that represents the convex combination of the two models, as well as the procedure that finds the minimum of the objective over the $M$ models (which is linear in $M$).
\end{remark}

\subsection{Square Losses with Nuisance Functions}\label{sec:square-loss}
    
We specialize the main theorem to the case of square losses, where the target label depends on unknown and estimated nuisance functions as defined in Equation~\eqref{eqn:square-loss}, with target parameter $\tau_0$ as defined in Equation~\eqref{eqn:target-param}. For such loss functions, we can instantiate the main theorem to obtain the following corollary.

\begin{corollary}[Square Losses with Nuisances]\label{cor:main}
Let:
\begin{align}
    \error(\hat{g}) = 2\sqrt{\E\left[\bias(X; \hat{g})^2\right]}
\end{align}
Assume that $\hat{Y}(\hat{g})$ and all functions $f_j$ are absolutely uniformly bounded by $U$. Then the Q-aggregation ensemble $\hat{\theta}$ with $\nu=1/2$, $\beta\geq\max\{112\,U^2, 56 U^3\}$ and loss $(\hat{Y}(\hat{g})-f(X))^2$, satisfies w.p. $1-\delta$:
\begin{multline}\label{eqn:bound-main}
    \|f_{\hat{\theta}} - \tau_0\|^2 \leq \min_{j=1}^M \|f_j - \tau_0\|^2 + \frac{\beta}{n} \log(1/\pi_j) \\
    + O\left(  \frac{\log(1/\delta)}{n} + \error(\hat{g})^2\right)
\end{multline}
\end{corollary}

If we impose more restrictions on the functions in $\calF$, then we can  show a corollary, with weaker requirements on the nuisance functions:
\begin{corollary}\label{cor:main2}
Assume that for all $\theta,\theta'\in \Theta$: 
\begin{align}
    \|f_{\theta}(X)-f_{\theta'}(X)\|_{L^\infty} \leq C \|f_{\theta}(X)-f_{\theta'}(X)\|_{L^2}^{\gamma}
\end{align}
for some constants $C$ and $0<\gamma\leq 1$. Moreover, let:
\begin{align}\label{eqn:error-2}
    \error(\hat{g}) = 2\, C\, \E\left[\left|\bias(X; \hat{g})\right|\right]
\end{align}
Then under the remainder conditions and definitions of Corollary~\ref{cor:main}, w.p. $1-\delta$
\begin{multline}\label{eqn:bound-main-2}
    \|f_{\hat{\theta}} - \tau_0\|^2 \leq \min_{j=1}^M \|f_j - \tau_0\|^2 + \frac{\beta}{n} \log(1/\pi_j) \\
    + O\left(  \frac{\log(1/\delta)}{n} + \error(\hat{g})^{\frac{2}{2-\gamma}}\right)
\end{multline}
\end{corollary}

If the function class $\calF$ contains smooth enough functions, then the above condition tends to hold. For instance, \cite[Lemma 5.1]{mendelson2010regularization} shows that if the functions $f_j\in F$ lie in an RKHS with a polynomial eigendecay, then the above condition holds for some constant $\gamma$ that depends on the rate of eigendecay. 
Moreover, as we show in Appendix~\ref{app:eigenvalue}, even when $\calF$ does not contain smooth functions, we can still show the above property for $\gamma=1$ and $C \sim M$, as long as functions in $\calF$ are not co-linear. Here, the number of models $M$ and not the logarithm of $M$  appears. However, assuming that the nuisance error term decays faster than $n^{-1/2}$, this linear in $M$ term appears only in lower-order terms.

\paragraph{Guarantees without Sample Splitting}

When learning with nuisance functions, it is often desirable to train the nuisance functions on a separate sample, as it reduces overfitting bias. However, sample splitting decreases efficiency and statistical power, and, thus, it might not be beneficial for small datasets. The following theorem provides theoretical guarantees when the nuisance estimate and the ensemble weights are constructed using the same sample.

\begin{theorem}[Guarantees without Sample Splitting]\label{thm:no_split}
    Under the same conditions as in Corollary~\ref{cor:main}, with the exception that the nuisance function $\hat{g}$ is trained on the same sample as the parameter $\hat{\theta}$. The Q-aggregation ensemble $\hat{\theta}$ satisfies w.p. $1-\delta$:
\begin{multline}\label{eqn:bound-main-no-split}
    \|f_{\hat{\theta}} - \tau_0\|^2 \leq \min_{j} \|f_j - \tau_0\|^2 + \frac{\beta}{n} \log(1/\pi_j) \\
    + O\left( \frac{\log(1/\delta)}{n} + r_n^2 + \error(\hat{g})^2\right)
\end{multline}
where $\error(\hat{g})^2$ is of the order:
\begin{align}
    \E\left[\bias(X;\hat{g})^2\right] + \E[(\hat{Y}(\hat{g}) - \hat{Y}(g_0))^4] 
\end{align}
and $r_n$ denotes the critical radius of the function space $\operatorname{star}(Y_G)$ where $Y_G := \left\{ (\hat{Y}(g) - \hat{Y}(g_0))^2 \mid  g \in G\right\}$.
\end{theorem}

\subsection{Connections to Neyman Orthogonality}

In this section, we draw a connection between the bias in labels and the literature of Neyman orthogonality. When we are interested in estimating an average treatment effect, then we can typically phrase this as estimating a parameter of the form:
$\tau_0 = \E[\hat{Y}(g_0)]$, 
for some appropriately defined function $\hat{Y}(g)$. In such settings, we say that the moment function $M(g):=\E[\hat{Y}(g)]$ is Neyman orthogonal \citep{chernozhukov2018double} with respect to the nuisance functions $g$ if for all $g\in G$, the following directional derivative is zero:
\begin{align}
    \partial_{\epsilon} M(g_0 + \epsilon\, (g-g_0)) \mid_{\epsilon=0} = 0
\end{align}
Similarly, we can define a conditional analogue, when the quantity of interest is a conditional expectation: $
    \tau_0(X) = \E[\hat{Y}(g_0)\mid X]$.
We say that the conditional moment function $M_X(g)=\E[\hat{Y}(g_0)\mid X]$ is conditionally Neyman orthogonal if for all $g\in G$, a.s. over $X$:
\begin{align}
    \partial_{\epsilon} M_X(g_0 + \epsilon\, (g-g_0)) \mid_{\epsilon=0} = 0
\end{align}
When the label function $\hat{Y}(g)$ satisfies the latter property then by a second order Taylor expansion we have:
\begin{align}
   \bias(X; \hat{g})
    ~&=\int_{0}^1 \partial_{\epsilon}^2 M_X(g_0 + \epsilon (\hat{g}-g_0)) d\epsilon
\end{align}
Thus we can instantiate all the theorems in the aforementioned sections, by using the above bias term. 
Typically this second order remainder term will depend quadratically in the error of the nuisances, i.e. 
\begin{align}
\int_{0}^1 \partial_{\epsilon}^2 M_X(g_0 + \epsilon (\hat{g}-g_0)) d\epsilon\lesssim \E\left[(\hat{g}(Z) - g_0(Z))^2\mid X\right]
\end{align}
or will contain only products of errors of the different nuisance functions that compose $g$ (see next section). In both cases,
the $\error(\hat{g})^2$ term, will most times be of lower order than $\log(M)/n$ or $\min_j \|f_j - \tau_0\|_{L^2}^2$. The next section and Appendix~\ref{app:iv-cate}, draw on this connection and provide two Neyman orthogonal label functions for the case of model selection for CATE under unobserved confounding and with access to an instrument.

\section{Doubly Robust Q-Aggregation}\label{sec:dr-loss}

We now focus on our main question: model selection for CATE, i.e., $\tau_0(X)=\E[Y(1)-Y(0)\mid X]$. We focus on the setting where we observe all confounding variables and we defer to Appendix~\ref{app:iv-cate} the case of unobserved confounding with access to an instrumental variable (e.g. A/B testing with non-compliance). 

In the no unobserved confounding case, we observe a set of variables $W$ (that is, a superset of $X$), such that the conditional exogeneity property holds: $Y(1),Y(0) \ci D\mid W$. Under conditional exogeneity, the CATE is identified as:
\begin{align}
    \tau_0(X) 
    = \E[h_0(1, W) - h_0(0,W)\mid X]
\end{align}
where $h_0(D,W) := \E[Y\mid D, W]$. As is well known \cite{oprescu2019orthogonal,semenova2021debiased,foster2023orthogonal,kennedy2020towards}, this identifying formula can be robustified by incorporating propensity estimation: $\tau_0(X) = \E[\hat{Y}(g_0)\mid X]$
\begin{align}
\hat{Y}(g) = h(1, W) - h(0, W) + a(D, W)\, (Y - h(D, W))
\end{align}
with $g=(h,a)$, $h$ an estimate of $h_0$ and $a$ an estimate of the signed inverse propensity function $a_0$ (a.k.a. the Riesz representer in \cite{chernozhukov2022riesznet}):
\begin{align}
    a_0(D, W) :=& \frac{D-p_0(W)}{p_0(W)\, (1-p_0(W))}, & p_0(W) :=& \E[D\mid W]
\end{align}
This gives rise to the doubly robust square loss:
\begin{align}
    \ell(Z; f(X), g) = (\hat{Y}(g) - f(X))^2 \tag{DR-Loss}
\end{align}
For any estimate $\hat{g}=(\hat{h}, \hat{a})$, define $\hat{q}(D, W)$ as:
\begin{align}
     (a_0(D, W) - \hat{a}(D, W))\, (h_0(D, W) - \hat{h}(D,W))
\end{align}
An important fact of the doubly robust target, exploited also in prior works, is that it satisfies the mixed bias property (see Appendix~\ref{app:mixed-bias}):
\begin{align}
    \bias(X; \hat{g}) := \E\left[\hat{Y}(g_0) - \hat{Y}(\hat{g})\mid X\right] = \E\left[\hat{q}(D, W)\mid X\right]
\end{align}
Thus we can apply Corollary~\ref{cor:main} to obtain:
\begin{corollary}[Main Corollary]\label{cor:main-cate}
Let $\tau_0(X)=\E[Y(1)-Y(0)\mid X]$ and suppose that conditional exogeneity holds conditional on some observed $W$ that is a superset of $X$. Let:
\begin{align}
    \error(\hat{g}) = 2\sqrt{\E\left[\E\left[ \hat{q}(D, W)\mid X\right]^2\right]}
\end{align}
and assume that $\hat{Y}(\hat{g})$ all functions $f\in F$ are uniformly and absolutely bounded by $U$. Then the doubly robust Q-aggregation ensemble $\hat{\theta}$, based on the DR-Loss, with $\nu=1/2$ and a uniform prior $\pi$, satisfies that w.p. $1-\delta$:
\begin{multline}
    \|f_{\hat{\theta}} - \tau_0\|^2 \leq \min_{j} \|f_j - \tau_0\|^2 \\
    + O\left(  \frac{\log(M/\delta)}{n} + \error(\hat{g})^2\right)
\end{multline}
For any prior $\pi$, if $\beta\geq\max\{112\,U^2, 56 U^3\}$, w.p. $1-\delta$:
\begin{multline}
    \|f_{\hat{\theta}} - \tau_0\|^2 \leq \min_{j} \|f_j - \tau_0\|^2 + \frac{\beta}{n} \log(1/\pi_j) \\
    + O\left(  \frac{\log(1/\delta)}{n} + \error(\hat{g})^2\right)
\end{multline}
\end{corollary}
The error term can also be further upper bounded by an application of a Cauchy-Schwarz inequality by:
\begin{align}
    \error~&(\hat{g}) \leq \min\{\|\hat{a}-a_0\|_{L^4}\,\|\hat{g}-g_0\|_{L^4}, \\
    ~&\|\hat{a}-a_0\|_{L^2}\, \|\hat{g}-g_0\|_{L^\infty}, \|\hat{a}-a_0\|_{L^\infty}\, \|\hat{g}-g_0\|_{L^2}\}
\end{align}
Thus we can either control the $L^4$ moment of the prediction error of the two nuisances, or guarantee an $L^\infty$ rate for one and an $L^2$ rate for the other. Note that if the above product of error terms is of lower order than $n^{-1/2}$, then the effect of these errors on the quality of the causal ensemble is of second-order importance.

If we further make an assumption that:
\begin{align}
    \|f_{\theta*}(X)-f_{\hat{\theta}}(X)\|_{L^\infty} \leq C \|f_{\theta*}(X)-f_{\hat{\theta}}(X)\|_{L^2}^\gamma
\end{align}
Then by Corollary~\ref{cor:main}, we can take in Corollary~\ref{cor:main-cate}:
\begin{align}
    \error(\hat{g}) ~&\lesssim 
    C\,\|\hat{a}-a_0\|_{L^2}^{\frac{1}{2-\gamma}}\, \|\hat{h}-h_0\|_{L^2}^{\frac{1}{2-\gamma}}
\end{align}
As we saw in Section~\ref{sec:square-loss}, the latter assumption holds if either the functions in $\calF$ are smooth or they contain independent components in their predictions.

\begin{remark}[Priors on CATE models]
In the context of CATE estimation, incorporating priors is quite natural. For instance, we have a strong prior that models that come out of meta-learning approaches that only use outcome modeling should perform worse. Thus, we can put lower prior probability on S- and T- learner models and larger prior probability on X-, DR- and R- learner models. Moreover, we can incorporate priors that more complicated models are less probable than simpler models. For instance, we can put higher prior on linear cate models or shallow random forest models.
\end{remark}

\section{Doubly Robust Q-Aggregation for CATE with Instruments}\label{app:iv-cate}

We consider a widely encountered setting of estimating a CATE from a stratified randomized trial with non-compliance. In this setting we know the cohort assignment policy $\pi_0(X) = \Pr(Z=1\mid X)$, and we are interested in estimating the conditional local average treatment effect 
$$\tau_0(X) =  \E[Y(1)-Y(0)\mid D(1)>D(0), X]$$
where $D(z)$ is the potential outcome of the chosen treatment under different cohort assignments (aka recommended treatments).
This effect is identified as:
\begin{align}
    \tau_0(X) ~&= \frac{\E[Y \tilde{Z}\mid X]}{\E[D \tilde{Z}\mid X]} = \frac{\text{Cov}(Y,Z\mid X)}{\text{Cov}(D, Z\mid X)} \\
    ~&= \frac{\E[Y\mid Z=1, X] - \E[Y\mid Z=0, X]}{\E[D\mid Z=1, X] - \E[D\mid Z=0, X]}
\end{align}
where $\tilde{Z} = Z - \pi_0(X)$.

Let $\hat{\alpha}$ be an estimate of:
\begin{align}
    a_0(X) ~&:= \E[Y \tilde{Z}\mid X] \\
    ~&= \E[Y\mid Z=1, X] - \E[Y\mid Z=0, X]
\end{align}
and $\hat{\beta}$ an estimate of:
\begin{align}
    \beta_0(X) ~&:= \E[D \tilde{Z}\mid X] \\
    ~&= \E[D\mid Z=1, X] - \E[D\mid Z=0, X]
\end{align}
and let $\hat{\tau}=\hat{\alpha}/\hat{\beta}$. Then we can construct the random variable
\begin{align}
    \hat{Y}(\hat{g}) = \hat{\tau}(X) + \frac{(Y - \hat{\tau}(X) D) \tilde{Z}}{\hat{\beta}(X)}
\end{align}
and we can consider Q-aggregation with the loss:
\begin{align}
    \ell(Z;f(X), g) = (\hat{Y}(g) - f(X))^2
\end{align}

In this case, we can show the following mixed bias property (see Appendix~\ref{app:mixed-bias}):
\begin{align}
    \E[\hat{Y}(\hat{g}) - \hat{Y}(g_0)\mid X] 
    =~& (\tau_0(X) - \hat{\tau}(X)) \frac{\beta_0(X) - \hat{\beta}(X)}{\hat{\beta}(X)}
\end{align}

Thus applying Corollary~\ref{cor:main} we have:
\begin{corollary}[Main Corollary for CATE with Instruments]
Let $\tau_0(X)=\E[Y(1)-Y(0)\mid D(1) > D(0), X]$ and suppose that the instrument is a stratified randomized trial with non-compliance and with a known propensity model $\pi_0(X)=\Pr(Z=1\mid X)$. Let:
\begin{align}
    \error(\hat{g})^2 = 4 \E\left[(\tau_0(X) - \hat{\tau}(X))^2 \frac{(\beta_0(X) - \hat{\beta}(X))^2}{\hat{\beta}(X)^2}\right]
\end{align}
and assume that $\hat{Y}(\hat{g})$ is uniformly and absolutely bounded by $U$ in a sufficiently small neighborhood of $g_0$ and that all functions $f\in F$ are uniformly and absolutely bounded by $U$. Then the doubly robust IV Q-aggregation ensemble $\hat{\theta}$ with $\nu=1/2$ and a uniform prior $\pi$, satisfies that w.p. $1-\delta$:
\begin{multline}
    \|f_{\hat{\theta}} - \tau_0\|^2 \leq \min_{j} \|f_j - \tau_0\|^2\\
     + O\left(  \frac{\log(M) +\log(1/\delta)}{n} + \error(\hat{g})^2\right)
\end{multline}
Moreover, for any prior $\pi$, for $\beta\geq\max\{112\,U^2, 56 U^3\}$, w.p. $1-\delta$:
\begin{multline}
    \|f_{\hat{\theta}} - \tau_0\|^2 \leq \min_{j} \|f_j - \tau_0\|^2 + \frac{\beta}{n} \log(1/\pi_j)\\
     +  O\left(  \frac{\log(1/\delta)}{n} + \error(\hat{g})^2\right)
\end{multline}
\end{corollary}

Moreover, note that since $\hat{\tau} = \hat{\alpha}/\hat{\beta}$, we also have that 
$$\tau_0 -\hat{\tau} = \frac{\alpha_0 \hat{\beta} - \hat{\alpha} \beta_0}{\hat{\beta}\beta_0 }= \frac{(\alpha_0 - \hat{\alpha})\hat{\beta} + \hat{\alpha} (\hat{\beta}-\beta_0)}{ \hat{\beta}\beta_0}$$
Thus overall we have:
\begin{align}
    \error(\hat{g})^2 =~& 4 \E\left[\frac{(\alpha_0(X) - \hat{\alpha}(X))^2 (\beta_0(X) - \hat{\beta}(X))^2}{\beta_0(X)^2\hat{\beta}(X)^2}\right]\\
    ~&+ 4 \E\left[\frac{\hat{\alpha}(X)^2 (\beta_0(X) - \hat{\beta}(X))^4}{\beta_0(X)^2\hat{\beta}(X)^4}\right]
\end{align}
Assuming that $\beta_0(X)\geq \underline{\beta}$ and that we impose the same restriction on our estimate $\hat{\beta}(X)$ (which is a "minimal conditional instrument strength") and assuming that $\hat{\alpha}(X)^2\leq \bar{\alpha}^2$, then we have:
\begin{align}
    \error(\hat{g})^2 \lesssim~&  \E\left[(\alpha_0(X) - \hat{\alpha}(X))^2 (\beta_0(X) - \hat{\beta}(X))^2\right]\\
    ~& +  \E\left[(\beta_0(X) - \hat{\beta}(X))^4\right]
\end{align}
By Cauchy-Schwarz inequality the latter is upper bounded by:
\begin{align}
    \min\{&\|\hat{\alpha}-\alpha_0\|_{L^4}^2 \|\hat{\beta}-\beta_0\|_{L^4}^2 + \|\hat{\beta}-\beta_0\|_{L^4}^2,\\
     &\|\hat{\alpha}-\alpha_0\|_{L^2}^2 \|\hat{\beta}-\beta_0\|_{L^\infty}^2 + \|\hat{\beta}-\beta_0\|_{L^\infty}^2\}
\end{align}
Thus if we learn the effect of the instrument on the treatment (typically referred to as the compliance model) with an $L^{\infty}$ guarantee (e.g. a high-dimensional sparse linear logistic regression), then it suffices to learn the effect of the instrument on the outcome with an $L^2$ error. Moreover, as long as the error on the compliance model $\hat{\beta}$ is $o(n^{-1/4})$ and the product of the errors of $\hat{\alpha}$ and $\hat{\beta}$ is $o(n^{-1/2})$, then the nuisance impact is of lower order in the final guarantee for the Q-aggregation ensemble.

Again, if we further make an assumption that:
\begin{align}
    \|f_{\theta*}(X)-f_{\hat{\theta}}(X)\|_{L^\infty} \leq C \|f_{\theta*}(X)-f_{\hat{\theta}}(X)\|_{L^2}
\end{align}
Then we can take:
\begin{align}
    \error(\hat{g}) ~&= 2\, C\,\E\left[\left|\E[\hat{Y}(\hat{g}) - \hat{Y}(g_0)\mid X]\right|\right] \\
    ~&\lesssim \|\hat{a}-a_0\|_{L^2}\, \|\hat{\beta}-\beta_0\|_{L^2} + \|\hat{\beta}-\beta_0\|_{L^2}^2
\end{align}
and it suffices to control only the RMSE errors of $\hat{\alpha}$ and $\hat{\beta}$ to be $o(n^{-1/4})$.

\begin{remark}[Beyond Stratified Trials with Non-Compliance] When the instrument policy $\pi_0(X)$ is un-known, i.e. when we have a conditional instrument that corresponds to some natural experiment, then the quantity $\hat{Y}(\hat{g})$ is sensitive to errors in the estimate $\hat{\pi}$ of $\pi_0$. To reduce this sensitivity, we can incorporate residualization of $Y$ and $D$ (as also proposed in \cite{Tripadvisor}):
\begin{align}
    \hat{Y}(g) = \hat{\tau}(X) + \frac{(\tilde{Y} - \hat{\tau}(X) \tilde{D}) \tilde{Z}}{\hat{\beta}(X)}
\end{align}
where $\tilde{Y}=Y-\hat{h}(X)$, $\tilde{D}=D-\hat{r}(X)$ and $\tilde{Z}=Z-\hat{\pi}(X)$ and $\hat{h}$ is an estimate of $h_0(X)=\E[Y\mid X]$, $\hat{r}$ is an estimate of $r_0(X)=\E[D\mid X]$ and $\hat{\pi}$ is an estimate of $\pi_0(X)=\E[Z\mid X]$. With such a proxy label the error term in the Corollaries presented in this section, will contain extra terms that can be upper bounded by terms of the order $\|\hat{\pi}-\pi_0\|_{L^4}^2\, \|\hat{r}-r_0\|_{L^4}^2$, $\|\hat{\pi}-\pi_0\|_{L^4}^2\, \|\hat{h}-h_0\|_{L^4}^2$,  $\|\hat{\pi}-\pi_0\|_{L^4}^2\, \|\hat{\beta}-\beta_0\|_{L^4}^2$, $\|\hat{r}-r_0\|_{L^4}^2\, \|\hat{\beta}-\beta_0\|_{L^4}^2$ and $\|\hat{h}-h_0\|_{L^4}^2\, \|\hat{\beta}-\beta_0\|_{L^4}^2$.
\end{remark}

\section{Experiments}
We assessed the performance of our doubly robust Q-aggregation method on fully simulated and semi-synthetic data sets. We constructed several candidate CATE models based on meta-learning approaches using XGboost for regression and
classification sub-problems, such as for estimating the nuisance functions: $\hat{\mu}(X,D)$ (an estimate of $\E [Y \mid X, D]$),  $\hat{\mu}_d(X)$ (an estimate of $\E [Y \mid X, D=d]$), and $\hat{\pi}(X)$ (an estimate for the propensity). For constructing the CATE models we considered 8 meta-learning strategies, namely S-, T-, IPW-, X-, DR-, R-, DRX-, and DAX-learners. Details of each learner are given in Appendix~\ref{sec:model_details}.

We present results for Q-aggregation, convex stacking, and best-ERM (selecting the model with the best doubly robust loss). Note that convex stacking is a special case of Q-aggregation with $\nu$ set to 0. For Q-aggregation, we chose $\nu=0.1$ and did not incorporate any priors. Each data set is divided into 3 portions: $60\%$, $20\%$, and $20\%$. Candidate CATE models and nuisance functions are trained on the $60\%$ of the data. Subsequently, the CATE ensemble weights are trained on $20\%$ of the data. Finally, CATE RMSE is evaluated on the remaining $20\%$ of the data. For each model, we report the RMSE regret when compared to the oracle model selection - the model that achieves the lowest RMSE on the test set (i.e. $\|f_{\hat{\theta}} - \tau_0\|^2 - \min_{j} \|f_j-\tau_0\|^2$). 

\begin{table*}
\centering
\tiny
\caption{Mean RMSE Regret Over DGPs And Semi-synthetic Datasets.} \label{table:mean_all}
\begin{tabular}{lrrrrrrrrrrr}
\toprule
{} &     S &     T &   IPS &     X &    DR &     R &   DRX &   DAX &  Qtrain &  Convextrain &  Besttrain \\
\midrule
\makecell{DGP} & 1.495 & 0.946 & 2.350 & 1.149 & 0.663 & 1.125 & 0.894 & 0.837 &   \textbf{0.485} &        0.488 &      0.567 \\
\midrule
\makecell{Simple\\Semi-synthetic}& 0.611 & 0.534 & 7.081 & 0.284 & 0.445 & 0.619 & 0.411 & 0.384 &   \textbf{0.194} &        0.204 &      0.234 \\
\midrule
\makecell{Fitted\\Semi-Synthetic} & \textbf{0.133} & 1.912 & 5.106 & 0.431 & 0.417 & 0.899 & 0.407 & 0.816 &   0.300 &        0.322 &      0.258 \\
\midrule
\makecell{All\\Experiments} & 0.746 & 1.131 & 4.846 & 0.621 & 0.508 & 0.881 & 0.570 & 0.679 &   \textbf{0.327} &        0.338 &      0.353 \\
\bottomrule
\end{tabular}
\end{table*}

\begin{table*}
\centering
\tiny
\caption{RMSE Regret For Semi-synthetic Datasets with Simple CATE.}
\label{table:semi_simple}

\begin{tabular}{llllll}
\toprule
{} &                                             401k &                                              welfare &                                              poverty &                                                 star &                                                  criteo \\
\midrule
Besttrain                                &    \makecell{[355 $\pm$ 147.9] \\ 362.4 (585.7)} &  \makecell{[0.0390 $\pm$ 0.0181] \\ 0.0406 (0.0722)} &  \makecell{[0.0075 $\pm$ 0.0068] \\ 0.0060 (0.0194)} &  \makecell{[0.0107 $\pm$ 0.0096] \\ 0.0070 (0.0287)} &  \makecell{[0.00081 $\pm$ 0.00040] \\ 0.0008 (0.00161)} \\
Convextrain                              &  \makecell{[362.3 $\pm$ 157.1] \\ 371.5 (634.4)} &  \makecell{[0.0289 $\pm$ 0.0156] \\ 0.0275 (0.0557)} &  \makecell{[0.0066 $\pm$ 0.0049] \\ 0.0057 (0.0166)} &  \makecell{[0.0095 $\pm$ 0.0093] \\ 0.0062 (0.0284)} &  \makecell{[0.00066 $\pm$ 0.00032] \\ 0.0006 (0.00122)} \\
Qtrain                                   &  \makecell{[341.1 $\pm$ 157.1] \\ 331.3 (634.2)} &  \makecell{[0.0281 $\pm$ 0.0156] \\ 0.0272 (0.0537)} &  \makecell{[0.0061 $\pm$ 0.0049] \\ 0.0051 (0.0165)} &  \makecell{[0.0095 $\pm$ 0.0093] \\ 0.0061 (0.0284)} &  \makecell{[0.00060 $\pm$ 0.00032] \\ 0.0006 (0.00112)} \\
\midrule
\makecell{$\%$ Decrease \\ w.r.t Best}   &            \makecell{4.0692 $\%$ \\ 9.4008 $\%$} &              \makecell{38.7569 $\%$ \\ 49.2038 $\%$} &              \makecell{22.8005 $\%$ \\ 17.7918 $\%$} &              \makecell{13.0432 $\%$ \\ 15.7684 $\%$} &                 \makecell{34.8879 $\%$ \\ 30.0913 $\%$} \\
\midrule
\makecell{$\%$ Decrease \\ w.r.t Convex} &           \makecell{6.2148 $\%$ \\ 12.1341 $\%$} &                \makecell{2.6238 $\%$ \\ 1.0101 $\%$} &               \makecell{8.1161 $\%$ \\ 13.0241 $\%$} &                \makecell{0.2770 $\%$ \\ 2.3870 $\%$} &                   \makecell{9.9629 $\%$ \\ 8.5184 $\%$} \\
\bottomrule
\end{tabular}
\end{table*}

\paragraph{Simulated Data}
We employed six distinct data generation processes and performed experiments (DGP) on 100 instances of each DGP. In all cases, a scalar covariate $X$ is drawn from a uniform distribution. Each DGP is characterized by 3 functions: propensity $\pi(X)$, conditional baseline response surface $m(X)$, and treatment effect $\tau(X)$. Treatment $D$ is sampled from a binomial distribution with probability $\pi(X)$, and outcome Y is simulated by:
$$Y = m(X) + D \tau(X) + \epsilon$$
where $\epsilon \sim \mathrm{N}(0,\sigma^2)$. The details of each DGP can be found in the supplementary materials.

\paragraph{Semi-synthetic Data}
Semi-synthetic data was generated from five datasets used in previous studies: welfare dataset \citep{welfare}, poverty dataset\citep{poverty}, Project STAR dataset \citep{STAR}, 401k eligibility dataset \citep{401k_1,401k_2}, and the Criteo Uplift Modeling \citep{criteo} dataset. Since our analysis assumes a binary treatment, only one treatment is considered for datasets with multiple treatments. 

We used two simulation modes to generate the outcome from covariates. In the simple mode, the outcome is generated based on a simple linear CATE and baseline response function plus Gaussian noise (with $\sigma = 0.2$). In the fitted mode, the output is generated from a random forest regressor that is fitted on the real dataset, with randomness introduced by the addition of the model residual uniformly at random. 

\vspace{-1em}\paragraph{Results}
Overall, there is no single meta-learning model that performs consistently well across all DGPs or semi-synthetic datasets. We present in Table \ref{table:mean_all} the mean RMSE regret for different DGPs and different semisynthetic datasets, where the RMSE regret is normalized by the mean RMSE regret of all models for each DGP/dataset. The results in Table \ref{table:mean_all} show that best-model selection, convex stacking, and Q aggregation achieve a significantly lower average RMSE regret compared to choosing any single model. 

On semi-synthetic datasets, Q-aggregation achieves comparable RMSE regret with convex stacking (see Figure \ref{fig:bar-semi-synthetic}) and consistently outperforms both best-ERM and convex stacking in the case of a simple CATE and baseline response function. Table \ref{table:semi_simple} summarizes the RMSE regret for model selection, convex stacking, and Q-aggregation for semi-synthetic datasets with simple CATE functions, where each cell in the table reports [mean $\pm$ st.dev.] median (95\%) of RMSE Regret over 100 experiments for each semi-synthetic dataset simulated using the simple mode. The last two rows present the percentage decrease in the mean and median RMSE regret for Q-aggregation with respect to convex stacking and model selection.

\begin{figure*}
\vspace{.3in}
\centerline{\includegraphics[width=\textwidth]{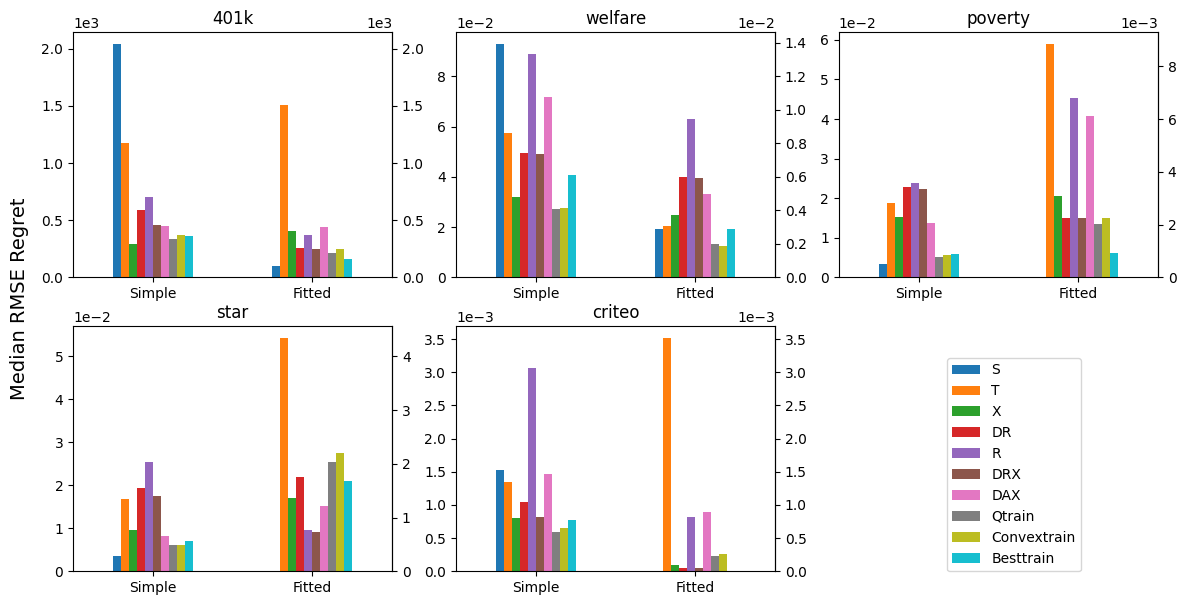}}
\vspace{.3in}
\caption{Median RMSE Regret for Each DGP.}
\label{fig:bar-semi-synthetic}
\end{figure*}

For simple CATE functions, Table \ref{table:semi_simple} shows that there is an decrease in the mean and median RMSE regret for Q-aggregation, when compared with convex stacking and model selection for every dataset. For datesets generated with the fitted CATE functions, the underlying model follows a S-learner setup. Thus, we expect the S-learners to be dominant models, which is evident Table \ref{table:mean_all}. As a result, we observe that the there is often an increase in the RMSE regret for ensemble models over the best single model, as is hinted by Theorem~\ref{thm:best_erm}, since in this case we have a dominant model.

\bibliography{refs}

\appendix

\onecolumn
\aistatstitle{Causal Q-Aggregation for CATE Model Selection: Appendix}

\tableofcontents
\newpage

\section{Detailed Comparison with Prior Works}\label{app:related-work}
In traditional supervised learning, the standard approach for model selection is through M-fold cross validation, also called the discrete super learner. \cite{vaart2006oracle} proved an oracle inequality (under conditions): 
\begin{align}
    \E [R(\hat{f})] - \E[R(f^*)] \leq O\left(\E[R(f^*)] + \frac{\log(M)}{n}\right)
\end{align}
where $\hat{f}$ is the best model selected by K-fold cross validation, $f^*$ is the oracle model selection, and $R(f)$ is the excess risk. Similar oracle inequalities are also provided in \cite{van2007super}.
Moreover, \cite{van2003unified} proved similar oracle inequalities for cross-validation even in the presence of nuisance functions in the loss. Our results on ERM-based selection and Q-aggregation based selection improve upon the results of \cite{van2003unified} on many critical directions. In summary, the results in \cite{van2003unified} cannot yield the doubly robust form of the dependence in the errors of the nuisances, for our main application to CATE model selection. Moreover, the selection rate results in \cite{van2003unified} depend on the performance of the best model in the set and only provide slow rates if the loss of the best model is not close to zero. Given the subtleties of the comparison of the results in \cite{van2003unified} and our work, we give very detailed comparison in the next subsection.

A similar oracle inequality was shown when ERM is used for model selection \citep{mitchell2009general}. In addition, under strong assumptions on the tails of the centered excess losses, the authors proved that model selection via ERM can achieve a tail bound that with probability $1-\delta$:
\begin{align}
R(\hat{f}) - R(f^*) \leq O\left(R(f^*) + \frac{\log(M)}{n} + \frac{\log(1/\delta)}{n}\right)    
\end{align}

Model Selection is particularly difficult for CATE estimation due to the lack of ground truth counterfactual data, making it difficult to evaluate model performance. Existing model selection criterion include the influence corrected loss \citep{alaa2019validating}, R-loss for double machine learning \citep{nie2021quasi}, doubly robust loss \citep{saito2020counterfactual}, the X-score \citep{kunzel2019metalearners}, matching score, \citep{rolling2014model}, calibration score \citep{dwivedi2020stable}, etc. In two recently conducted benchmarking studies on model selection for CATE estimators \citep{mahajan2022empirical, curth2023search}, the results showed that no single evaluation metric stood out as the dominant choice. Moreover, \cite{mahajan2022empirical} demonstrated that softmax stacking of the estimators often lead to better performance for any proxy evaluation metric.

Another line of work employs stacking which optimizes model weights by minimizing some performance metric on a held-out sample. \cite{grimmer2017estimating} constructed a weighted average across different potential outcome models via a cross validation based super learning approach\citep{van2007super} that estimates the weights from out-of-sample performance through regression. In contrast to our work, \cite{grimmer2017estimating} applies stacking to the nuisance function that is used to construct the estimator, instead of the estimator itself. Closer to our line of work, \cite{nie2021quasi} proposed to use a positive linear combination of different estimators that minimizes their proposed R-loss on a validation sample. Although both work demonstrated empirical success, they did not provide any theoretical guarantees for stacking. 

Bridging this gap, \cite{han2022ensemble} studied the the theoretical guarantees of their proposed stacking algorithm when compared to oracle model selection. The proposed causal stacking algorithm learns a convex combination of the candidate models by optimizing the doubly robust loss on a validation set. Notably, what we are presenting in this work, that is Q aggregation, can be reduced to causal stacking when choose the same loss and some hyperparameter is set to 0. Han and Wu showed that, with probability $1-\delta$, the RMSE regret is on the order of $\sqrt{\frac{\log(M^2)/\delta}{n}}$, which is a slower rate that what we presented in this paper. Interestingly, allowing the weights to be any non-negative value instead of summing to 1 improved performance when the treatment assignment is balanced. 

More recently, \cite{sugasawa2023bayesian} sought a different approach to construct ensembles which uses the Bayesian predictive synthesis\citep{mcalinn2019dynamic} framework, which allows the model/posterior weights to depend on the covariates, enabling the final ensemble to capture the heterogeneity of the performance of different models.

In our work, we provide additional theoretical results utilizing the orthogonal statistical learning framework \citep{foster2023orthogonal} when a Neyman orthogonal loss (e.g. the doubly robust loss) is used. In Section \ref{sec:ols_results}, we present results for both model selection and convex stacking using plug-in empirical risk minimization. In addition, we propose Q-aggregation \citep{Rigollet2014} for CATE estimation, which constructs a convex combination of the candidate functions via a modified version of empirical risk minimization, and achieves optimal rate \citep{aggregation} of $O(\frac{\log(M)}{n})$ plus higher order error terms for the nuisance functions.

\subsection{Comparison to \cite{van2003unified}}\label{app:vanderlaan}
\cite{van2003unified} proved oracle inequalities for cross-validation that considers the need to estimate nuisance functions for the loss function. 
Here, we reproduce the main theorem in \cite{van2003unified} using our notation.

\begin{theorem}[Theorem 1 of \cite{van2003unified}]Let $f_{\hat{k}}$ and $f_{\tilde{k}}$ denote the best model selected by cross-validation and oracle model selection respectively. Define the excess risk with respect to the true candidate model as $d(f, f_0) = \E_{S_n}[P(\ell_{f, g_0} -\ell_{f_0, g_0})]$, where $S_n$ is the data split. Let $p$ denote the proportion of data used for validation in each split. Further define:
\begin{align}
    r_1(n) =~& \max_{\bar{k}\in \{\hat{k}, \tilde{k}\}}\frac{\E\left[P(\ell_{f_{\bar{k}},\hat{g}} - \ell_{f_0,\hat{g}})-P(\ell_{f_{\bar{k}},g_0} - \ell_{f_0,g_0})\right]}{\sqrt{\E[P(\ell_{f_{\bar{k}},g_0} - \ell_{f_0,g_0})]}}\\
    r_2(n) =~& \E\left[\max_k \sqrt{P\left((\ell_{f_{k},\hat{g}} - \ell_{f_0,\hat{g}})-(\ell_{f_{k},g_0} - \ell_{f_0,g_0})\right)^2}\right] \\
    \tilde{r}(n) = ~& \sqrt{\E[d(f_{\tilde{k}},f_0)]}
\end{align}
Assuming that the nuisance estimator $\hat{g}$ converges to the true nuisance function $g_0$ as $n \rightarrow \infty$, the loss function is uniformly bounded by some constant $M_1$ for any $f\in F$, and $Var(\ell_{f,g_0} - \ell_{f_0,g_0}) \leq M_2 \E[\ell_{f,g_0} - \ell_{f_0,g_0}]$ for some constant $M_2$, then for any $\delta>0$:
\begin{align}
     \sqrt{\E[d(f_{\hat{k}},f_0)]} \leq \frac{(1+\delta)r_1(n) + \sqrt{(1+\delta)^2r_1(n)^2 + 4\epsilon_n(\delta)}}{2} =: \sqrt{R(n)} \label{eqn:sqrt_oracle}
\end{align} where
\begin{align}
    \epsilon_n(\delta) =~& (1+2 \delta) \tilde{r}^2(n)+4(1+\delta)^2\left(\frac{2M_1}{3} + \frac{M_2}{\delta}\right) \frac{1+\log (M)}{n p}+ \\
    & (1+\delta) r_1(n) \tilde{r}(n)+\frac{2 M_3(1+\delta) \log(M)}{(n p)^{0.5}} \max \left(r_2(n),(n p)^{-0.5} I\left(r_2(n)>0\right)\right)
\end{align}
And $M_3$ is given by:
\begin{align}
    M_3(n) = \frac{M_1}{3}+\sqrt{2} \frac{\sqrt{\log (2)}}{\log (M)}+\frac{\sqrt{2}}{\sqrt{\log (M)}}+b_0+\int_{b_0}^{\infty} 2 M^{1-m(x)} d x
\end{align} where
\begin{align}
    m(x) = \frac{1}{2} \frac{x^2}{1 / \log (M)+ \frac{4}{3M_1} x}
\end{align}
and $b_0$ is the solution to $1-m(x)=0$.\footnote{The theorem of \cite{van2003unified} also showed a form of asymptotic optimality only under the extra condition that:
\begin{align}
    O(r_1(n)^2 + \log(M)r_2(n)^2) = O(\tilde{r}(n)^2)
\end{align}}
\end{theorem}

Note that $1-m(x)$ is decreasing in $x$. Thus $M_3(n)$ is upper and lower bounded by some strictly positive constant factor that only depends on $M_1$. Up to a constant factor $C(M_1, M_2)$ that depends only on $M_1, M_2$ (and assuming for simplification that $\delta\leq 1$), the above can be simplified as:
\begin{align}
    R(n) \sim~& (1+\delta)^2 r_1(n)^2 + 4\epsilon_n(\delta)\\  
    \sim~& (1 + 2\delta)\E[d(f_{\tilde{k}},f_0)] + C(M_1, M_2) \left(\frac{1}{\delta} \frac{\log(M)}{n} + r_1(n)^2 + r_1(n)\,\tilde{r}(n) + \frac{\log(M)}{\sqrt{n}} r_2(n)
    \right)
\end{align}
Thus this upper bound, provides selection rates of the form:
\begin{align}
    \E[d(f_{\hat{k}},f_0)] - \E[d(f_{\tilde{k}},f_0)]  \lesssim 2\delta \E[d(f_{\tilde{k}},f_0)] + C(M_1, M_2) \left(\frac{1}{\delta} \frac{\log(M)}{n} + r_1(n)^2  + r_1(n)\,\tilde{r}(n) + \frac{\log(M)}{\sqrt{n}} r_2(n)\right)
\end{align}
As a side note, if we assume that the true nuisance parameter is known, i.e. $g_0$ is known, then $r_1(n) = r_2(n) = 0$, and we get the same type of results as in \cite{van2007super}.

Invoking the theorem with the square loss that we used in our paper, we get:
\begin{align}
    d(f, f_0) :=~& \|f-f_0\|^2\\
    r_1(n)^2 =~& \Theta\left(\frac{\E\left[\bias(X;\hat{g})\, (f_{\bar{k}}(X)-f_0(X))\right]^2}{\E\left[\|f_{\bar{k}}-f_0\|^2\right]} \right) \stackrel{(i)}{\leq}\Theta\left(\frac{\E\left[\bias(X;\hat{g})^2\right]\, \E\left[(f_{\bar{k}}(X)-f_0(X))^2\right]}{\E\left[\|f_{\bar{k}}-f_0\|^2\right]} \right)\\
    =~& \Theta\left(\frac{\E \left[\bias(X;\hat{g})^2\right] \E[\|f_{\bar{k}}-f_0\|^2]}{\E\left[\|f_{\bar{k}}-f_0\|^2\right]} \right)
     = \Theta(\E\left[\bias(X;\hat{g})^2\right])\\
    r_2(n)^2 =~& 
    \E\left[\max_j\left\{\E\left[\E[(\hat{Y}(\hat{g}) - \hat{Y}(g_0))^2|X] (f_j(X) - f_0(X))^2\right]\right\}\right]
    = \Theta\left(\E\left[(\hat{Y}(\hat{g}) - \hat{Y}(g_0))^2\right]\right)\\
    f_* :=~& f_{\tilde{k}}, ~~~ \tau_0 := f_0, ~~~~ \hat{f} := f_{\hat{k}} \\
    \tilde{r}(n) =~& \|f_* - \tau_0\|^2
\end{align}
Moreover, inequality $(i)$ is tight whenever $\bias(X;\hat{g})$ is independent of $f_k(X) - f_0(X)$, which cannot be excluded without further hard-to-justify assumptions. Therefore, the result of \cite{van2003unified} yields bounds of the form:
\begin{align}
     {\cal S}(n) :=~& \|\hat{f} - \tau_0\|^2 - \|f_* -\tau_0\|^2\\
     \leq~& \Theta\left(\delta \|f_* -\tau_0\|^2 + \frac{1}{\delta} \frac{\log(M)}{n} + \E\left[\bias(X;\hat{g})^2\right] + \sqrt{\E\left[\bias(X;\hat{g})^2\right]}\|f_* - \tau_0\| + {\frac{\log(M)}{\sqrt{n}}}\E\left[(\hat{Y}(\hat{g}) - \hat{Y}(g_0))^2\right]^{ {\frac{1}{2}}} \right)
\end{align}
for any constant $\delta>0$. Note that the optimal choice of $\delta$, yields a rate of:
\begin{align}
     {\cal S}(n) 
     \leq~& \Theta\left(\|f_* -\tau_0\|\, \sqrt{\frac{\log(M)}{n}}  + \E\left[\bias(X;\hat{g})^2\right] + \sqrt{\E\left[\bias(X;\hat{g})^2\right]}\|f_* - \tau_0\| +  {\frac{\log(M)}{\sqrt{n}}}\E\left[(\hat{Y}(\hat{g}) - \hat{Y}(g_0))^2\right]^{ {\frac{1}{2}}} \right)
\end{align}

On the contrary, our Theorem~\ref{thm:best_erm} yields an oracle inequality of the form:
\begin{align}
    {\cal S}(n) \leq~& 
    O\left(\epsilon +\frac{\log(M)}{n} + \E\left[\bias(X;\hat{g})^2\right] \right)&
    \epsilon =~& \|f_*-\tau_0\|^2 - \|\bar{f}_* - \tau_0\|^2
\end{align}
We identify the following critical differences in the two results:
\begin{enumerate}
    \item The term $ {\frac{\log(M)}{\sqrt{n}}}\E\left[(\hat{Y}(\hat{g}) - \hat{Y}(g_0))^2\right]^{ {\frac{1}{2}}}$ that appears in the result of \cite{van2003unified} does not appear in our bound. This is crucial for the doubly robust application. Note that $\E[\bias(X;\hat{g})^2]$ takes the form of the product of the errors of the propensity and the regression estimates, i.e. $\E[(\hat{\alpha}-\alpha_0)^2\, (\hat{g}-g_0)^2]$. On the contrary the quantity $\E\left[(\hat{Y}(\hat{g}) - \hat{Y}(g_0))^2\right]^{ {\frac{1}{2}}}$ does not take this product form. Thus, we would incur a dependence on the error of propensity and the regression, i.e. $\E[(\hat{\alpha}-\alpha_0)^2]^{ {\frac{1}{2}}} + \E[(\hat{g}-g_0)^2]^{ {\frac{1}{2}}}$. The latter wont converge faster than $1/ {\sqrt{n}}$, even for parametric functions and thereby the estimation error of the regression and the propensity cannot be ignored in the selection error.
    \item For any reasonable value of $\delta$, the result of \cite{van2003unified} would have a linear dependence on $\|f_*-\tau_0\|^2$. On the contrary our error only depends on the strictly smaller quanity $\epsilon = \|f_*-\tau_0\|^2 - \|\bar{f}_*-\tau_0\|^2$, where $\bar{f}_*$ is the minizer of the square loss within the convex hull of the original set of functions. The quantity $\epsilon$ has the property that when there is a ``dominant'' model, such that the best in the convex hull puts almost all weight on one model, then $\epsilon$ is almost zero, while $\|f_*-\tau_0\|^2$ can be very far away from zero. Note that when $\|f_*-\tau_0\|$ is far from zero, then the optimal choice of $\delta$, yields only a slow rate of $\sqrt{\log(M)/n}$ and not a fast rate of $\log(M)/n$.
\end{enumerate}

Moreover, our main Q-aggregation Theorem~\ref{thm:main} yields an oracle inequality of the form:
\begin{align}
    {\cal S}(n) \leq~& 
    O\left(\frac{\log(M)}{n} + \E\left[\bias(X;\hat{g})^2\right] \right)
\end{align}
We identify the following critical differences in the two results:
\begin{enumerate}
    \item The term $ {\frac{\log(M)}{\sqrt{n}}}\E\left[(\hat{Y}(\hat{g}) - \hat{Y}(g_0))^2\right]^{ {\frac{1}{2}}}$ does not appear in our bound, which is crucial for our doubly robust Q-aggregation results for the same reason as outlined in the bullet points above.
    \item For any reasonable value of $\delta$, the result of \cite{van2003unified} would have a linear dependence on $\|f_*-\tau_0\|^2$. On the contrary our error bound incurs no dependence on the performance of the best model! Even if the best model is very far from the true function $f_0$, we still compete with the best model at a $\log(M)/n +\E[\bias(X;\hat{g})^2]$ rate.
\end{enumerate}

\section{Proofs for ERM-Based Selection}\label{app:erm}
\subsection{Preliminaries}
We will be invoking a generalization of orthogonal statistical learning theorem \citep{foster2023orthogonal}.

First, we describe the assumptions for the theorem. Let $F$ and $F'$ be two function spaces such that $F\subseteq F'$. Consider $f_* \in F$ and $f' \in F'$. Let $d(g,g_0)$ denote a distance metric in the function space $G$. The the remaining of the section, we will assume that the loss $\ell_{f,g}$ is $C_b$-lipschitz in $f(X)$.

\begin{assumption}[Orthogonal Loss] \label{assum:1}
The population risk is Neyman Orthogonal, i.e. for all $\bar{f} \in \text{Star}(F,f')$ and $g \in G$: 
\begin{align}
    D_{f}D_{g}P \ell_{f',g_0}(Z)[g- g_0,f-f']=0 
\end{align}
\end{assumption}

\begin{assumption}[First Order Optimality]\label{assum:2}
The estimator $f'$ for the population risk satisfies the first-order optimality condition:
\begin{align}
    D_{f}P \ell_{f',g_0}(Z)[f-f']\geq 0 \quad \forall f \in \text{Star}(F, f')
\end{align}
\end{assumption}

\begin{assumption}[Higher Order Smoothness] \label{assum:3}
There exist constants $
\beta_1$ and $\beta_2$ such that the following holds:
\begin{itemize}
    \item[a)] Second order smoothness with respect to the target $f$. For all $f \in F$ and all $\bar{f} \in  \text{Star}(F, f')$:
        \begin{align}
            D^{2}_{f}P \ell_{\bar{f},g_0}(Z)[f - f',f-f'] \leq \beta_1 \|f-f'\|^2
        \end{align}
    \item[b)] Higher order smoothness. There exist $r\in [0,1)$ such that for all $f \in \text{Star}(F, f')$, $g \in G$, and  $\bar{g} \in \text{Star}(G, g_0)$:
        \begin{align}
            \left|D^{2}_{g}D_{f}P \ell_{f',\bar{g}}(Z)[f - f',g-g_0,g-g_0]\right| \leq \beta_2 \| f - f'\|^{1-r} d(g,g_0)^2
        \end{align}
\end{itemize}
\end{assumption}

\begin{assumption}[Strong Convexity]\label{assum:4}
The population risk is strongly convex with respect to the target function: there exist constants $\lambda$, $
\kappa >0$ such that for all $f \in F$ and an estimate $\hat{g} \in G$.
\begin{align}
    D^{2}_{f} P\ell_{\bar{f},\hat{g}}(Z)[f-f',f-f']\geq \lambda \|f-f'\|^2 - \kappa d(\hat{g},g_0)^{\frac{4}{1+r}} \quad \forall \bar{f}\in \text{Star}(F,f')
\end{align}
\end{assumption}

Alternatively, Assumption \ref{assum:1} and Assumption \ref{assum:3}(b) may be replaced by the following assumption.

\begin{assumption}\label{assum:for_square_loss}
    There exist $r\in [0,1)$ and constant $\beta_3$ such that for all $f, \bar{f}\in \text{Star}(F,f')$, $g\in G$:
    \begin{align}
       |D_f P\ell(\bar{f}, g)[f-f'] - D_f P\ell(\bar{f}, g_0)[f-f']| \leq \beta_3\|f-f'\|^{1-r}d(g,g_0)^{2}
    \end{align}
\end{assumption}

\begin{theorem}[Fast Rates Under Strong Convexity \citep{foster2023orthogonal}] \label{thm:osl}
Suppose \crefrange{assum:1}{assum:4}, or alternatively, \cref{assum:2,assum:4,assum:for_square_loss} and Assumption \ref{assum:3}(a), are satisfied for some $f'\in F'$ and function space $F\subseteq F'$. Then any estimator $\hat{f} \in F$ satisfies:
\begin{align}
    P\ell_{\hat{f},g_0}(Z) - P\ell_{f',g_0}(Z) \leq O\left(  P\ell_{\hat{f},\hat{g}}(Z) - P\ell_{f',\hat{g}}(Z) + d(\hat{g},g_0)^{\frac{4}{1+r}}\right)
\end{align}
Moreover, if the estimator $\hat{f}$ satisfies with probability at least $1-\delta$,
\begin{align}
    P\ell_{\hat{f},\hat{g}}(Z) - P\ell_{f',\hat{g}}(Z) \leq \epsilon_n(\delta) \| \hat{f} - f'\| + \alpha_n(\delta)
\end{align}
for functions $\epsilon_n(\delta)$ and $\alpha_n(\delta)$, then we have, with probability at least $1-\delta$
\begin{align}
    P\ell_{\hat{f},g_0}(Z) - P\ell_{f',g_0}(Z) \leq O\left(\epsilon_n(\delta)^2 + \alpha_n(\delta) +d(\hat{g},g_0)^{\frac{4}{1+r}} \right)
\end{align}   
\end{theorem}
Despite the small generalization made to consider any $f'\in F'$ that satisfies the first order optimality condition, the proof follows the original proof closely, and is included in Section \ref{sec:osl_proof} for completeness. We also showed that the same results can be obtained if Assumption \ref{assum:for_square_loss} is satisfied instead of \cref{assum:1,assum:3}.

\begin{lemma}\label{lemma:generalization}
    Assume that the Assumptions \ref{assum:1} and \ref{assum:3}(b) hold, or alternatively \cref{assum:for_square_loss} holds, for some $f'\in F'$. Then for any $f_* \in F$ we have:
    \begin{align}
        P\ell_{\hat{f},\hat{g}}(Z) - P\ell_{f',\hat{g}}(Z) \leq 
        P\ell_{\hat{f},\hat{g}}(Z) - P\ell_{f_*,\hat{g}}(Z) + {2(P\ell_{f_*,g_0}(Z) - P\ell_{f',g_0}(Z)) + O\left(((2/\lambda)^{\frac{2}{1+r}} + \kappa)\,d(\hat{g},g_0)^{\frac{4}{1+r}}\right)}
    \end{align}
    Moreover, if $\ell$ is a square loss, we have:
    \begin{align}
        P\ell_{\hat{f},\hat{g}}(Z) - P\ell_{f',\hat{g}}(Z) \leq 
        P\ell_{\hat{f},\hat{g}}(Z) - P\ell_{f_*,\hat{g}}(Z) + {2\left(\|f_*-\tau_0\|^2 - \|f'-\tau_0\|^2\right) + O\left(d(\hat{g},g_0)^{\frac{4}{1+r}}\right)}
    \end{align}
\end{lemma}
\begin{proof}
\begin{align}
    P\ell_{\hat{f},\hat{g}}(Z) - P\ell_{f',\hat{g}}(Z) =~& (P\ell_{\hat{f},\hat{g}}(Z) - P\ell_{f_*,\hat{g}}(Z)) +(P\ell_{f_*,\hat{g}}(Z) - P\ell_{f',\hat{g}}(Z))
\end{align}
Let us first consider the case where Assumptions \ref{assum:1} and \ref{assum:3}(b) holds.
For any $g$, taking a first order Taylor expansion yields:
\begin{align}
    P\ell_{f_*,g}(Z) - P\ell_{f',g}(Z) = D_{f}P\ell_{\bar{f},g}(Z)[f_*-f']
\end{align} for some $\bar{f}\in\text{Star}(F,f')$.
Let us first consider the excess risk for evaluated at $g_0$. 

Taking a second order Taylor expansion for the $(P\ell_{f_*,\hat{g}}(Z) - P\ell_{f',\hat{g}}(Z))$ term at $g_0$, we get:
\begin{align}
    (P\ell_{f_*,\hat{g}}(Z) - P\ell_{f',\hat{g}}(Z)) =~& (P\ell_{f_*,g_0}(Z) - P\ell_{f',g_0}(Z))+ D_{g}D_{f}P\ell_{\bar{f},g_0}(Z)[f_*-f',\hat{g}-g_0]\\
    +~& \frac{1}{2}D_{g}^2D_{f}P\ell_{\bar{f},\bar{g}}(Z)[f_*-f',\hat{g}-g_0,\hat{g}-g_0] \\
    =~&(P\ell_{f_*,g_0}(Z) - P\ell_{f',g_0}(Z))+\frac{1}{2}D_{g}^2D_{f}P\ell_{\bar{f},\bar{g}}(Z)[f_*-f',\hat{g}-g_0,\hat{g}-g_0] \tag{By Assumption \ref{assum:1}}\\
    \leq~&(P\ell_{f_*,g_0}(Z) - P\ell_{f',g_0}(Z)) +\frac{1}{2}D_{g}^2D_{f}P\ell_{\bar{f},\bar{g}}(Z)[f_*-f',\hat{g}-g_0,\hat{g}-g_0] \\
    \leq~&(P\ell_{f_*,g_0}(Z) - P\ell_{f',g_0}(Z)) + \frac{\beta_2}{2} \| f_* - f'\|^{1-r} d(\hat{g},g_0)^2 \tag{By Assumption \ref{assum:3}(b)}\\
    \leq~& (P\ell_{f_*,g_0}(Z) - P\ell_{f',g_0}(Z)) {+ \frac{\lambda}{2}\| f_* - f'\|^{2} + O\left((2/\lambda)^{\frac{2}{1+r}} d(\hat{g},g_0)^{\frac{4}{1+r}}\right) }\tag{Young's inequality}
\end{align}

Alternatively, if Assumption \ref{assum:for_square_loss} holds, then it follows immediately from a first order Taylor expansion that
\begin{align}
    (P\ell_{f_*,\hat{g}}(Z) - P\ell_{f',\hat{g}}(Z)) -(P\ell_{f_*,g_0}(Z) - P\ell_{f',g_0}(Z)) =~& D_f P\ell(\bar{f}, g)[\hat{f}-f'] - D_f P\ell(\bar{f}, g_0)[\hat{f}-f'] \\
    \leq~& \beta_3\|f-f'\|^{1-r}d(\hat{g},g_0)^{2}\\
    \leq~& {\frac{\lambda}{2}\| f_* - f'\|^{2} + O\left((2/\lambda)^{\frac{2}{1+r}} d(\hat{g},g_0)^{\frac{4}{1+r}}\right)}
\end{align}
and the same result follows.

{By the first order optimality of $f'$ and strong convexity, we have that:
\begin{align}
    P\ell_{f_*,g_0}(Z) - P\ell_{f',g_0}(Z) \geq D_{f} P\ell_{f',g_0}[f_*-f'] + \frac{\lambda}{2} \|f_*-f'\|^2 - \frac{\kappa}{2} d(\hat{g},g_0)^{\frac{4}{1+r}}\geq \frac{\lambda}{2} \|f_*-f'\|^2 - \frac{\kappa}{2} d(\hat{g},g_0)^{\frac{4}{1+r}}
\end{align}
Thus we deduce that:
\begin{align}
    P\ell_{f_*,\hat{g}}(Z) - P\ell_{f',\hat{g}}(Z) \leq 2(P\ell_{f_*,g_0}(Z) - P\ell_{f',g_0}(Z)) + O\left(((2/\lambda)^{\frac{2}{1+r}}+\kappa) d(\hat{g},g_0)^{\frac{4}{1+r}}\right) 
\end{align}}

For square losses, $\lambda=2$, $\kappa=0$ and as we will show in Section \ref{sec:sq_loss} (by the realizability of $\tau_0$, i.e. $\E[Y(g_0)\mid X]=\tau_0(X)$):
\begin{align}
    (P\ell_{f_*,g_0}(Z) - P\ell_{f',g_0}(Z)) = \|f_*-\tau_0\|^2 - \|f'-\tau_0\|^2
\end{align}
So we get:
\begin{align}
    P\ell_{\hat{f},\hat{g}}(Z) - P\ell_{f',\hat{g}}(Z) \leq~& (P\ell_{\hat{f},\hat{g}}(Z) - P\ell_{f_*,\hat{g}}(Z)) + 2(\|f_*-\tau_0\|^2 - \|f'-\tau_0\|^2) + O\left( d(\hat{g},g_0)^{\frac{4}{1+r}}\right)
\end{align}
\end{proof}

{\begin{corollary}[Generalization of Theorem \ref{thm:osl}] \label{cor:generalization}
Suppose \crefrange{assum:1}{assum:4}, or alternatively, \cref{assum:2,assum:4,assum:for_square_loss} and Assumption \ref{assum:3}(a), are satisfied for some $f'\in F'$ and function space $F\subseteq F'$ and for some constants $\lambda, \beta_1, \beta_2, \beta_3, \kappa, r$. Let $f_*$ be any function in $F$. Then any estimator $\hat{f} \in F$ satisfies
\begin{align}
    P\ell_{\hat{f},g_0}(Z) - P\ell_{f',g_0}(Z) \leq O\left(P\ell_{\hat{f},\hat{g}}(Z) - P\ell_{f_*,\hat{g}}(Z) +(P\ell_{f_*,g_0}(Z) - P\ell_{f',g_0}(Z)) + d(\hat{g},g_0)^{\frac{4}{1+r}}\right)
\end{align}
and for square losses:
\begin{align}
    P\ell_{\hat{f},g_0}(Z) - P\ell_{f',g_0}(Z) \leq O\left(  P\ell_{\hat{f},\hat{g}}(Z) - P\ell_{f_*,\hat{g}}(Z) + \|f_*-\tau_0\|^2 - \|f'-\tau_0\|^2 + d(\hat{g},g_0)^{\frac{4}{1+r}}\right)
\end{align}
Moreover, if the estimator $\hat{f}$ satisfies w.p. $1-\delta$
\begin{align}
    P\ell_{\hat{f},\hat{g}}(Z) - P\ell_{f_*,\hat{g}}(Z) \leq \epsilon_n(\delta) \| \hat{f} - f_*\| + \alpha_n(\delta)
\end{align}
for functions $\epsilon_n(\delta)$ and $\alpha_n(\delta)$, then we have, with probability at least $1-\delta$
\begin{align}
    P\ell_{\hat{f},g_0}(Z) - P\ell_{f',g_0}(Z) \leq O\left((P\ell_{f_*,g_0}(Z) - P\ell_{f',g_0}(Z)) + \epsilon_n(\delta)^2 + \alpha_n(\delta) +d(\hat{g},g_0)^{\frac{4}{1+r}} \right)
\end{align} 
and for square losses:
\begin{align}
    P\ell_{\hat{f},g_0}(Z) - P\ell_{f',g_0}(Z) \leq O\left(\|f_*-\tau_0\|^2 - \|f'-\tau_0\|^2  + \epsilon_n(\delta)^2 + \alpha_n(\delta) +d(\hat{g},g_0)^{\frac{4}{1+r}} \right)
\end{align} 
\end{corollary}}
\begin{proof}
The first result follows from a direct application of Lemma \ref{lemma:generalization}.

Next, 
\begin{align}
    P\ell_{\hat{f},\hat{g}}(Z) - P\ell_{f_*,\hat{g}}(Z) \leq~& \epsilon_n(\delta) \| \hat{f} - f_*\| + \alpha_n(\delta)\\
    \leq~&\epsilon_n(\delta)( \| \hat{f} - f'\|+\| f_* - f'\|) + \alpha_n(\delta)\\
    \leq~&\epsilon_n(\delta)\| \hat{f} - f'\|+ \frac{1}{2}\epsilon_n(\delta)^2+\frac{1}{2}\| f_* - f'\|^2 + \alpha_n(\delta)
\end{align}
{By the first order optimality of $f'$ and strong convexity, we have that:
\begin{align}
    P\ell_{f_*,g_0}(Z) - P\ell_{f',g_0}(Z) \geq D_{f} P\ell_{f',g_0}[f_*-f'] + \lambda \|f_*-f'\|^2 - \kappa d(g,g_0)^{\frac{4}{1+r}}\geq \lambda \|f_*-f'\|^2 - \kappa d(\hat{g},g_0)^{\frac{4}{1+r}}
\end{align}
Thus:
\begin{align}
    P\ell_{\hat{f},\hat{g}}(Z) - P\ell_{f_*,\hat{g}}(Z) 
    \leq~&\epsilon_n(\delta)\| \hat{f} - f'\|+ \frac{1}{2}\epsilon_n(\delta)^2+\alpha_n(\delta) + O\left(P\ell_{f_*,g_0}(Z) - P\ell_{f',g_0}(Z) +d(\hat{g},g_0)^{\frac{4}{1+r}}\right)
\end{align}}

Applying Lemma \ref{lemma:generalization} yields:
\begin{align}
    P\ell_{\hat{f},\hat{g}}(Z) - P\ell_{f',\hat{g}}(Z) \leq~& (P\ell_{\hat{f},\hat{g}}(Z) - P\ell_{f_*,\hat{g}}(Z)) + O\left((P\ell_{f_*,g_0}(Z) - P\ell_{f',g_0}(Z))+  d(\hat{g},g_0)^{\frac{4}{1+r}}\right)\\
    \leq~& \epsilon_n(\delta)\| \hat{f} - f'\| + O\left(\epsilon_n(\delta)^2+ \alpha_n(\delta) + (P\ell_{f_*,g_0}(Z) - P\ell_{f',g_0}(Z)) +  d(\hat{g},g_0)^{\frac{4}{1+r}}\right)
\end{align}
{
Finally, invoking Theorem \ref{thm:osl}, we get:
\begin{align}
    P\ell_{\hat{f},g_0}(Z) - P\ell_{f',g_0}(Z) \leq O\left(\epsilon_n(\delta)^2 + \alpha_n(\delta)+(P\ell_{f_*,g_0}(Z) - P\ell_{f',g_0}(Z))+d(\hat{g},g_0)^{\frac{4}{1+r}} \right)
\end{align}}
Similarly for square losses.
\end{proof}

When strong convexity is not satisfied, slower oracle rates can be shown.

\begin{assumption}[Universal Orthogonality]\label{assum:5}
For all $\bar{f} \in \text{Star}(F, f_*) + \text{Star}(F - f_*, 0)$,
\begin{align}
    D_{g}D_{f}P \ell_{\bar{f},g_0}(Z)[f-f_*,g- g_0]=0 \quad \forall g \in G, f \in F 
\end{align}  
\end{assumption}

\begin{assumption}\label{assum:6}
    The derivatives $D^2_g P \ell_{f,g}(Z)$ and $D^2_f D_g P \ell_{f,g}(Z)$ are continuous. Furthermore, there exists a constant $\beta$ such that for all $f \in \text{Star}(F,f_*)$ and $\bar{g} \in \text{Star}(G,g_0)$,
    \begin{align}
        |D^2_g P \ell_{f,\bar{g}}[g-g_0,g-g_0]| \leq \beta d(g, g_0)^2 \quad \forall g \in G
    \end{align}
\end{assumption}

Similarly, \crefrange{assum:5}{assum:6} can be replaced by the following assumption:

\begin{assumption}\label{assum:for_square_loss_slow}
    There exist a constant $\beta_4$ such that for all $f, \bar{f}\in \text{Star}(F,f')$, $g\in G$:
    \begin{align}
        |D_f P\ell(\bar{f}, g)[f-f'] - D_f P\ell(\bar{f}, g_0)[f-f']| \leq \beta_4 d(g,g_0)^{2}
    \end{align}
\end{assumption}

\begin{theorem}[Slow Rates without Strong Convexity]\label{thm:osl_slow}
Suppose that $f_*$ satisfies \crefrange{assum:5}{assum:6}, or alternatively, Assumption \ref{assum:for_square_loss_slow} holds, then with probability at least $1-\delta$, the target estimator $\hat{f}$ has the following exexss risk bound:
\begin{align}
    P\ell_{\hat{f},g_0}(Z) - P\ell_{f',g_0}(Z) \leq O\left(  P\ell_{\hat{f},\hat{g}}(Z) - P\ell_{f',\hat{g}}(Z) + d(g, g_0)^2\right)
\end{align}

\end{theorem}

\subsection{Best-ERM: Proof of Theorem~\ref{thm:best_erm} of the Main Paper}
First, we verify that \cref{assum:2,assum:4,assum:for_square_loss} are satisfied for the square losses, so that Corollary \ref{cor:generalization} can be invoked.

Consider $f'$ to be the minimizer of the population risk in the function space $F'=\co(F)$. Since we can rewrite $\ell_{f,g_0}(Z) = \E[(f-\tau_0)^2] + Var(\hat{Y}(g_0))$, it is straight forward to see that Assumption \ref{assum:2} is satisfied.  

Moreover, for any function $g$, we have:
\begin{align}
    D^{2}_{f}P \ell_{\bar{f},g_0}(Z)[f - f',f-f'] = \E[2(f-f')^2] = 2\|f-f'\|^2
\end{align}
Therefore, Assumption \ref{assum:4} are satisfied with $r=0$ and $\kappa=0$.

In Section \ref{sec:sq_loss}, we showed that:
\begin{align}
    | (P\ell_{f,g}(Z) - P\ell_{f',g}(Z)) - (P\ell_{f,g_0}(Z) - P\ell_{f',g_0}(Z)) | = 2 \sqrt{\E\left[\bias(\hat{g})]^2\right]}\, \|f_{\theta^*}-f_{\hat{\theta}}\|
\end{align}
Thus, Assumption \ref{assum:for_square_loss} is satisfied with $r=0$, $\beta_3=2$, and $d(\hat{g},g_0) = \E\left[\bias(\hat{g})]^2\right]^{\frac{1}{4}}$.

Next, we need to bound $ P\ell_{\hat{f},\hat{g}}(Z) - P\ell_{f_*,\hat{g}}(Z)$.

Invoking Lemma 14 from \cite{foster2023orthogonal}, with probability at least $1-\delta$, we can bound the centered excess loss for ERM over any function class $F$ by:
\begin{align}
    (P-P_n)({\ell}_{\hat{f},\hat{g}}(Z) - {\ell}_{f_*,\hat{g}}(Z)) \leq O\left(\delta_n \|\hat{f}-f_*\| + \delta_n^2\right) 
\end{align}
where for any $f_* \in F$,  $\delta_n = r_n + \sqrt{\frac{\log(1/\delta)}{n}}$, and $r_n$ is the critical radius of $\text{Star}(F-f_*,0)$.
In particular, we consider $f_*$ to be the minimizer of $\|f_*-f'\|$ in $F$.
By definition of $\hat{f}$, $P_n({\ell}_{\hat{f},\hat{g}}(Z) - {\ell}_{f_*,\hat{g}}(Z))<0$, so we have:
\begin{align}
    P({\ell}_{\hat{f},\hat{g}}(Z) - {\ell}_{f_*,\hat{g}}(Z)) \leq
    (P-P_n)({\ell}_{\hat{f},\hat{g}}(Z) - {\ell}_{f_*,\hat{g}}(Z)) \leq O\left(\delta_n \|\hat{f}-f_*\| + \delta_n^2\right) \label{eqn:excess_risk}
\end{align}

As a result, we can instantiate Corollary \ref{cor:generalization} with $r=0$, and $d(\hat{g},g_0)=\E\left[\bias(\hat{g})^2\right]^\frac{1}{4}$. Moreover, we can let $\epsilon_n = O(\delta_n)$ and $\alpha_n = O(\delta_n^2)$, to obtain:
\begin{align}
    \|\hat{f} - \tau_0\|^2 - \|f' - \tau_0\|^2 = ~&
    P\ell_{\hat{f},g_0}(Z) - P\ell_{f',g_0}(Z) \\
    \leq~& O\left(\delta_n^2 + \|f_*-\tau_0\|^2 - \|f'-\tau_0\|^2 +\E\left[\bias(\hat{g})^2\right]\right)
\end{align}

Applying the fact that, for $F = \{ f_1, f_2, \dots, f_M\}$, the critical radius of the $\text{Star}(F-f_*,0)$ is $r_n^2 = \Theta\left(\frac{\log(\max\{M,n\})}{n}\right)$, we get:
\begin{align}
     \|\hat{f} - \tau_0\|^2 - \|f' - \tau_0\|^2 \leq O\left( \frac{\log(\max\{M,n\}/\delta)}{n} + \|f_*-\tau_0\|^2 - \|f'-\tau_0\|^2 +\E\left[\bias(\hat{g})^2\right]\right)
\end{align}

Since by definition, $f'$ is equivalent to $\bar{f_*}$, we obtain a stronger result than in Theorem \ref{thm:best_erm} in the main text, which bounds the regret with respect to the optimal convex combination of models instead of the best model:
\begin{align}
    \mathcal{S}(\hat{\theta})  =~& \|f_{\hat{\theta}} - \tau_0\|^2 - \|f_{\theta_*} - \tau_0\|^2 \\
    \leq~& \|f_{\hat{\theta}} - \tau_0\|^2 - \|\bar{f_{\theta_*}} - \tau_0\|^2 \\
    \leq~& O\left( \frac{\log(\max\{M,n\}/\delta)}{n} + \|f_*-\tau_0\|^2 - \|f'-\tau_0\|^2 +\E\left[\bias(\hat{g})^2\right]\right)
\end{align}

\subsection{Convex-ERM: Proof of Theorem~\ref{thm:convex_erm} of the Main Paper}

For ERM over the convex hull, we can follow the same arguments to invoke Theorem \ref{thm:osl} with $r=0$, $\beta_3=2$, and $d(\hat{g},g_0) = \E\left[\bias(\hat{g})]^2\right]^{\frac{1}{4}}$ as in the previous section, with $f_*=f' = \bar{f_*}$ and  $F'=F = \co(\{f_1,f_2,\dots,f_M\})$, we get that:
\begin{align}
    \|\hat{f} - \tau_0\|^2 - \|f_* - \tau_0\|^2 
    \leq~& O\left(\delta_n^2 + \E\left[\bias(\hat{g})^2\right]\right)
\end{align}

By metric entropy arguments, the critical radius for this $\text{Star}(F-f_*,0)$ is $r_n^2 = \Theta\left(\sqrt{\frac{M\log(n)}{n}}\right)$. Combining, we get:
\begin{align}
    \|\hat{f} - \tau_0\|^2 - \|f_* - \tau_0\|^2 
    \leq~& O\left(\frac{M\log(n/\delta)}{n}+\E\left[\bias(\hat{g})^2\right] \right)
\end{align}

On the other hand, a different nuisance error term can be obtained by invoking Theorem \ref{thm:osl_slow}. First, we show that square losses satisfy Assumption \ref{assum:for_square_loss_slow}. Assuming $f$ is uniformly bounded by $U$ for all $f\in F$, we have that:
\begin{align}
 |D_f P\ell(\bar{f}, g)[f-f'] - D_f P\ell(\bar{f}, g_0)[f-f']| =~& |\E[(\bar{f}-\hat{Y}(g))(f-f') - (\bar{f}-\hat{Y}(g_0))(f-f')]|\\
 =~& |\E[(\hat{Y}(g)-\hat{Y}(g_0))(f-f')]|\\
 =~& 2|\E[\bias(\hat{g})(\hat{f}-f')]|\\
 \leq~& 4U\E[|\bias(\hat{g})|]
\end{align} 
Thus, Assumption \ref{assum:for_square_loss_slow} is satisfied with $\beta_4 = 4U$ and $d(\hat{g},g_0) =\E[|\bias(\hat{g})|]^{\frac{1}{2}}$.

Next, we need a generalization bound that does not involve $\|\hat{f}-f'\|$ terms. With probability at least $1-\delta$:
\begin{align}
    (P-P_n)(\ell_{\hat{f},\hat{g}} - \ell_{f', \hat{g}})\leq~& C_b (P-P_n)(\hat{f} - f') \tag{Since we assume $\ell$ is $C_b$ Lipschitz in $f(X)$}\\
    \leq~& O\left(R(F) + \sqrt{\frac{\log(1/\delta)}{n}}\right) \tag{By standard generalization arguments}
\end{align} where $R(F)$ is the Rademacher complexity of the function space $F$.

Thus invoking Theorem \ref{thm:osl_slow}, we get:
\begin{align}
    \|\hat{f} - \tau_0\|^2 - \|f_* - \tau_0\|^2 
    \leq~& O\left(R(F) + \sqrt{\frac{\log(1/\delta)}{n}}+\E\left[|\bias(\hat{g})|\right]\right)\\
    \leq~& O\left(\sqrt{\frac{\log(M)}{n}}+ \sqrt{\frac{\log(1/\delta)}{n}}+\E\left[|\bias(\hat{g})|\right]\right) \tag{Lemma 26.11 in \cite{UML}}\\
    \leq~& O\left(\sqrt{\frac{\log(M/\delta)}{n}}+\E\left[|\bias(\hat{g})|\right]\right) 
\end{align}

As with the Best-ERM case, this bounds the regret with respect to the optimal convex combination of models instead of the best model, which is a stronger result than in Theorem \ref{thm:convex_erm} in the main text.

\section{Mixed Bias Properties}\label{app:mixed-bias}

\subsection{Doubly Robust Targets for CATE under Conditional Exogeneity}
\begin{align}
    \E\left[\hat{Y}(g_0) - \hat{Y}(\hat{g})\mid X\right]
    =~&\E\left[a_0(D, W)\, h_0(D, W)\mid X\right]- \E\left[a_0(D, W)\, \hat{h}(D, W) \mid X\right]\\
    ~&~ +\E\left[ \hat{a}(D, W)\, (h_0(D, W) - \hat{h}(D, W))\mid X\right]\\
    =~& \E\left[(a_0(D, W) - \hat{a}(D, W))\, (h_0(D, W) - \hat{h}(D,W))\mid X\right]\\
    =~& \E\left[\hat{q}(D, W)\mid X\right]
\end{align}

\subsection{Doubly Robust Targets for CATE with Instruments}
\begin{align}
    \E[\hat{Y}(\hat{g}) - \hat{Y}(g_0)\mid X] 
    =~& \hat{\tau}(X) + \frac{\E[Y\tilde{Z}\mid X] - \hat{\tau}(X) \E[D\tilde{Z}\mid X]}{\hat{\beta}(X)} - \tau_0(X)\\
    =~& \hat{\tau}(X) + \frac{\tau_0(X) \beta_0(X) - \hat{\tau}(X) \beta_0(X)}{\hat{\beta}(X)} - \tau_0(X)\\
    =~& \hat{\tau}(X) - \tau_0(X) + (\tau_0(X) - \hat{\tau}(X)) \frac{\beta_0(X)}{\hat{\beta}(X)}\\
    =~& (\tau_0(X) - \hat{\tau}(X)) \left(\frac{\beta_0(X)}{\hat{\beta}(X)} - 1\right)\\
    =~& (\tau_0(X) - \hat{\tau}(X)) \frac{\beta_0(X) - \hat{\beta}(X)}{\hat{\beta}(X)}
\end{align}

\section{Experiment Details}

\subsection{Candidate Models} \label{sec:model_details}
We consider 8 meta-learners to be the candidate models for this study. The estimators $\hat{\mu}(X,D)$, $\hat{\mu}_0(X)$, $\hat{\mu}_1(X)$, $\hat{\pi}$ are fitted using XGboost regressors and classifiers. 

\textbf{T-learner}
$$\hat{\tau}_T(X) = \hat{\mu}_1(X) - \hat{\mu}_0(X) $$

\textbf{S-learner}
$$\hat{\tau}_S(X) = \hat{\mu}(X, D=1) - \hat{\mu}(X,D=0) $$

\textbf{Inverse Propensity Weighted Learner}
$$\hat{\tau}_{IPW}=\arg \min _{f \in F} \sum_i\left(\frac{D_iY_i}{\hat{\pi}(X_i)}-\frac{(1-D_i)Y_i}{1-\hat{\pi}(X_i)}-f\left(X_i\right)\right)^2$$

\textbf{ X-learner} \citep{kunzel2019metalearners} Using the potential outcome estimators, we first compute the imputed treatment effects:
\begin{align}
     \hat{\tau}^0 ~&=\arg \min _{f \in F} \sum_{i, D_i=0}\left(\hat{\mu}_1\left(X_i\right)-Y_i-f\left(X_i\right)\right)^2 \\
    \hat{\tau}^1 ~&=\arg \min _{f \in F} \sum_{i, D_i=1}\left(Y_i-\hat{\mu}_0\left(X_i\right)-f\left(X_i\right)\right)^2
\end{align}
The final estimator is an weighted average of the two imputed treatment effects:
$$\hat{\tau}_{IPW}(X)= \hat{\pi}(X)\hat{\tau}^0 + (1-\hat{\pi}(X))\hat{\tau}^1$$

\textbf{DR-learner} \citep{foster2023orthogonal,kennedy2020towards} The doubly robust potential outcome is formulated by the addition of the inverse propensity weighted residual term:
$$Y_d^{DR}=\hat{\mu}(X, d)+\frac{Y-\hat{\mu}(X, D)}{\hat{\pi}_d(X)} \cdot \mathbbm{1}(D=d)$$
The final estimator is:
$$\hat{\tau}_{D R}:=\arg \min _{f \in F} \sum_i\left(Y_{i, 1}^{DR}-Y_{i, 0}^{DR}-f\left(X_i\right)\right)^2$$

\textbf{R-learner} \citep{nie2021quasi} The R-learner employs an additional nuisance estimator, $\hat{m}(X)$, for the conditional response surface $\E[Y\mid X]$. Let $\tilde{Y} = Y - \hat{m}(X)$ and $\tilde{D} = D - \hat{\pi}(X)$ denote the residuals. Then, the final estimator is:
$$\hat{\tau}_R:=\arg \min _{f \in F} \sum_i\left(\tilde{Y}_i-f\left(X_i\right) \cdot \tilde{D}_i\right)^2$$

\textbf{Doubly Robust X-learner} The DRX-learner follows the doubly robust framework, but replaces $\hat{\mu}(X,d)$ by their X-learner analogues:
\begin{align}
    \hat{g}_0(X) ~&= \hat{\mu}(X,0) (1 - \hat{\pi}(X)) + 
            (\hat{\mu}(X,1) - \hat{\tau}^0) \hat{\pi}(X)\\
    \hat{g}_1(X) ~&= \hat{\mu}(X,1) (1 - \hat{\pi}(X)) + 
            (\hat{\mu}(X,0) - \hat{\tau}^1) \hat{\pi}(X)
\end{align}

\textbf{DAX-learner} This is a combination of the X-learner strategy and the Domain-Adaptation learner proposed in the EconML library \cite{econml} that we introduce in this paper. In the context of the X-Learner, we train a model $\hat{\tau}^1$ on the treated data and then we use it to calculate $\hat{\tau}(X) = \hat{\tau}^1(X)\, (1-\hat{\pi}(X)) + \hat{\tau}^{0}(X)\, \hat{\pi}(X)$ on all the data points. Thus in this case, when measuring the quality of the downstream CATE estimate $\hat{\tau}$ in the final step of the X-Learner, we care about the quality of $\hat{\tau}^1$ as measured by the metric:
\begin{align}
\E_X\left[(1-\hat{\pi}(X))^2 (\hat{\tau}^1(X) - \tau_0(X))^2\right]
\end{align}
On the contrary, the regression estimator for $\hat{\tau}^1$ would be guaranteeing:
\begin{align}
\E_{X\mid D=1}\left[(\hat{\tau}^1(X) - \tau_0(X))^2\right]
\end{align}
We can correct this covariate shift if instead we minimized the weighted square loss:
\begin{align}
\E_{X\mid D=1}\left[\frac{1}{\hat{\pi}(X)} (1 - \hat{\pi}(X))^2 (\hat{\tau}^1(Z) - \tau_0(Z))^2\right]
\end{align}
which corresponds to a density ratio weighted square loss. We can achieve this by using a sample-weighted regression oracle with weights $(1 - \hat{\pi}(X))^2\frac{1}{\hat{\pi}(X)}$, when training $\hat{\tau}^1$. Similarly, we should use weights $\hat{\pi}(X)^2 \frac{1}{1-\hat{\pi}(X)}$, when training $\hat{\tau}^0$. Then we proceed exactly as in the X-learner. We will call this the Domain-Adapted-X-Learner (DAX-Learner).

\subsection{Data Generating Processes for Synthetic Datasets}
The details of each DGP is summarized in Table \ref{tab:DGP}.
\begin{table}[ht!]
\centering
\caption{Summary of Data Generating Processes.}\label{tab:DGP}
\begin{tabular}[t]{llll}
\toprule
DGP &   $\pi(X)$ &  $m(X)$ &  $\tau(X)$  \\
\midrule
1 & 0.05 & $0.1 \mathbbm{1}(0.6\leq X\leq0.8)$&0.5 \\
2 & 0.05 & $0.1 \mathbbm{1}(0.6\leq X\leq0.8)$&$0.5+ 0.1 \mathbbm{1}(0.2\leq X\leq0.4)$ \\
3 & $0.5 - 0.49 \mathbbm{1}(0.3\leq X\leq0.6)$ & $0.1 \mathbbm{1}(0.6\leq X\leq0.8)$&$0.5X^2$ \\
4 & $0.5 - 0.49 \mathbbm{1}(0.3\leq X\leq0.6)$ & $0.1 \mathbbm{1}(0.6\leq X\leq0.8)$&$0.5X^2\mathbbm{1}(X\leq0.6) + 0.18$ \\
5 & 0.05 & 0.1 &$0.5 \mathbbm{1}(0.5\leq X\leq0.8)$ \\
6 & 0.95 & 0.1 &$0.5 \mathbbm{1}(0.5\leq X\leq0.8)$ \\
\end{tabular}
\end{table}

\subsection{Semi-synthetic Datasets}
The details dataset (after pre-processing) used for Semi-synthetic experiments are summarized in the Table \ref{tab:datasets}. 
\begin{table}[H]
\centering
\caption{Summary of Semi-synthetic Datasets.}\label{tab:datasets}
\begin{tabular}[t]{llllll}
\toprule
 &  401k & welfare &  STAR & Poverty & Criteo \\
\midrule
Number of Entries &  $9716$ & $12907$ &  $5220$ & $2480$ & $13979592$ \\
Number of Covariates &  $14$ & $42$ &  $15$ & $24$ & $12$ \\
Percentage Treated &  $37.18\%$ & $51.33\%$ &  $30.23\%$ & $50.89\%$ & $85.00\%$ \\
\end{tabular}
\end{table}

\subsection{Additional Results}
The RMSE regret for the candidate models as well as model selection, convex stacking, and Q-aggregation is summarized in \cref{table:toy_full,table:Semi_fitted_full,table:Semi_simple_full}. Each cell in the table reports [mean $\pm$ st.dev.] median (95\%) of RMSE Regret over the 100 experiments for each semi-synthetic dataset or DGP. The last two rows presents the percentage decrease in the mean and median RMSE regret for Q-aggregation with respect to convex stacking and model selection using the DR-loss. 

\begin{figure*}
\vspace{.3in}
\centerline{\includegraphics[width=\textwidth]{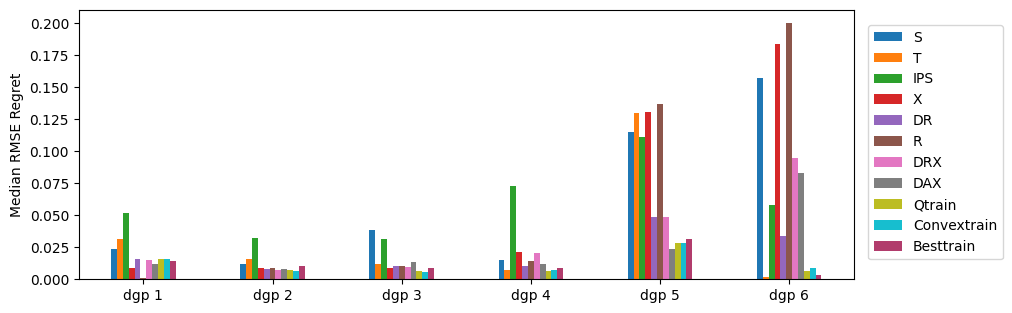}}
\vspace{.3in}
\caption{Median RMSE Regret for Each DGP.}
\label{fig:bar-toy}
\end{figure*}

\begin{table}[H]
\centering
\tiny
\caption{RMSE Regret For Fully-synthetic Datasets.}
\label{table:toy_full}

\begin{tabular}{lllllll}
\toprule
{} &                                            DGP 1 &                                            DGP 2 &                                            DGP 3 &                                            DGP 4 &                                            DGP 5 &                                            DGP 6 \\
\midrule
S                                        &  \makecell{[0.026 $\pm$ 0.012] \\ 0.023 (0.047)} &  \makecell{[0.014 $\pm$ 0.009] \\ 0.012 (0.032)} &  \makecell{[0.037 $\pm$ 0.010] \\ 0.039 (0.051)} &  \makecell{[0.017 $\pm$ 0.006] \\ 0.015 (0.027)} &  \makecell{[0.122 $\pm$ 0.034] \\ 0.115 (0.186)} &  \makecell{[0.159 $\pm$ 0.026] \\ 0.157 (0.194)} \\
T                                        &  \makecell{[0.032 $\pm$ 0.007] \\ 0.032 (0.042)} &  \makecell{[0.016 $\pm$ 0.007] \\ 0.016 (0.029)} &  \makecell{[0.014 $\pm$ 0.011] \\ 0.012 (0.036)} &  \makecell{[0.008 $\pm$ 0.006] \\ 0.007 (0.020)} &  \makecell{[0.130 $\pm$ 0.027] \\ 0.130 (0.168)} &  \makecell{[0.005 $\pm$ 0.006] \\ 0.002 (0.018)} \\
IPS                                      &  \makecell{[0.061 $\pm$ 0.047] \\ 0.052 (0.143)} &  \makecell{[0.045 $\pm$ 0.038] \\ 0.032 (0.118)} &  \makecell{[0.032 $\pm$ 0.008] \\ 0.032 (0.047)} &  \makecell{[0.075 $\pm$ 0.020] \\ 0.073 (0.112)} &  \makecell{[0.104 $\pm$ 0.044] \\ 0.111 (0.164)} &  \makecell{[0.058 $\pm$ 0.021] \\ 0.058 (0.094)} \\
X                                        &  \makecell{[0.010 $\pm$ 0.009] \\ 0.009 (0.025)} &  \makecell{[0.010 $\pm$ 0.006] \\ 0.009 (0.021)} &  \makecell{[0.011 $\pm$ 0.009] \\ 0.009 (0.023)} &  \makecell{[0.021 $\pm$ 0.008] \\ 0.021 (0.033)} &  \makecell{[0.130 $\pm$ 0.026] \\ 0.130 (0.168)} &  \makecell{[0.175 $\pm$ 0.024] \\ 0.183 (0.204)} \\
DR                                       &  \makecell{[0.017 $\pm$ 0.011] \\ 0.016 (0.037)} &  \makecell{[0.010 $\pm$ 0.007] \\ 0.008 (0.022)} &  \makecell{[0.013 $\pm$ 0.011] \\ 0.011 (0.035)} &  \makecell{[0.010 $\pm$ 0.007] \\ 0.010 (0.022)} &  \makecell{[0.052 $\pm$ 0.034] \\ 0.048 (0.111)} &  \makecell{[0.036 $\pm$ 0.015] \\ 0.034 (0.063)} \\
R                                        &  \makecell{[0.004 $\pm$ 0.005] \\ 0.001 (0.016)} &  \makecell{[0.010 $\pm$ 0.005] \\ 0.009 (0.019)} &  \makecell{[0.015 $\pm$ 0.012] \\ 0.010 (0.037)} &  \makecell{[0.015 $\pm$ 0.008] \\ 0.015 (0.029)} &  \makecell{[0.135 $\pm$ 0.026] \\ 0.136 (0.174)} &  \makecell{[0.189 $\pm$ 0.024] \\ 0.200 (0.215)} \\
DRX                                      &  \makecell{[0.017 $\pm$ 0.012] \\ 0.015 (0.039)} &  \makecell{[0.010 $\pm$ 0.008] \\ 0.007 (0.026)} &  \makecell{[0.012 $\pm$ 0.009] \\ 0.010 (0.036)} &  \makecell{[0.020 $\pm$ 0.008] \\ 0.020 (0.032)} &  \makecell{[0.053 $\pm$ 0.034] \\ 0.048 (0.111)} &  \makecell{[0.104 $\pm$ 0.035] \\ 0.095 (0.176)} \\
DAX                                      &  \makecell{[0.018 $\pm$ 0.014] \\ 0.012 (0.045)} &  \makecell{[0.010 $\pm$ 0.008] \\ 0.008 (0.025)} &  \makecell{[0.018 $\pm$ 0.016] \\ 0.013 (0.047)} &  \makecell{[0.016 $\pm$ 0.016] \\ 0.012 (0.042)} &  \makecell{[0.031 $\pm$ 0.027] \\ 0.024 (0.076)} &  \makecell{[0.090 $\pm$ 0.028] \\ 0.083 (0.140)} \\
Besttrain                                &  \makecell{[0.017 $\pm$ 0.013] \\ 0.014 (0.040)} &  \makecell{[0.011 $\pm$ 0.008] \\ 0.010 (0.027)} &  \makecell{[0.012 $\pm$ 0.011] \\ 0.009 (0.036)} &  \makecell{[0.010 $\pm$ 0.009] \\ 0.008 (0.021)} &  \makecell{[0.036 $\pm$ 0.031] \\ 0.032 (0.089)} &  \makecell{[0.008 $\pm$ 0.013] \\ 0.003 (0.041)} \\
Convextrain                              &  \makecell{[0.017 $\pm$ 0.011] \\ 0.016 (0.039)} &  \makecell{[0.009 $\pm$ 0.008] \\ 0.007 (0.027)} &  \makecell{[0.008 $\pm$ 0.009] \\ 0.006 (0.025)} &  \makecell{[0.008 $\pm$ 0.007] \\ 0.007 (0.020)} &  \makecell{[0.033 $\pm$ 0.025] \\ 0.028 (0.083)} &  \makecell{[0.011 $\pm$ 0.009] \\ 0.009 (0.027)} \\
Qtrain                                   &  \makecell{[0.016 $\pm$ 0.011] \\ 0.015 (0.038)} &  \makecell{[0.009 $\pm$ 0.007] \\ 0.007 (0.023)} &  \makecell{[0.009 $\pm$ 0.009] \\ 0.006 (0.025)} &  \makecell{[0.008 $\pm$ 0.007] \\ 0.007 (0.021)} &  \makecell{[0.033 $\pm$ 0.025] \\ 0.028 (0.083)} &  \makecell{[0.010 $\pm$ 0.010] \\ 0.006 (0.026)} \\
\midrule
\makecell{$\%$ Decrease \\ w.r.t Best}   &             \makecell{1.264 $\%$ \\ -7.700 $\%$} &            \makecell{22.081 $\%$ \\ 40.148 $\%$} &            \makecell{41.138 $\%$ \\ 44.073 $\%$} &            \makecell{20.952 $\%$ \\ 25.093 $\%$} &            \makecell{11.236 $\%$ \\ 11.446 $\%$} &          \makecell{-13.003 $\%$ \\ -49.579 $\%$} \\
\midrule
\makecell{$\%$ Decrease \\ w.r.t Convex} &              \makecell{2.324 $\%$ \\ 2.740 $\%$} &             \makecell{0.119 $\%$ \\ -6.044 $\%$} &            \makecell{-4.316 $\%$ \\ -8.569 $\%$} &              \makecell{0.526 $\%$ \\ 5.734 $\%$} &             \makecell{0.080 $\%$ \\ -0.636 $\%$} &            \makecell{12.474 $\%$ \\ 35.252 $\%$} \\
\bottomrule
\end{tabular}
\end{table}

\begin{table}[H]
\centering
\tiny
\caption{RMSE Regret For Semi-synthetic Datasets with Simple CATE.}
\label{table:Semi_simple_full}
\begin{tabular}{llllll}
\toprule
{} &                                             401k &                                              welfare &                                              poverty &                                                 star &                                                  criteo \\
\midrule
S                                        &     \makecell{[2042 $\pm$ 411.7] \\ 2042 (2677)} &  \makecell{[0.0947 $\pm$ 0.0248] \\ 0.0929 (0.1346)} &  \makecell{[0.0050 $\pm$ 0.0050] \\ 0.0034 (0.0139)} &  \makecell{[0.0055 $\pm$ 0.0059] \\ 0.0035 (0.0174)} &  \makecell{[0.00151 $\pm$ 0.00044] \\ 0.0015 (0.00223)} \\
T                                        &     \makecell{[1172 $\pm$ 229.5] \\ 1172 (1572)} &  \makecell{[0.0546 $\pm$ 0.0267] \\ 0.0573 (0.1035)} &  \makecell{[0.0193 $\pm$ 0.0072] \\ 0.0188 (0.0320)} &  \makecell{[0.0224 $\pm$ 0.0211] \\ 0.0169 (0.0670)} &  \makecell{[0.00136 $\pm$ 0.00054] \\ 0.0013 (0.00231)} \\
IPS                                      &      \makecell{[5845 $\pm$ 1489] \\ 5513 (8983)} &  \makecell{[0.7937 $\pm$ 0.0893] \\ 0.7757 (0.9694)} &  \makecell{[0.1739 $\pm$ 0.0373] \\ 0.1731 (0.2329)} &  \makecell{[0.6294 $\pm$ 0.2361] \\ 0.5938 (1.0945)} &  \makecell{[0.05433 $\pm$ 0.02190] \\ 0.0459 (0.09577)} \\
X                                        &  \makecell{[302.6 $\pm$ 155.6] \\ 289.2 (582.4)} &  \makecell{[0.0337 $\pm$ 0.0170] \\ 0.0319 (0.0649)} &  \makecell{[0.0159 $\pm$ 0.0064] \\ 0.0152 (0.0267)} &  \makecell{[0.0145 $\pm$ 0.0158] \\ 0.0096 (0.0490)} &  \makecell{[0.00079 $\pm$ 0.00032] \\ 0.0008 (0.00135)} \\
DR                                       &    \makecell{[581 $\pm$ 225.8] \\ 587.6 (994.3)} &  \makecell{[0.0522 $\pm$ 0.0195] \\ 0.0494 (0.0876)} &  \makecell{[0.0241 $\pm$ 0.0068] \\ 0.0229 (0.0375)} &  \makecell{[0.0206 $\pm$ 0.0165] \\ 0.0194 (0.0431)} &  \makecell{[0.00105 $\pm$ 0.00055] \\ 0.0010 (0.00193)} \\
R                                        &   \makecell{[727.5 $\pm$ 217.9] \\ 704.6 (1096)} &  \makecell{[0.0872 $\pm$ 0.0174] \\ 0.0889 (0.1157)} &  \makecell{[0.0248 $\pm$ 0.0072] \\ 0.0239 (0.0378)} &  \makecell{[0.0280 $\pm$ 0.0176] \\ 0.0254 (0.0540)} &  \makecell{[0.00302 $\pm$ 0.00057] \\ 0.0031 (0.00390)} \\
DRX                                      &    \makecell{[460.5 $\pm$ 194] \\ 460.3 (801.7)} &  \makecell{[0.0510 $\pm$ 0.0188] \\ 0.0489 (0.0832)} &  \makecell{[0.0240 $\pm$ 0.0069] \\ 0.0223 (0.0390)} &  \makecell{[0.0198 $\pm$ 0.0157] \\ 0.0176 (0.0411)} &  \makecell{[0.00082 $\pm$ 0.00040] \\ 0.0008 (0.00161)} \\
DAX                                      &  \makecell{[446.6 $\pm$ 179.7] \\ 444.6 (741.2)} &  \makecell{[0.0727 $\pm$ 0.0173] \\ 0.0718 (0.1047)} &  \makecell{[0.0147 $\pm$ 0.0068] \\ 0.0138 (0.0260)} &  \makecell{[0.0133 $\pm$ 0.0170] \\ 0.0083 (0.0389)} &  \makecell{[0.00145 $\pm$ 0.00030] \\ 0.0015 (0.00190)} \\
Besttrain                                &    \makecell{[355 $\pm$ 147.9] \\ 362.4 (585.7)} &  \makecell{[0.0390 $\pm$ 0.0181] \\ 0.0406 (0.0722)} &  \makecell{[0.0075 $\pm$ 0.0068] \\ 0.0060 (0.0194)} &  \makecell{[0.0107 $\pm$ 0.0096] \\ 0.0070 (0.0287)} &  \makecell{[0.00081 $\pm$ 0.00040] \\ 0.0008 (0.00161)} \\
Convextrain                              &  \makecell{[362.3 $\pm$ 157.1] \\ 371.5 (634.4)} &  \makecell{[0.0289 $\pm$ 0.0156] \\ 0.0275 (0.0557)} &  \makecell{[0.0066 $\pm$ 0.0049] \\ 0.0057 (0.0166)} &  \makecell{[0.0095 $\pm$ 0.0093] \\ 0.0062 (0.0284)} &  \makecell{[0.00066 $\pm$ 0.00032] \\ 0.0006 (0.00122)} \\
Qtrain                                   &  \makecell{[341.1 $\pm$ 157.1] \\ 331.3 (634.2)} &  \makecell{[0.0281 $\pm$ 0.0156] \\ 0.0272 (0.0537)} &  \makecell{[0.0061 $\pm$ 0.0049] \\ 0.0051 (0.0165)} &  \makecell{[0.0095 $\pm$ 0.0093] \\ 0.0061 (0.0284)} &  \makecell{[0.00060 $\pm$ 0.00032] \\ 0.0006 (0.00112)} \\
\midrule
\makecell{$\%$ Decrease \\ w.r.t Best}   &            \makecell{4.0692 $\%$ \\ 9.4008 $\%$} &              \makecell{38.7569 $\%$ \\ 49.2038 $\%$} &              \makecell{22.8005 $\%$ \\ 17.7918 $\%$} &              \makecell{13.0432 $\%$ \\ 15.7684 $\%$} &                 \makecell{34.8879 $\%$ \\ 30.0913 $\%$} \\
\midrule
\makecell{$\%$ Decrease \\ w.r.t Convex} &           \makecell{6.2148 $\%$ \\ 12.1341 $\%$} &                \makecell{2.6238 $\%$ \\ 1.0101 $\%$} &               \makecell{8.1161 $\%$ \\ 13.0241 $\%$} &                \makecell{0.2770 $\%$ \\ 2.3870 $\%$} &                   \makecell{9.9629 $\%$ \\ 8.5184 $\%$} \\
\bottomrule
\end{tabular}
\end{table}

\begin{table}[H]
\centering
\tiny
\caption{RMSE Regret For Semi-synthetic Datasets with Fitted CATE.}
\label{table:Semi_fitted_full}
\begin{tabular}{llllll}
\toprule
{} &                                             401k &                                              welfare &                                              poverty &                                                         star &                                                  criteo \\
\midrule
S                                        &  \makecell{[115.1 $\pm$ 104.1] \\ 98.59 (306.3)} &  \makecell{[0.0029 $\pm$ 0.0018] \\ 0.0029 (0.0061)} &  \makecell{[0.0002 $\pm$ 0.0004] \\ 0.0000 (0.0013)} &          \makecell{[0.0843 $\pm$ 0.1903] \\ 0.0000 (0.4518)} &  \makecell{[0.00000 $\pm$ 0.00001] \\ 0.0000 (0.00001)} \\
T                                        &     \makecell{[1526 $\pm$ 340.2] \\ 1511 (2052)} &  \makecell{[0.0033 $\pm$ 0.0014] \\ 0.0031 (0.0056)} &  \makecell{[0.0088 $\pm$ 0.0013] \\ 0.0088 (0.0110)} &          \makecell{[4.3250 $\pm$ 0.6061] \\ 4.3450 (5.1982)} &  \makecell{[0.00334 $\pm$ 0.00089] \\ 0.0035 (0.00478)} \\
IPS                                      &      \makecell{[1643 $\pm$ 1291] \\ 1212 (4731)} &  \makecell{[0.0343 $\pm$ 0.0059] \\ 0.0343 (0.0437)} &  \makecell{[0.0219 $\pm$ 0.0151] \\ 0.0181 (0.0575)} &  \makecell{[345.6718 $\pm$ 130.7644] \\ 321.2090 (563.5089)} &  \makecell{[0.00227 $\pm$ 0.00107] \\ 0.0024 (0.00376)} \\
X                                        &  \makecell{[431.7 $\pm$ 271.9] \\ 400.8 (919.6)} &  \makecell{[0.0039 $\pm$ 0.0017] \\ 0.0037 (0.0069)} &  \makecell{[0.0032 $\pm$ 0.0010] \\ 0.0031 (0.0048)} &          \makecell{[1.3916 $\pm$ 0.4319] \\ 1.3608 (2.1141)} &  \makecell{[0.00013 $\pm$ 0.00011] \\ 0.0001 (0.00035)} \\
DR                                       &  \makecell{[321.8 $\pm$ 284.4] \\ 259.5 (882.3)} &  \makecell{[0.0063 $\pm$ 0.0021] \\ 0.0060 (0.0103)} &  \makecell{[0.0025 $\pm$ 0.0014] \\ 0.0022 (0.0048)} &          \makecell{[1.6791 $\pm$ 0.8769] \\ 1.7604 (3.0720)} &  \makecell{[0.00006 $\pm$ 0.00004] \\ 0.0000 (0.00015)} \\
R                                        &    \makecell{[410.2 $\pm$ 216.7] \\ 374 (837.5)} &  \makecell{[0.0095 $\pm$ 0.0020] \\ 0.0094 (0.0129)} &  \makecell{[0.0068 $\pm$ 0.0011] \\ 0.0068 (0.0086)} &          \makecell{[0.8925 $\pm$ 0.4380] \\ 0.7595 (1.7322)} &  \makecell{[0.00081 $\pm$ 0.00008] \\ 0.0008 (0.00094)} \\
DRX                                      &  \makecell{[314.5 $\pm$ 279.5] \\ 251.4 (842.7)} &  \makecell{[0.0063 $\pm$ 0.0021] \\ 0.0059 (0.0103)} &  \makecell{[0.0025 $\pm$ 0.0014] \\ 0.0022 (0.0048)} &          \makecell{[0.8536 $\pm$ 0.6130] \\ 0.7277 (2.0083)} &  \makecell{[0.00006 $\pm$ 0.00004] \\ 0.0001 (0.00013)} \\
DAX                                      &   \makecell{[501.6 $\pm$ 254.7] \\ 437.1 (1063)} &  \makecell{[0.0052 $\pm$ 0.0015] \\ 0.0050 (0.0080)} &  \makecell{[0.0062 $\pm$ 0.0011] \\ 0.0061 (0.0082)} &          \makecell{[1.2216 $\pm$ 0.4937] \\ 1.2192 (2.0556)} &  \makecell{[0.00090 $\pm$ 0.00021] \\ 0.0009 (0.00123)} \\
Besttrain                                &  \makecell{[257.5 $\pm$ 262.1] \\ 159.9 (812.1)} &  \makecell{[0.0029 $\pm$ 0.0015] \\ 0.0029 (0.0053)} &  \makecell{[0.0017 $\pm$ 0.0020] \\ 0.0009 (0.0057)} &          \makecell{[2.0018 $\pm$ 1.4700] \\ 1.6779 (4.7187)} &  \makecell{[0.00002 $\pm$ 0.00003] \\ 0.0000 (0.00010)} \\
Convextrain                              &  \makecell{[266.2 $\pm$ 180.9] \\ 245.3 (577.4)} &  \makecell{[0.0019 $\pm$ 0.0012] \\ 0.0019 (0.0040)} &  \makecell{[0.0023 $\pm$ 0.0014] \\ 0.0022 (0.0048)} &          \makecell{[2.1629 $\pm$ 0.8906] \\ 2.2012 (3.5904)} &  \makecell{[0.00027 $\pm$ 0.00007] \\ 0.0003 (0.00038)} \\
Qtrain                                   &    \makecell{[245.6 $\pm$ 183] \\ 215.8 (579.5)} &  \makecell{[0.0019 $\pm$ 0.0012] \\ 0.0020 (0.0041)} &  \makecell{[0.0021 $\pm$ 0.0014] \\ 0.0020 (0.0047)} &          \makecell{[2.0705 $\pm$ 0.9286] \\ 2.0370 (3.5692)} &  \makecell{[0.00024 $\pm$ 0.00006] \\ 0.0002 (0.00035)} \\
\midrule
\makecell{$\%$ Decrease \\ w.r.t Best}   &          \makecell{4.8579 $\%$ \\ -25.8868 $\%$} &              \makecell{50.4544 $\%$ \\ 46.1057 $\%$} &            \makecell{-17.3363 $\%$ \\ -55.4313 $\%$} &                     \makecell{-3.3152 $\%$ \\ -17.6290 $\%$} &              \makecell{-91.0409 $\%$ \\ -100.0000 $\%$} \\
\midrule
\makecell{$\%$ Decrease \\ w.r.t Convex} &           \makecell{8.4122 $\%$ \\ 13.6626 $\%$} &              \makecell{-0.7623 $\%$ \\ -4.4831 $\%$} &                \makecell{8.7545 $\%$ \\ 9.9212 $\%$} &                        \makecell{4.4656 $\%$ \\ 8.0570 $\%$} &                 \makecell{10.7980 $\%$ \\ 11.0042 $\%$} \\
\bottomrule
\end{tabular}

\end{table}

\section{Proof of Main Theorem}

In this section, we will be proving Theorem \ref{thm:main} in the main paper, which we also reproduce here:
\begin{theorem}[Main Theorem]
Assume that for some function $\error(\hat{g})$:
\begin{align}
    P(\tilde{\ell}_{\theta,g_0}(Z) - \tilde{\ell}_{\theta',g_0}(Z)) - P(\tilde{\ell}_{\theta,\hat{g}}(Z) - \tilde{\ell}_{\theta',\hat{g}}(Z)) 
    \leq \error(\hat{g})\, \|f_{\theta} - f_{\theta'}\|^{\gamma}
\end{align}
for some $\gamma\in [0, 1]$ and $\ell(f, g_0)$ is $\sigma$-strongly convex in expectation with respect to $f$. Moreover, assume that the loss $\ell(Z; f(X), \hat{g})$ is $C_b(\hat{g})$-Lipszhitz with respect to $f(X)$, and both the loss and the candidate functions are uniformly bounded in $[-b,b]$ almost surely. For any $\mu \leq \min\{\nu, 1-\nu\} \frac{\sigma}{14}$  and $\beta \geq \{ \frac{8C_b^2(1-\nu)^2}{\mu}, 8\sqrt{3}bC_b(1-\nu),\frac{2C_b\nu(\nu C_b + \mu b)}{\mu}\}$, w.p. $1-\delta$:
\begin{align}
    R(\hat{\theta}, g_0) \leq~& \min_{j=1}^M\left( R(e_j, g_0)  +\frac{\beta}{n}\log(1/\pi_j)\right) 
     + O\left(  \frac{\log(1/\delta) \max\{\frac{1}{\mu} C_b(\hat{g})^2, b\, C_b(\hat{g}) \}}{n}\right) 
    + \left(\frac{1}{\mu}\right)^{\frac{\gamma}{2-\gamma}} \error(\hat{g})^{\frac{2}{2-\gamma}}
\end{align}
\end{theorem}

We remind some key definitions and assumptions from the main text.
We assume that the function $\ell$ is $\sigma$-strongly convex in $f$ in expectation, i.e.:
\begin{align}
    P\left(\ell(Z;f'(X), g_0) - \ell(Z;f(X), g_0)\right) \geq 
    P\left(\langle \nabla_f \ell(Z; f(X), g_0), f'(X)-f(X)\rangle\right) 
    + \frac{\sigma}{2} \|f-f'\|^2
\end{align}
We consider the modified loss, which penalizes the weight put on each model, based on the individual model performance:
\begin{align}
    \tilde{\ell}_{\theta,g}(Z) =~& (1-\nu) \ell_{\theta,g}(Z) + \nu \sum_{j} \theta_j \ell_{e_j,g}(Z) &
    \tilde{R}(\theta,g) =~& P\tilde{\ell}_{\theta,g}(Z)
\end{align}
We also define with:
\begin{align}
    K(\theta) :=~&\sum_{j} \theta_j \log(1/\pi_j)
\end{align}
Then the Q-aggregation ensemble is defined as:
\begin{align}
    \hat{\theta} =~& \arg\min_{\theta} P_n\tilde{\ell}_{\theta, \hat{g}}(Z) + \frac{\beta}{n} K(\theta)
\end{align}

We introduce some population quantities that will play a key role in our analysis. Define:
\begin{align}
    V(\theta) :=~& \sum_{j} \theta_j \|f_j - f_{\theta}\|^2 
\end{align}
$$    H(\theta, g) := \tilde{R}(\theta, g) + \mu V(\theta) 
    + \frac{\beta}{n} K(\theta)$$
and define:
\begin{align}
    \theta^* =~& \arg\min_{\theta} \tilde{R}(\theta, g_0) + \mu V(\theta) + \frac{\beta}{n} K(\theta) 
    ~=~ \arg\min_{\theta} H(\theta, g_0) \\
\end{align}
Then we introduce the shorthand notation and invoke the error decomposition assumption:
\begin{align}
    {\cal E}(\hat{g}) :=~& P(\tilde{\ell}_{\hat{\theta},g_0}(Z) - \tilde{\ell}_{\theta^*,g_0}(Z)) - P(\tilde{\ell}_{\hat{\theta},\hat{g}}(Z) - \tilde{\ell}_{\theta^*,\hat{g}}(Z)) ~\leq~ \error(\hat{g})\, \|f_{\hat{\theta}} - f_{\theta^*}\|^{\gamma}
\end{align}

\begin{lemma}\label{lem:prel-strong-conv} If $\ell(Z; f(X), g)$ is $\sigma(g)$-strongly convex in expectation with respect to $f$, then for any $\theta,\theta'\in \Theta$:
\begin{align}\label{eqn:strong-conv-R-lemma}
    R(\theta, g)-R(\theta', g) \geq \langle \nabla R(\theta', g), \theta - \theta' \rangle + \frac{\sigma(g)}{2}\|f_{\theta}-f_{\theta'}\|^2
\end{align}
\end{lemma}
\begin{proof}
Let $F$ be the random column vector defined as $F := [f_{1}(X)\cdots f_{M}(X)]$. Note that by the chain rule:
\begin{align}
    \langle \nabla_{\theta} \ell(Z;f_\theta'(X), g), \theta - \theta' \rangle & = \langle (\nabla_{\theta} (f_{\theta'}(X)))^\top \nabla_{f} \ell(Z;f_{\theta'}(X), g), \theta - \theta' \rangle \\
    & =  \langle F \nabla_{f} \ell(Z;f_{\theta'}(X), g), \theta - \theta' \rangle \\
    & = \nabla_{f} \ell(Z;f_{\theta'}(X), g))^\top F^\top (\theta-\theta') \\
    & = \langle \nabla_{f} \ell(Z;f_{\theta'}(X), g), F^\top (\theta-\theta') \rangle \\ 
    & = \langle \nabla_{f} \ell(Z;f_{\theta'}(X), g), f_\theta(X)-f_\theta'(X) \rangle
\end{align}
Taking expectation over $Z$, we have:
\begin{align}
    \langle \nabla_{\theta} R(\theta'), \theta - \theta' \rangle =~& \langle \nabla_{\theta} \E\left[\ell(Z;f_\theta'(X), g)\right], \theta - \theta' \rangle\\
    =~& \E\left[\langle \nabla_{\theta} \ell(Z;f_\theta'(X), g), \theta - \theta' \rangle\right]\\
    =~& \E\left[\langle\nabla_{f} \ell(Z;f_{\theta'}(X), g), f_\theta(X)-f_\theta'(X)\rangle\right]
\end{align}
The result then follows by the definition of $\sigma(g)$-strong convexity in expectation of the function $\ell(f,g)$ and the definition of $R(\theta,g)$.
\end{proof}

\begin{lemma}\label{lem:v-versus-tilde} 
If $R(\theta, g)$ is $\sigma(g)$-strongly convex in expectation, then for any $\theta\in \Theta$:
\begin{equation}
    \tilde{R}(\theta, g) \geq R(\theta, g) + \frac{\nu \sigma(g)}{2} V(\theta)
\end{equation}
\end{lemma}
\begin{proof}
By $\sigma$-strong convexity of the loss and Lemma~\ref{lem:prel-strong-conv}, we have that:
\begin{align}
    R(e_j, g) - R(\theta, g) \geq \langle \nabla R(\theta, g), e_j - \theta\rangle + \frac{\sigma(g)}{2}\|f_j - f_\theta\|^2
\end{align}
Multiplying each equation by $\theta_j$ and summing across $j$, yields:
\begin{align}
    R(\theta, g) \leq \sum_j \theta_j R(e_j, g) - \frac{\sigma(g)}{2} V(\theta)
\end{align}
where we used the fact that $\sum_{j}\theta_j=1$ and $\sum_{j} \theta_j \langle \nabla R(\theta, g), e_j - \theta\rangle=0$.

Thus we have that:
\begin{align}
    \tilde{R}(\theta, g) =~& (1-\nu) R(\theta, g) + \nu \sum_j \theta_j R(e_j, g)\\
    \geq~& (1-\nu) R(\theta, g) + \nu R(\theta, g) + \frac{\nu\sigma(g)}{2} V(\theta)\\
    =~& R(\theta, g) + \frac{\nu\sigma(g)}{2} V(\theta)
\end{align}
\end{proof}

\begin{lemma}\label{lem:strong-conv-tilde}
If $\ell(Z; f(X), g_0)$ is $\sigma$-strongly convex in $f$, in expectation, and $\theta_*=\argmin_{\theta\in \Theta} H(\theta, g_0)$, then for any $\theta\in \Theta$:
\begin{align}
    H(\theta, g_0) - H(\theta^*, g_0) \geq \left(\frac{(1-\nu)\sigma}{2} - \mu\right) \|f_{\theta} - f_{\theta^*}\|^2
\end{align}
\end{lemma}
\begin{proof}
By Lemma~\ref{lem:prel-strong-conv} we have:
\begin{align}\label{eqn:strong-conv-R}
    R(\theta, g_0)-R(\theta', g_0) \geq \langle \nabla R(\theta', g_0), \theta - \theta' \rangle + \frac{\sigma}{2}\|f_{\theta}-f_{\theta'}\|^2
\end{align}
Now let $q(Z;\theta, g) = \sum_{j=1}^{M}\theta_{j} \ell(Z;f_{e_{j}}(X), g)$ and $L_e(g)=[\ell(Z;f_{e_{1}}(X),g)\cdots \ell(Z;f_{e_{M}}(X), g)]^{T}$ and let $Q(\theta, g) = \E\left[q(Z;\theta, g)\right]$. We have that:
\begin{align}
    q(Z;\theta, g)-q(Z; \theta', g) & =\sum_{j=1}^{M}\ell(Z;f_{e_{j}}(X), g)(\theta_{j}-\theta'_{j})
    = \langle{L_{e}(g), \theta-\theta'}\rangle
    = \langle \nabla q(Z; \theta', g), \theta-\theta' \rangle
\end{align}
By taking the expectation over $Z$:
\begin{align}\label{eqn:strong-conv-G}
    Q(\theta, g) - Q(\theta', g) = \sum_{j=1}^{M}\theta_{j}R(e_j, g)-\sum_{j=1}^{M}\theta_{j}'R(e_j, g)=\langle \nabla Q(\theta', g), \theta-\theta' \rangle
\end{align}
By adding Equations~\eqref{eqn:strong-conv-R} and \eqref{eqn:strong-conv-G} we have that: 
\begin{align}
    \tilde{R}(\theta, g_0)-\tilde{R}(\theta', g_0) =~& (1 - \nu) (R(\theta, g_0)-R(\theta', g_0)) + \nu (Q(\theta, g_0) - Q(\theta', g_0))\\
    \geq~& (1-\nu) \langle \nabla R(\theta',g_0), \theta - \theta' \rangle + \nu \langle \nabla Q(\theta',g_0), \theta-\theta' \rangle + \frac{(1-\nu) \sigma}{2}\|f_{\theta}-f_{\theta'}\|^2 \\
    \geq~& \langle \nabla \tilde{R}(\theta',g_0), \theta - \theta' \rangle+ \frac{(1-\nu) \sigma}{2}\|f_{\theta}-f_{\theta'}\|^2 
\end{align}
 meaning that $\tilde{R}$ will also be $(1-\nu)\sigma$-strongly convex in $\theta$. We also have that for any $\theta, \theta'$:\footnote{Since the function $V(\theta)$ can be interpreted as the expected variance of $f_j(X)$ over the distribution of $j$. Thus with simple manipulations it can be shown to be equal to $V(\theta)=\sum_j \theta_j \|f_j\|^2- \|f_\theta\|^2$. Since this is a quadratic in $\theta$ function, with Hessian $\E[FF']$, where $F$ is the vector $(f_1(X), \ldots, f_j(X))$, the above equality follows by an exact second order Taylor expansion and the observation that $(\theta-\theta')^\top\E[FF'](\theta-\theta') = \|f_\theta - f_{\theta'}\|^2$}
\begin{align}
    V(\theta)-V(\theta') = \langle \nabla V(\theta'), \theta - \theta' \rangle - \|f_{\theta}-f_{\theta'}\|_{2}^2
\end{align}
which implies that $V(\theta)$ is at most 2-strongly concave. 
We also have:
\begin{align}
    K(\theta) - K(\theta') = \sum_j (\theta_j - \theta'_j) \log(1/\pi_j) = \langle \nabla K(\theta'), \theta - \theta' \rangle
\end{align}

Adding the three inequalities and substituting $\theta'$ with $\theta^*$ we have:
\begin{align}
        H(\theta, g_0) - H(\theta^*, g_0) \geq \langle \nabla H(\theta^*, g_0), \theta-\theta^* \rangle + \left(\frac{(1-\nu)\sigma}{2} - \mu\right) \|f_{\theta} - f_{\theta^*}\|^2
\end{align}
Note that by the optimality of $\theta^*$ over the convex set $\Theta$ we have that $\langle \nabla H(\theta^*, g_0), \theta-\theta^* \rangle \geq 0$, which immediately yields the statement we want to prove.
\end{proof}

\begin{lemma}\label{lem:main}
Define the centered and offset empirical processes:
\begin{align}
    Z_{n,1} =~& (1-\nu)(P-P_n) (\ell_{\hat{\theta}, \hat{g}}(Z) - \ell_{\theta^*, \hat{g}}(Z)) - \mu \sum_{j} \hat{\theta}_j \|f_j - f_{\theta^*}\|^2 {-\frac{1}{s}K(\hat{\theta})}\\
    Z_{n,2} =~& \nu(P-P_n) \sum_{j} (\hat{\theta}_j-\theta_j^*)\ell_{e_j, \hat{g}}(Z) - \mu \hat{\theta}^\top H \theta^* {-\frac{1}{s}K(\hat{\theta})} & H_{ij} = \|f_i - f_j\|^2 
\end{align}
where $s$ is a positive number. Assume that:
\begin{align}
    {\cal E}(\hat{g}) := P(\tilde{\ell}_{\hat{\theta},g_0}(Z) - \tilde{\ell}_{\theta^*,g_0}(Z)) - P(\tilde{\ell}_{\hat{\theta},\hat{g}}(Z) - \tilde{\ell}_{\theta^*,\hat{g}}(Z)) \leq \error(\hat{g})\, \|f_\theta - f_{\theta^*}\|
\end{align}
and $\ell(Z; f(X), g_0)$ is $\sigma$-strongly convex in expectation with respect to $f$. For {$\beta \geq \frac{4n}{s}$} and $\mu \leq \min\{\nu, 1-\nu\} \frac{\sigma}{14}$:
\begin{align}
    R(\hat{\theta}) \leq \min_{j} [R(e_j, g_0){+ \frac{\beta}{n} \log(1/\pi_j)}] + 2\, (Z_{n,1} + Z_{n,2}) + \frac{1}{\mu}\error(\hat{g})^2 \label{eqn:lemma9}
\end{align}
\end{lemma}
\begin{proof}
\begin{align}
    \tilde{R}(\hat{\theta}, g_0) - \tilde{R}(\theta^*, g_0) =~& 
    P(\tilde{\ell}_{\hat{\theta},g_0}(Z) - \tilde{\ell}_{\theta^*,g_0}(Z))\\
    =~& 
    P(\tilde{\ell}_{\hat{\theta},\hat{g}}(Z) - \tilde{\ell}_{\theta^*,\hat{g}}(Z)) + {\cal E}(\hat{g})\\
    \leq~& 
    (P - P_n)(\tilde{\ell}_{\hat{\theta}, \hat{g}}(Z) - \tilde{\ell}_{\theta^*, \hat{g}}(Z)) + {\cal E}(\hat{g}) 
    {-\frac{\beta}{n}(K(\hat{\theta})-K(\theta^*))}
    \tag{By optimality of $\hat{\theta}$, $P_n(\tilde{\ell}_{\hat{\theta}, \hat{g}}(Z) - \tilde{\ell}_{\theta^*, \hat{g}}(Z)) {+ \frac{\beta}{n}(K(\hat{\theta})-K(\theta^*))}\leq 0$}\\
    =~& \mu\sum_j \hat{\theta}_j \|f_j - f_{\theta^*}\|^2 + \mu \hat{\theta}^\top H\theta^*+ Z_{n,1} + Z_{n,2} + {\cal E}(\hat{g})\\
    ~& {-\frac{\beta}{n}(K(\hat{\theta})-K(\theta^*)) + \frac{2}{s}K(\hat{\theta})}
    \tag{definition of $Z_{n,1}, Z_{n,2}$}
\end{align}
By simple algebraic manipulations note that for any $\theta\in \Theta$:
\begin{align}
    \sum_j \hat{\theta}_j \|f_j - f_{\theta^*}\|^2 = V(\hat{\theta}) + \|f_{\hat{\theta}} - f_{\theta^*}\|^2 \\
    \hat{\theta}^\top H \theta^* = V(\hat{\theta}) + V(\theta^*) + \|f_{\hat{\theta}} - f_{\theta^*}\|^2
\end{align}
Letting $Z_n=Z_{n,1} + Z_{n,2}$, we have:
\begin{align}
    \tilde{R}(\hat{\theta}, g_0) - \tilde{R}(\theta^*, g_0) \leq~& 2\mu V(\hat{\theta}) + \mu V(\theta^*) + 2\mu \|f_{\hat{\theta}} - f_{\theta^*}\|^2 + Z_n + {\cal E}(\hat{g}) {-\frac{\beta}{n}(K(\hat{\theta})-K(\theta^*)) + \frac{2}{s}K(\hat{\theta})} \label{eqn:exp_epsilon}\\
    \leq~& 2\mu V(\hat{\theta}) + \mu V(\theta^*) + 2\mu \|f_{\hat{\theta}} - f_{\theta^*}\|^2 + Z_n + \error(\hat{g})\, \|f_{\hat{\theta}} - f_{\theta^*}\|^{\gamma}\\
    ~& {-\frac{\beta}{n}(K(\hat{\theta})-K(\theta^*)) + \frac{2}{s}K(\hat{\theta})}
\end{align}
By Young's inequality, for any $a\geq 0$ and $b\geq 0$, we have $a\,b\leq \frac{a^p}{p} + \frac{b^q}{q}$, such that $\frac{1}{p}+\frac{1}{q}=1$. Taking $a= \left(\frac{\gamma}{2\mu}\right)^{\gamma/2} \error(\hat{g})$ and $b=\left(\frac{2\mu}{\gamma}\right)^{\gamma/2}\|f_{\hat{\theta}}-f_{\theta^*}\|^\gamma$ and $q=2/\gamma$ and $p=\frac{1}{1-1/q}=\frac{1}{1-\gamma/2}=\frac{2}{2-\gamma}$
\begin{align}
    \error(\hat{g})\, \|f_{\hat{\theta}} - f_{\theta^*}\|^{\gamma} \leq \underbrace{\frac{2}{2-\gamma}\left(\frac{\gamma}{2\mu}\right)^{\frac{\gamma}{2-\gamma}} \error(\hat{g})^{\frac{2}{2-\gamma}}}_{{\cal A}(\hat{g})}+ \mu \|f_{\hat{\theta}} - f_{\theta^*}\|^2
\end{align}
Thus we can derive:
\begin{align}
    \tilde{R}(\hat{\theta}, g_0) - \tilde{R}(\theta^*, g_0) \leq~& 2\mu V(\hat{\theta}) + \mu V(\theta^*) + 3\mu \|f_{\hat{\theta}} - f_{\theta^*}\|^2 + Z_n + {\cal A}(\hat{g}) \\\label{eqn:first-upper}
    ~& {-\frac{\beta}{n}(K(\hat{\theta})-K(\theta^*)) + \frac{2}{s}K(\hat{\theta})}
\end{align}
By Lemma~\ref{lem:strong-conv-tilde}, we have:
\begin{align}
    \tilde{R}(\hat{\theta}, g_0) - \tilde{R}(\theta^*, g_0) \geq~& \mu (V(\theta^*) - V(\hat{\theta})) {+\frac{\beta}{n}(K(\theta^*)-K(\hat{\theta}))} + \left(\frac{(1-\nu)\sigma}{2} - \mu\right)\|f_{\hat{\theta}} - f_{\theta^*}\|^2\\
    \geq~& \mu (V(\theta^*) - V(\hat{\theta})) {+\frac{\beta}{n}(K(\theta^*)-K(\hat{\theta}))} + 6\mu  \|f_{\hat{\theta}} - f_{\theta^*}\|^2  \tag{Since: $\mu \leq \frac{(1-\nu)\sigma}{14}\implies \frac{(1-\nu)\sigma}{2} \geq 7\mu$}
\end{align}
Combining the latter upper and lower bound on $\tilde{R}(\hat{\theta}, g_0) - \tilde{R}(\theta^*, g_0)$, we have:
\begin{align}
    3\mu \|f_{\hat{\theta}} - f_{\theta^*}\|^2 \leq 3\mu V(\hat{\theta}) + Z_n {+ \frac{2}{s} K(\hat{\theta})} + {\cal A}(\hat{g})
\end{align}
Plugging the bound on $3\mu\|f_{\hat{\theta}} - f_{\theta^*}\|^2$ back into Equation~\eqref{eqn:first-upper} we have:
\begin{align}
    \tilde{R}(\hat{\theta}, g_0) - \tilde{R}(\theta^*, g_0) \leq ~& 5\mu V(\hat{\theta}) + \mu V(\theta^*) + 2\, Z_n + 2\,{\cal A}(\hat{g}) \\
    ~& {+\frac{\beta}{n}(K(\theta^*)-K(\hat{\theta})) + \frac{4}{s}K(\hat{\theta})}
\end{align}
By Lemma~\ref{lem:v-versus-tilde}, we have:
\begin{align}
    R(\hat{\theta}, g_0) \leq~& \tilde{R}(\theta^*, g_0) + \mu V(\theta^*)+ \left(5\mu - \frac{\nu\sigma(g_0)}{2}\right) V(\hat{\theta}) + 2 Z_n + 2{\cal A}(\hat{g}) \\
    ~& {+\frac{\beta}{n}K(\theta^*) + \left(\frac{4}{s} - \frac{\beta}{n}\right)K(\hat{\theta})}\\
    \leq~& \tilde{R}(\theta^*, g_0) + \mu V(\theta^*) {+ \frac{\beta}{n}K(\theta^*)} + 2 Z_n + 2{\cal A}(\hat{g})   \tag{since $\mu\leq \frac{\nu\sigma}{10}$,  $V(\hat{\theta})>0$, {$\beta \geq \frac{4n}{s}$ and $K(\hat{\theta}\geq 0 )$}}\\
    =~& \min_{\theta\in \Theta} [\tilde{R}(\theta, g_0) + \mu V(\theta){+ \frac{\beta}{n}K(\theta^*)}] + 2 Z_n + 2{\cal A}(\hat{g}) \tag{Since $\theta_*$ minimizes $H(\theta, g_0)$}\\
    \leq~& \min_{j} [\tilde{R}(e_j, g_0) + \mu V(e_{j}) {+ \frac{\beta}{n}\log(1/\pi_j)}]+ 2 Z_n + 2{\cal A}(\hat{g}) \\
    =~& \min_{j} [R(e_j, g_0){+ \frac{\beta}{n} \log(1/\pi_j)}] + 2 Z_n +2{\cal A}(\hat{g}) \tag{Since $\tilde{R}(e_j, g_0)=R(e_j, g_0)$ and $V(e_j)=0$ for all $e_j$}
\end{align}
\end{proof}

\subsection{Concentration of offset empirical processes}

\begin{lemma}\label{lem:z1}
Assuming that the loss $\ell(f(X), \hat{g})$ is $C_b$ lipschitz in $f(X)$ and all $f_i(X)$ are bounded in $[-b,b]$, almost surely, then for any positive $s\leq \min\left\{ \frac{\mu n}{(2C_b(1-\nu))^2}, \frac{n}{2\sqrt{3}b C_b (1-\nu)}\right\}$:
\begin{align}
    {\Pr\left(Z_{n,1} \geq \frac{\log(1/\delta)}{s}\right) \leq \delta}
\end{align}
\end{lemma}
\begin{proof}
By Lemma~\ref{lem:exp-tail} it suffices to show that:
\begin{align}
    {I:=\E\left[\exp\left\{s Z_{n,1}\right\}\right] \leq 1}
\end{align}
By symmetrization (sym.) and contraction (contr.):
\begin{align}
    I =~& \E\left[\exp\left\{s \left((1-\nu)(P-P_n)(\ell_{\hat{\theta}}(Z) - \ell_{\theta^*}(Z)) - \mu \sum_j \hat{\theta}_j \|f_j - f_{\theta^*}\|^2 {-\frac{1}{s}\sum_j\hat{\theta}_j \log(1/\pi_j)}\right)\right\}\right]\\
    \leq~& \E\left[\exp\left\{s \max_{\theta} \left((1-\nu)(P-P_n)(\ell_{\theta}(Z) - \ell_{\theta^*}(Z)) - \mu \sum_j \theta_j \|f_j - f_{\theta^*}\|^2 {+\frac{1}{s}\sum_j\theta_j \log(\pi_j)} \right)\right\}\right]\\
    \leq~& \E\left[\exp\left\{s \max_{\theta} \left(2\,(1-\nu)P_{n,\epsilon}(\ell_{\theta}(Z) - \ell_{\theta^*}(Z)) - \mu \sum_j \theta_j \|f_j - f_{\theta^*}\|^2 {+\frac{1}{s}\sum_j\theta_j\log(\pi_j)}\right)\right\}\right]\tag{sym.}\\
    \leq~& \E\left[\exp\left\{s \max_{\theta} \left(2\,(1-\nu) C_b P_{n,\epsilon}(f_{\theta}(Z) - f_{\theta^*}(Z)) - \mu \sum_j \theta_j \|f_j - f_{\theta^*}\|^2 +\frac{1}{s}\sum_j\theta_j \log(\pi_j)\right)\right\}\right] \tag{contr.}
\end{align}
Since the quantity:
\begin{align}
   2\,(1-\nu) C_b P_{n,\epsilon}(f_{\theta}(Z) - f_{\theta^*}(Z)) - \mu \sum_j \theta_j \|f_j - f_{\theta^*}\|^2 +\frac{1}{s}\sum_j\theta_j \log(\pi_j)
\end{align}
is a linear function of $\theta$ (recall that $f_{\theta}=\sum_j \theta_j f_j$), the maximum over $\theta$ is attained at one of $e_1, \ldots, e_M$. Thus:
\begin{align}
    I \leq ~&\E\left[\exp\left\{s \max_{j} \left(2\,(1-\nu) C_b P_{n,\epsilon}(f_{j}(Z) - f_{\theta^*}(Z)) - \mu \|f_j - f_{\theta^*}\|^2{+\frac{1}{s}\log(\pi_j)}\right)\right\}\right]\\
    \leq~& \E\left[\max_{j}{\pi_j} \exp\left\{s \left(2\,(1-\nu) C_b P_{n,\epsilon}(f_{j}(Z) - f_{\theta^*}(Z)) - \mu \|f_j - f_{\theta^*}\|^2\right)\right\}\right]\\
    \leq~& \E\left[\sum_{j} {\pi_j} \exp\left\{s \left(2\,(1-\nu) C_b P_{n,\epsilon}(f_{j}(Z) - f_{\theta^*}(Z)) - \mu \|f_j - f_{\theta^*}\|^2\right)\right\}\right]\\
    =~&\sum_{j} {\pi_j} \E\left[ \exp\left\{s \left(2\,(1-\nu) C_b P_{n,\epsilon}(f_{j}(Z) - f_{\theta^*}(Z)) - \mu \|f_j - f_{\theta^*}\|^2\right)\right\}\right]\\
    =~& \sum_{j} {\pi_j}\E_{Z_{1:n}}\E_\epsilon\left[ \exp\left\{s \left(2\,(1-\nu) C_b P_{n,\epsilon}(f_{j}(Z) - f_{\theta^*}(Z)) - \mu \|f_j - f_{\theta^*}\|^2\right)\right\}\right]\\
    =~& \sum_{j} {\pi_j}\E_{Z_{1:n}}\left[ \E_\epsilon\left[ \exp\left\{s 2\,(1-\nu) C_b P_{n,\epsilon}(f_{j}(Z) - f_{\theta^*}(Z))\right\}\right] \exp\left\{- s \mu \|f_j - f_{\theta^*}\|^2\right\}\right]
\end{align}
By Hoeffding's inequality in Lemma~\ref{lem:hoeffding}, we can bound the first term in the product:
\begin{align}
    \E_\epsilon\left[ \exp\left\{s 2\,(1-\nu) C_b P_{n,\epsilon}(f_{j}(Z) - f_{\theta^*}(Z))\right\}\right]\leq \exp\left\{\frac{2 s^2}{n}\,(1-\nu)^2 C_b^2 P_n(f_{j}(Z) - f_{\theta^*}(Z))^2\right\}
\end{align}
Thus we have:
\begin{align}
    I\leq~&\sum_{j} {\pi_j}\E_{Z_{1:n}}\left[ \exp\left\{\frac{2s^2}{n}\,(1-\nu)^2 C_b^2 P_n(f_{j}(Z) - f_{\theta^*}(Z))^2\right\}\, \exp\left\{- s \mu \|f_j - f_{\theta^*}\|^2\right\}\right]\\
    =~& \sum_{j} {\pi_j}\E_{Z_{1:n}}\left[ \exp\left\{\frac{2s^2}{n}\,(1-\nu)^2 C_b^2 P_n(f_{j}(Z) - f_{\theta^*}(Z))^2\right\}\, \exp\left\{- s \mu P (f_j(Z) - f_{\theta^*}(Z))^2\right\}\right]\\
    =~& \sum_{j} {\pi_j} \E_{Z_{1:n}}\left[ \exp\left\{\frac{2s^2\,(1-\nu)^2 C_b^2}{n} \left(P_n - \frac{\mu\, n}{2 s(1-\nu)^2 C_b^2}P\right)(f_{j}(Z) - f_{\theta^*}(Z))^2\right\}\right]
\end{align}
Since $s$ is small enough such that: $\frac{\mu\, n}{2 s (1-\nu)^2 C_b^2}\geq 2$ (which is satisfied by our assumptions on $s$), we have that:
\begin{align}
    I \leq~& \sum_{j} \pi_j \E_{Z_{1:n}}\left[ \exp\left\{\frac{2s^2\,(1-\nu)^2 C_b^2}{n} \left(P_n - 2P\right)(f_{j}(Z) - f_{\theta^*}(Z))^2\right\}\right]\\
    =~& \sum_{j} \pi_j\E_{Z_{1:n}}\left[ \exp\left\{\frac{2s^2\,(1-\nu)^2 C_b^2}{n} \left(\left(P_n - P\right)(f_{j}(Z) - f_{\theta^*}(Z))^2 - P (f_{j}(Z) - f_{\theta^*}(Z))^2 \right)\right\}\right]
\end{align}
Since $f_j(Z)\in [-b, b]$, we have that $f_j(Z) - f_{\theta^*}(Z)\in [-2b, 2b]$. Thus:
\begin{align}
(f_j(Z)-f_{\theta^*}(Z))^2\geq \frac{(f_j(Z)-f_{\theta^*}(Z))^2}{4b^2} (f_j(Z)-f_{\theta^*}(Z))^2 = \frac{1}{4b^2} (f_j(Z)-f_{\theta^*}(Z))^4.
\end{align}
Thus:
\begin{align}
    I \leq~& \sum_{j} \pi_j \E_{Z_{1:n}}\left[ \exp\left\{\frac{2s^2\,(1-\nu)^2 C_b^2}{n} \left(\left(P_n - P\right)(f_{j}(Z) - f_{\theta^*}(Z))^2 - \frac{1}{4b^2} P (f_{j}(Z) - f_{\theta^*}(Z))^4 \right)\right\}\right]
\end{align}
It suffices to show that each of the summands is upper bounded by $1$. Then their weighted average is also upper bounded by $1$.

Applying Lemma~\ref{lem:bernstein} with:
\begin{align}
    V_i =~& (f_j(Z_i) - f_{\theta^*}(Z_i))^2\\
    c =~& 4b^2\\
    \lambda =~& \frac{2s^2(1-\nu)^2C_b^2}{n^2}\\
    c_0 =~& \frac{1}{4b^2}
\end{align}
we have that as long as:
\begin{align}
    \lambda := \frac{2s^2(1-\nu)^2C_b^2}{n^2} \leq \frac{2c_0}{1+2c_0 c} =: \frac{2\cdot 1/4b^2}{1 + 2} = \frac{1}{6b^2}
\end{align}
Then:
\begin{align}
    \E_{Z_{1:n}}\left[ \exp\left\{\frac{2s^2\,(1-\nu)^2 C_b^2}{n} \left(\left(P_n - P\right)(f_{j}(Z) - f_{\theta^*}(Z))^2 - \frac{1}{4b^2} P (f_{j}(Z) - f_{\theta^*}(Z))^4 \right)\right\}\right]\leq 1
\end{align}
Thus we need that:
\begin{align}
    s \leq \frac{n}{2\sqrt{3} b (1-\nu) C_b}
\end{align}
which is satisfied by our assumptions on $s$.
\end{proof}

\begin{lemma}\label{lem:z2}
Assuming that the loss $\ell(f(X), \hat{g})$ is $C_b$ lipschitz in $f(X)$ and bounded in $[-b,b]$, for any positive $s\leq \frac{2\mu n}{C_b \nu (\nu C_b + 4\mu b)}$:
\begin{align}
    \Pr\left(Z_{n,2} \geq \frac{\log(1/\delta)}{s}\right) \leq \delta
\end{align}
\end{lemma}
\begin{proof}
It suffices to show that:
\begin{align}
    I := \E\left[\exp\left\{s Z_{n,2})\right\}\right] \leq 1
\end{align}
By definition of $Z_{n,2}$:
\begin{align}
    I =~& \E\left[\exp\left\{s \left( \nu(P-P_n) \sum_{j} (\hat{\theta}_j-\theta_j^*)\ell_{e_j}(Z) - \mu \hat{\theta}^\top H \theta^*  -\frac{1}{s}\sum_j\hat{\theta}_j \log(1/\pi_j) \right)\right\}\right]\\
    =~& \E\left[\exp\left\{s \left( \nu(P-P_n) \sum_{j} \hat{\theta}_j \ell_{e_j}(Z) + \frac{1}{s}\sum_j\hat{\theta}_j \log(\pi_j)\right)\right\} \exp\left\{s\left( - \nu (P-P_n)\sum_{j} \theta_j^*\ell_{e_j}(Z) - \mu \hat{\theta}^\top H \theta^* \right)\right\}\right]
\end{align}
Note that:
\begin{align}
    - \nu (P-P_n)\sum_{j} \theta_j^*\ell_{e_j}(Z) - \mu \hat{\theta}^\top H \theta^* =~& \sum_{j} \theta_j^* \left( -\nu (P-P_n) \ell_{e_j}(Z) - \mu \hat{\theta}^\top H_{\cdot, j}\right) \\
    =~& \sum_{j} \theta_j^* \underbrace{\left( -\nu (P-P_n) \ell_{e_j}(Z) - \mu \sum_{k} \hat{\theta}_k \|f_k - f_j\|^2\right)}_{\Phi_j}
\end{align}
Thus:
\begin{align}
    I=~&  \E\left[\exp\left\{s \left( \nu(P-P_n) \sum_{j} \hat{\theta}_j \ell_{e_j}(Z)  +\frac{1}{s}\sum_j\hat{\theta}_j \log(\pi_j)\right)\right\} \exp\left\{\sum_{j}\theta_j^* s\Phi_j\right\}\right]
\end{align}
By Jensen's inequality:
\begin{align}
    \exp\left\{\sum_{j}\theta_j^* s\Phi_j\right\} \leq \sum_{j}\theta_j^* \exp\left\{s\Phi_j\right\}
\end{align}
Thus:
\begin{align}
    I\leq~& \sum_{j}\theta_j^*  \E\left[\exp\left\{s \left( \nu(P-P_n) \sum_{k} \hat{\theta}_k \ell_{e_k}(Z){+\frac{1}{s}\sum_k \hat{\theta}_k \log(\pi_k)}\right)\right\} \exp\left\{ s\Phi_j\right\}\right]\\
    =~& \sum_{j}\theta_j^* \E\left[\exp\left\{s \left( \nu(P-P_n) \sum_{k} \hat{\theta}_k \ell_{e_k}(Z) + \Phi_j {+\frac{1}{s}\sum_k\hat{\theta}_k \log(\pi_k)}\right)\right\}\right]\\
    =~& \sum_{j}\theta_j^*  \E\left[\exp\left\{s \left( \nu(P-P_n) \left(\sum_{k} \hat{\theta}_k \ell_{e_k}(Z) - \ell_{e_j}(Z)\right) - \mu \sum_{k} \hat{\theta}_k \|f_k - f_j\|^2{+\frac{1}{s}\sum_k\hat{\theta}_k \log(\pi_k)}\right)\right\}\right]\\
    =~& \sum_{j}\theta_j^*  \E\left[\exp\left\{s \left( \nu(P-P_n) \sum_{k} \hat{\theta}_k \left(\ell_{e_k}(Z) - \ell_{e_j}(Z)\right) - \mu \sum_{k} \hat{\theta}_k \|f_k - f_j\|^2 + \frac{1}{s}\sum_k\hat{\theta}_k \log(\pi_k)\right)\right\}\right]\\
    \leq~& \sum_{j}\theta_j^* \E\left[\exp\left\{s \max_{\theta} \left( \nu(P-P_n) \sum_{k} \theta_k \left(\ell_{e_k}(Z) - \ell_{e_j}(Z)\right) - \mu \sum_{k} \theta_k \|f_k - f_j\|^2 + \frac{1}{s}\sum_k\theta_k 
    \log(\pi_k)\right)\right\}\right]
\end{align}
Since the objective in the exponent is linear in $\theta$ it takes its value at one of the extreme points $e_1, \ldots, e_M$:
\begin{align}
    I \leq~& \sum_{j}\theta_j^*  \E\left[\exp\left\{s \max_{k} \left( \nu(P-P_n) \left(\ell_{e_k}(Z) - \ell_{e_j}(Z)\right) - \mu \|f_k - f_j\|^2+\frac{1}{s} \log(\pi_k)\right)\right\}\right]\\
    =~& \sum_{j}\theta_j^*  \E\left[\max_{k}{\pi_k}\exp\left\{s  \left( \nu(P-P_n) \left(\ell_{e_k}(Z) - \ell_{e_j}(Z)\right) - \mu \|f_k - f_j\|^2\right)\right\}\right]\\
    \leq~& \sum_{j}\theta_j^* \E\left[\sum_{k}{\pi_k} \exp\left\{s  \left( \nu(P-P_n) \left(\ell_{e_k}(Z) - \ell_{e_j}(Z)\right) - \mu \|f_k - f_j\|^2\right)\right\}\right]\\
    \leq~& \sum_{j}\theta_j^* \sum_{k} {\pi_k}\E\left[ \exp\left\{s  \left( \nu(P-P_n) \left(\ell_{e_k}(Z) - \ell_{e_j}(Z)\right) - \mu \|f_k - f_j\|^2\right)\right\}\right]
\end{align}
By Lipschitzness of $\ell$, we have that:
\begin{align}
    \|\ell_{e_k} - \ell_{e_j}\|^2 \leq C_b^2 \|f_k - f_j\|^2
\end{align}
Thus:
\begin{align}
    I \leq~& \sum_{j}\theta_j^* \sum_{k} {\pi_k}\E\left[ \exp\left\{s  \left( \nu(P-P_n) \left(\ell_{e_k}(Z) - \ell_{e_j}(Z)\right) - \frac{\mu}{C_b^2} \|\ell_{e_k} - \ell_{e_j}\|^2\right)\right\}\right] \\
    =~& \sum_{j}\theta_j^* \sum_{k} {\pi_k}\E\left[ \exp\left\{s  \left( \nu(P-P_n) \left(\ell_{e_k}(Z) - \ell_{e_j}(Z)\right) - \frac{\mu}{C_b^2} P(\ell_{e_k}(Z) - \ell_{e_j}(Z))^2\right)\right\}\right]\\
    =~& \sum_{j}\theta_j^* \sum_{k}{\pi_k} \E\left[ \exp\left\{s \nu \left((P-P_n) \left(\ell_{e_k}(Z) - \ell_{e_j}(Z)\right) - \frac{\mu}{\nu\, C_b^2} P(\ell_{e_k}(Z) - \ell_{e_j}(Z))^2\right)\right\}\right]
\end{align}
Invoking Lemma~\ref{lem:bernstein} with:
\begin{align}
    V_i =~& \ell_{e_j}(Z_i) - \ell_{e_k}(Z_i)\\
    c_0 =~& \frac{\mu}{\nu C_b^2}\\
    \lambda = \frac{s \nu}{n}\\
    c=~& 2b C_b
\end{align}
Hence we need:
\begin{align}
    \lambda := \frac{s\nu}{n} \leq \frac{2c_0}{1 + 2c_0c} =: \frac{2 \mu}{\nu C_b^2 (1+ 2 \frac{\mu}{\nu C_b^2} 2 b C_b)} = \frac{2 \mu}{(\nu C_b^2+ 4 \mu b C_b )} \implies s \leq \frac{2 \mu n}{\nu C_b (\nu C_b+ 4 \mu b)}
\end{align}
\end{proof}

\subsection{Putting it all together}

With the choice of $\beta$ and $\mu$ as described in the theorem statement, we can invoke the above Lemmas. Letting $C_b(\hat{g})$ the Lipschitz constant associated with an estimate $\hat{g}$, we have that w.p. $1-\delta$:
\begin{align}
    R(\hat{\theta}) \leq \min_{j}\left( R(e_j)  +\frac{\beta}{n}\log(1/\pi_j)\right) + O\left(  \frac{\log(1/\delta) \max\{\frac{1}{\mu} C_b(\hat{g})^2, b\, C_b(\hat{g}) \}}{n}\right) + \frac{2}{\mu} \error(\hat{g})^2
\end{align}

\subsection{Basic concentration proofs}

\begin{lemma}\label{lem:exp-tail}
For any random variable $Z$, if $\E\left[\exp\left\{ Z\right\} \right] \leq 1$, then $\Pr(Z \geq \log(1/\delta)) \leq \delta$.
\end{lemma}
\begin{proof}
Suppose that one shows that:
\begin{align}
    \E\left[\exp\left\{ Z\right\} \right] \leq 1
\end{align}
Then by Markov's inequality:
\begin{align}
    \Pr(Z \geq x) = \Pr(\exp(Z) \geq \exp(x)) \leq \frac{\E[\exp(Z)]}{\exp(x)}\leq \frac{1}{\exp(x)} = \exp(-x)
\end{align}
Thus:
\begin{align}
    \Pr(Z\geq \log(1/\delta)) \leq \delta
\end{align}
\end{proof}

\begin{lemma}\label{lem:hoeffding}
Suppose $\epsilon_1, \ldots, \epsilon_n$ are independent Rademacher random variables drawn equiprobably in $\{-1, 1\}$ and that $a_1,\ldots, a_n$ are real numbers. Then:
\begin{align}
    \E_{\epsilon_{1:n}}\left[\exp\left\{\frac{1}{n} \sum_i \epsilon_i a_i \right\}\right] \leq \exp\left\{\frac{1}{2n^2} \sum_i a_i^2 \right\}
\end{align}
\end{lemma}
\begin{proof}
First note that for by the derivation in Example 2.2 of \url{https://www.stat.berkeley.edu/~mjwain/stat210b/Chap2_TailBounds_Jan22_2015.pdf}, we have that:
\begin{align}
    \E[\exp(\epsilon_i a_i/n)] \leq \exp(a_i^2/2n^2)
\end{align}
Invoking now independence of the $\epsilon_i$ we have:
\begin{align}
    \E_{\epsilon_{1:n}}\left[\exp\left\{\frac{1}{n}\sum_i \epsilon_i a_i \right\}\right] =~& \E_{\epsilon_{1:n}}\left[\prod_{i=1}^n \exp\left\{\epsilon_i a_i/n \right\}\right]\\
    =~& \prod_{i=1}^n \E_{\epsilon_{1:n}}\left[ \exp\left\{\epsilon_i a_i/n \right\}\right] \tag{independence}\\
    \leq~& \prod_{i=1}^n \exp\left(\frac{a_i^2}{2n^2}\right) = \exp\left\{\frac{1}{2n^2} \sum_i a_i^2 \right\}
\end{align}
\end{proof}

The following proposition can also be proven as an application of Bernstein's inequality.
\begin{lemma}[Proposition 1, \cite{Rigollet2014}]\label{lem:bernstein}
Let $V_1,\ldots, V_n$ be i.i.d. random variables, such that $|V_i|\leq c$ a.s.. Let $c_0>0$, then for any $0< \lambda < \frac{2c_0}{1+ 2c_0 c}$:
\begin{align}
    \E\left[\exp\left\{n\lambda \left( \frac{1}{n}\sum_i (V_i - \E[V_i] - c_0 \E[V_i^2]\right)\right\}\right] \leq 1  
\end{align}
\end{lemma}

\subsection{Auxiliary Lemmas}
\begin{lemma}[Symmetrization Inequality \citep{Rigollet2014}]\label{sym}
    let $\mathcal{F}$ be a function class, $A_f, f \in \mathcal{F}$ be a given function on $\mathcal{F}$ and $\Phi$ be a convex non-decreasing function then
    $$
    \mathbb{E} \Phi\left(\sup _{f \in \mathcal{F}}\left[\mathbb{E}[f(Z)]-\frac{1}{n}\sum_{i=1}^{n} f(Z_i)-A_f\right]\right) \leq \mathbb{E} \Phi\left(2 \sup _{f \in \mathcal{F}}\left[\frac{1}{n}\sum_{i=1}^{n} \varepsilon_i f(Z_i)-A_f\right]\right)
    $$
\end{lemma}

\begin{lemma}[Contraction Inequality, Theorem 4.12 in \citep{ledoux1991probability}]\label{contr}
    Let $F: \mathbb{R}_{+} \rightarrow \mathbb{R}_{+}$be convex and increasing. Let further $\varphi_i: \mathbb{R} \rightarrow \mathbb{R}, i \leq N$, be such that $|\varphi_i(s) - \varphi_i(t)|\leq |s-t|$ and $\varphi_i(0)=0$. Then, for any bounded subset $T$ in $\mathbb{R}^N$
    $$
    \mathbb{E} F\left(\frac{1}{2}\left\|\sum_{i=1}^N \varepsilon_i \varphi_i\left(t_i\right)\right\|_T\right) \leq \mathbb{E} F\left(\left\|\sum_{i=1}^N \varepsilon_i t_i\right\|_T\right)
    $$
\end{lemma}

\section{Proof of Corollaries for Square Losses} \label{sec:sq_loss}
\begin{proof}[Proof of Corollary \ref{cor:main}]
Note that if we assume that $\hat{Y}(\hat{g})$ is uniformly bounded for any $\hat{g}\in G$ with $d(\hat{g},g_0)\leq \epsilon$ for some constant $\epsilon$ and since for $n$ larger than the designated constant, we have that w.p. $1-\delta$, $\hat{Y}(\hat{g})$ is absolutely and uniformly bounded by $U$. Moreover, $f_j(X)$ is absolutely and uniformly bounded by $U$. Note that by an exact first order Taylor expansion:
\begin{align}
    |\ell(Z; f(X), g) - \ell(Z; f_*(X), g)|=~&\left|\int_{0}^1 2 (\hat{Y}(g) - (1-\tau)f(X) - \tau f_*(X)) (f(X) - f_*(X))d\tau\right|\\
    ~&\leq 2\left(\sup_{f\in F} |f(X)| + |\hat{Y}(g)|\right) |f(X) - f_*(X)|\\
    ~&\leq 2\,U\,|f(X) - f_*(X)| 
\end{align}
Thus the loss function is $C_b=2U$ lipschitz and absolutely bounded by $4U^2$. Moreover, we have that the loss function $\ell(Z; f(X), g_0)$ is $2$-strongly convex in $f$ in expectation. Thus, we can take $\sigma=2$ in Theorem~\ref{thm:main}. Thus we can apply the main theorem with $\nu=1/2$, $\mu=1/14$ and $\beta=\max\{112\,U^2, 56 U^3\}$.

Finally, we analyze the quantity ${\cal E}(\hat{g})$. We first note that for any $\theta, g$, we have:
\begin{align}
    \ell_{\theta,g}(Z) - \ell_{\theta^*, g}(Z) ~&= (\hat{Y}(g) - f_{\theta}(X))^2 - (\hat{Y}(g) - f_{\theta^*}(X))^2\\
    =~& f_{\theta}(X)^2 - f_{\theta^*}(X)^2 + 2 \hat{Y}(g)\, (f_{\theta^*}(X) - f_{\theta}(X))
\end{align}
Moreover, for any $\theta \in \Theta$:
\begin{align}
    \sum_{j} \theta_j \ell_{e_j,g}(Z) =~& \sum_{j} \theta_j (\hat{Y}(g) - f_{j}(X))^2 \\
    =~& \hat{Y}(g)^2 \sum_{j} \theta_j + \sum_{j}\theta_j f_{j}(X)^2 - 2\sum_j \theta_j \hat{Y}(g) f_j(X)\\
    =~& \hat{Y}(g)^2 + \sum_{j}\theta_j f_{j}(X)^2 - 2\sum_j \theta_j \hat{Y}(g) f_j(X) \tag{$\sum_{j} \theta_j=1$}\\
    =~& \hat{Y}(g)^2 + \sum_{j}\theta_j f_{j}(X)^2 - 2 \hat{Y}(g) f_\theta \tag{$f_\theta = \sum_{j} \theta_j f_j$}
\end{align}
Thus we have that:
\begin{align}
    \sum_{j} \theta_j \ell_{e_j,g}(Z) - \sum_{j} \theta_j^* \ell_{e_j,g}(Z) = \sum_{j}(\theta_j-\theta_j^*) f_{j}(X)^2 + 2 \hat{Y}(g) (f_{\theta^*}(X) - f_\theta(X))
\end{align}
Then by the definition of $\tilde{\ell}_{\theta, g}$, we have:
\begin{align}
    \ell_{\theta,g}(Z) - \ell_{\theta^*, g}(Z) = (1- \nu) (f_{\theta}(X)^2 - f_{\theta^*}(X)^2) + \nu \sum_{j}(\theta_j-\theta_j^*) f_{j}(X)^2 + 2\hat{Y}(g) (f_{\theta^*}(X) - f_\theta(X))
\end{align}
Only the last part depends on the parameter $g$. Thus we get:
\begin{align}
    {\cal E}(\hat{g}) =~& 2 \E[(\hat{Y}(g_0) - \hat{Y}(\hat{g}))(f_{\theta^*}(X) - f_{\hat{\theta}}(X))]
    =~ 2 \E\left[\E\left[\hat{Y}(g_0) - \hat{Y}(\hat{g})\mid X\right]\, (f_{\theta^*}(X) - f_{\hat{\theta}}(X))\right]
\end{align}
Applying a Cauchy-Schwarz inequality, we have:
\begin{align}
    {\cal E}(\hat{g}) =~&  2 \sqrt{\E\left[\E\left[\hat{Y}(g_0) - \hat{Y}(\hat{g})\mid X\right]^2\right]}\, \|f_{\theta^*}-f_{\hat{\theta}}\|_{L^2}
\end{align}
and we can take:
\begin{align}
    \error(\hat{g}) = 2\sqrt{\E\left[\E\left[\hat{Y}(g_0) - \hat{Y}(\hat{g})\mid X\right]^2\right]}
\end{align}
in Theorem~\ref{thm:main}. Finally, note that since $\E[\hat{Y}(g_0)\mid X] = \tau_0(X)$, the excess risk evaluated at the true parameters satisfies that:
\begin{align}
    R(\theta, g_0) - R(\theta', g_0) =~& \E\left[f_{\theta}(X)^2 - f_{\theta'}(X)^2 + 2\hat{Y}(g_0)(f_{\theta}(X) -f_{\theta'}(X))\right]\\
    =~& \E\left[f_{\theta}(X)^2 - f_{\theta'}(X)^2 + 2\tau_0(X)\, (f_{\theta}(X) -f_{\theta'}(X))\right]\\
    =~& \E\left[(f_{\theta}(X) - \tau_0(X))^2 - (f_{\theta'}(X) - \tau_0(X))^2\right]\\
    =~& \|f_\theta -\tau_0\|^2 - \|f_{\theta'} - \tau_0\|^2
\end{align}
Thus we have that:
\begin{align}
R(\hat{\theta}, g_0) - \min_{j} R(e_j, g_0) = \|f_{\hat{\theta}} - \tau_0\|^2 - \min_{j} \|f_j - \tau_0\|^2
\end{align}
\end{proof}

\begin{proof}[Proof of Corollary \ref{cor:main2}]
    The proof is identical to Corollary~\ref{cor:main}, but we simply replace the application of the Cauchy-Schwarz inequality with the inequality:
    \begin{align}
        2 \E\left[\E\left[\hat{Y}(g_0) - \hat{Y}(\hat{g})\mid X\right]\, (f_{\theta^*}(X) - f_{\hat{\theta}}(X))\right] \leq~& 2 \E\left[\left|\E\left[\hat{Y}(g_0) - \hat{Y}(\hat{g})\mid X\right]\right|\right]\, \|f_{\theta_*} - f_{\hat{\theta}}\|_{L^{\infty}}\\
        \leq~& 2\, C\, \E\left[\left|\E\left[\hat{Y}(g_0) - \hat{Y}(\hat{g})\mid X\right]\right|\right]\, \|f_{\theta_*} - f_{\hat{\theta}}\|_{L^{2}}^\gamma
    \end{align}
\end{proof}

\subsection{Relaxed Nuisance Guarantees Under Eigenvalue Condition}\label{app:eigenvalue}

\begin{lemma}\label{lem:RE-condition}
    Let $F$ be the random column vector defined as $F := (f_{1}(X), \cdots, f_{M}(X))$. Suppose that $\E[FF']\succeq c I$ and $\left\|\|F\|_{\infty}\right\|_{L^{\infty}} \leq U$. Then we can take $\gamma=1$ and $C = 2\, U \sqrt{\frac{M}{c}}$ in Corollary~\ref{cor:main2}.
\end{lemma}
Here, the number of models $M$ and not the logarithm of them appears. However, assuming that the nuisance error term decays faster than $n^{-1/2}$, this linear in $M$ term will only appear in the lower-order terms.

\begin{proof}
Let $F$ be the random column vector defined as $F := (f_{1}(X), \cdots, f_{M}(X))$ and assume that $\left\|\|F\|_{\infty}\right\|_{L^{\infty}} \leq U$. Noting that:
\begin{align}
    f_{\theta*}(X)-f_{\hat{\theta}}(X) = F'(\theta^* - \hat{\theta}).
\end{align}
we have:
\begin{align}
    \|f_{\theta*}(X)-f_{\hat{\theta}}(X)\|_{L^\infty}~&\leq \left\|\|F\|_{\infty}\right\|_{L^{\infty}} \|\theta^* - \hat{\theta}\|_1 \\
    ~&\leq U\, \|\theta^* - \hat{\theta}\|_1 \leq U \sqrt{M} \|\theta^* - \hat{\theta}\|_2
\end{align}
Then we if we further assume that $\E[FF']\succeq c I$, i.e. the CATE models all contain independent signal components and are not colinear. Then we can write:
\begin{align}
    \|f_{\theta*}(X)-~&f_{\hat{\theta}}(X)\|_{L^\infty} \leq U \sqrt{M} \sqrt{\|\theta^* - \hat{\theta}\|_2^2}\\
    \leq~& U \sqrt{M} \sqrt{\frac{1}{c} (\theta^* - \hat{\theta})\E[FF'](\theta^* - \hat{\theta})}\\
    =~& U \sqrt{M} \sqrt{\frac{1}{c} \E[(F'(\theta^* - \hat{\theta}))^2]}\\
    =~& U \sqrt{M} \frac{1}{\sqrt{c}} \|f_{\theta^*} - f_{\hat{\theta}}\|_{L^2}
\end{align}
Thus under these further assumptions, we can set:
\begin{align}
    C = 2\, U \sqrt{\frac{M}{c}}
\end{align}
in Corollary~\ref{cor:main2}. 
\end{proof}
\section{Proof of Main Theorem without Sample Splitting}

We will solely change the statement and proof of Lemma~\ref{lem:main} and subsequently Lemma~\ref{lem:z1} and Lemma~\ref{lem:z2} will be invoked at $g_0$ and not $\hat{g}$. The rest of the Lemmas do not depend on whether $\hat{g}$ was trained on the same samples or not as the ones that were used to construct $\hat{\theta}$. 

\begin{lemma}
Define the centered and offset empirical processes:
\begin{align}
    Z_{n,1} =~& (1-\nu)(P-P_n) (\ell_{\hat{\theta}, g_0}(Z) - \ell_{\theta^*, g_0}(Z)) - \mu \sum_{j} \hat{\theta}_j \|f_j - f_{\theta^*}\|^2 -\frac{1}{s}K(\hat{\theta})\\
    Z_{n,2} =~& \nu(P-P_n) \sum_{j} (\hat{\theta}_j-\theta_j^*)\ell_{e_j, g_0}(Z) - \mu \hat{\theta}^\top H \theta^* -\frac{1}{s}K(\hat{\theta})& H_{ij} = \|f_i - f_j\|^2 \\
    Z_{n,3} =~& \max_{j}\left[ (P-P_n) (f_{j}(X) - f_{\theta^*}(X))^2 - \mu\|f_j - f_{\theta^*}\|^2 + \frac{1}{s}\log(\pi_j)\right]\\
    Z_{n,4} =~& (P-P_n) (\hat{Y}(\hat{g}) - \hat{Y}(g_0))^2
\end{align}
Assume that:
\begin{align}
    {\cal E}(\hat{g}) := P(\tilde{\ell}_{\hat{\theta},g_0}(Z) - \tilde{\ell}_{\theta^*,g_0}(Z)) - P(\tilde{\ell}_{\hat{\theta},\hat{g}}(Z) - \tilde{\ell}_{\theta^*,\hat{g}}(Z)) \leq \error(\hat{g})\, \|f_\theta - f_{\theta^*}\|
\end{align}
and $\ell(f, g_0)$ is $1$-strongly convex in expectation with respect to $f$ since we consider square losses. For any $\mu \leq \min\{\nu, 1-\nu\} \frac{1}{14}$:
\begin{align}
    R(\hat{\theta}) \leq \min_{j} R(e_j) + 2\, (Z_{n,1} + Z_{n,2} + Z_{n,3} + Z_{n,4}) + \frac{1}{\mu}\error(\hat{g})^2
\end{align}
\end{lemma}
\begin{proof}
\begin{align}
    \tilde{R}(\hat{\theta}, g_0) - \tilde{R}(\theta^*, g_0) =~& 
    P(\tilde{\ell}_{\hat{\theta},g_0}(Z) - \tilde{\ell}_{\theta^*,g_0}(Z))\\
    =~& 
    P(\tilde{\ell}_{\hat{\theta},\hat{g}}(Z) - \tilde{\ell}_{\theta^*,\hat{g}}(Z)) + {\cal E}(\hat{g})\\
    \leq~& 
    (P - P_n)(\tilde{\ell}_{\hat{\theta}, \hat{g}}(Z) - \tilde{\ell}_{\theta^*, \hat{g}}(Z)) - \frac{\beta}{n}(K(\hat{\theta})-K(\theta^*)) + {\cal E}(\hat{g})
    \tag{By optimality of $\hat{\theta}$, $P_n(\tilde{\ell}_{\hat{\theta}, \hat{g}}(Z) - \tilde{\ell}_{\theta^*, \hat{g}}(Z)) {+ \frac{\beta}{n}(K(\hat{\theta})-K(\theta^*))}\leq 0$}\\
    =~& (P - P_n)(\tilde{\ell}_{\hat{\theta}, g_0}(Z) - \tilde{\ell}_{\theta^*, g_0}(Z)) + \Delta(\hat{g}, \hat{\theta}) - \frac{\beta}{n}(K(\hat{\theta})-K(\theta^*)) + {\cal E}(\hat{g})\\
    =~& \mu\sum_j \hat{\theta}_j \|f_j - f_{\theta^*}\|^2 + \mu \hat{\theta}^\top H\theta^*+ Z_{n,1} + Z_{n,2}\\
    ~& +\frac{2}{s}K(\hat{\theta}) - \frac{\beta}{n}(K(\hat{\theta})-K(\theta^*)) + \Delta(\hat{g}, \hat{\theta}) + {\cal E}(\hat{g}) \tag{definition of $Z_{n,1}, Z_{n,2}$}
\end{align}
where:
\begin{align}
    \Delta(\hat{g}, \hat{\theta}) = (P - P_n)\left\{\tilde{\ell}_{\hat{\theta}, \hat{g}}(Z) - \tilde{\ell}_{\theta^*, \hat{g}}(Z) - \left(\tilde{\ell}_{\hat{\theta}, g_0}(Z) - \tilde{\ell}_{\theta^*, g_0}(Z)\right)\right\}
\end{align}
For the case of square losses, the above difference in differences simplifies to:
\begin{align}
    \Delta(\hat{g}, \hat{\theta}) =~& (P - P_n)\left\{2 (\hat{Y}(\hat{g}) - \hat{Y}(g_0))\, (f_{\hat{\theta}}(X) - f_{\theta^*}(X))\right\}\\
    =~& \sum_{j} \hat{\theta}_j \left[(P - P_n)\left\{2 (\hat{Y}(\hat{g}) - \hat{Y}(g_0))\, (f_j(X) - f_{\theta^*}(X))\right\} - \frac{1}{s}\log(\frac{1}{\pi_j})\right] + \frac{1}{s}K(\hat{\theta}) \\
    =~& \mu \sum_{j} \hat{\theta}_j \|f_j - f_{\hat{\theta}^*}\|^2 +\frac{1}{s}K(\hat{\theta})\\
    ~&+ \sum_{j} \hat{\theta}_j \left[(P - P_n)\left\{2 (\hat{Y}(\hat{g}) - \hat{Y}(g_0))\, (f_j(X) - f_{\theta^*}(X))\right\} - \mu \|f_j - f_{\theta^*}\|^2+\frac{1}{s}\log(\pi_j)\right]\\
    =~& \mu \sum_{j} \hat{\theta}_j \|f_j - f_{\hat{\theta}^*}\|^2 +\frac{1}{s}K(\hat{\theta}) \\
    ~& + \max_{\theta\in \Theta} \sum_{j} \theta_j \left[(P - P_n)\left\{2 (\hat{Y}(\hat{g}) - \hat{Y}(g_0))\, (f_j(X) - f_{\theta^*}(X))\right\} - \mu \|f_j - f_{\theta^*}\|^2 + \frac{1}{s}\log(\pi_j)\right]
\end{align}
Since the last term is linear in $\hat{\theta}$ and $\Theta$ is the simplex, it is maximized at a corner, which yields:
\begin{align}
    \Delta(\hat{g}, \hat{\theta}) \leq~& \mu \sum_{j} \hat{\theta}_j \|f_j - f_{\hat{\theta}^*}\|^2 +\frac{1}{s}K(\hat{\theta})\\
    ~&+ \max_{j} \left[(P - P_n)\left\{2 (\hat{Y}(\hat{g}) - \hat{Y}(g_0))\, (f_j(X) - f_{\theta^*}(X))\right\} - \mu \|f_j - f_{\theta^*}\|^2+ \frac{1}{s}\log(\pi_j)\right]\\
    \leq~& \mu \sum_{j} \hat{\theta}_j \|f_j - f_{\hat{\theta}^*}\|^2 +\frac{1}{s}K(\hat{\theta}) + (P - P_n)(\hat{Y}(\hat{g}) - \hat{Y}(g_0))^2 \\
    ~& +\max_{j} \left[(P - P_n)(f_j(X) - f_{\theta^*}(X))^2 - \mu \|f_j - f_{\theta^*}\|^2 + \frac{1}{s}\log(\pi_j)\right]\\
    \leq~& \mu \sum_{j} \hat{\theta}_j \|f_j - f_{\hat{\theta}^*}\|^2 + Z_{n,3} + Z_{n,4} + \frac{1}{s}K(\hat{\theta})
\end{align}

By simple algebraic manipulations note that for any $\theta\in \Theta$:
\begin{align}
    \sum_j \hat{\theta}_j \|f_j - f_{\theta^*}\|^2 = V(\hat{\theta}) + \|f_{\hat{\theta}} - f_{\theta^*}\|^2 \\
    \hat{\theta}^\top H \theta^* = V(\hat{\theta}) + V(\theta^*) + \|f_{\hat{\theta}} - f_{\theta^*}\|^2
\end{align}
Letting $Z_n=Z_{n,1} + Z_{n,2} + Z_{n,3} + Z_{n,4}$, we have:
\begin{align}
    \tilde{R}(\hat{\theta}, g_0) - \tilde{R}(\theta^*, g_0) \leq~& 3\mu V(\hat{\theta}) + \mu V(\theta^*) + 3\mu \|f_{\hat{\theta}} - f_{\theta^*}\|^2 + Z_n +\frac{3}{s}K(\hat{\theta}) - \frac{\beta}{n}(K(\hat{\theta})-K(\theta^*)) + {\cal E}(\hat{g})\\
    \leq~& 3\mu V(\hat{\theta}) + \mu V(\theta^*) + 3\mu \|f_{\hat{\theta}} - f_{\theta^*}\|^2 + Z_n\\
    ~& +\frac{3}{s}K(\hat{\theta}) - \frac{\beta}{n}(K(\hat{\theta})-K(\theta^*)) + \error(\hat{g})\, \|f_{\hat{\theta}} - f_{\theta^*}\|\\
    \leq~& 3\mu V(\hat{\theta}) + \mu V(\theta^*) + 4\mu \|f_{\hat{\theta}} - f_{\theta^*}\|^2 + Z_n \\
    ~& +\frac{3}{s}K(\hat{\theta}) -\frac{\beta}{n}(K(\hat{\theta})-K(\theta^*)) + \frac{1}{\mu}\error(\hat{g})^2  \label{eqn:first-upper-no-split}
\end{align}
where the last inequality follows by an application of the AM-GM inequality. By Lemma~\ref{lem:strong-conv-tilde}, we have:
\begin{align}
    \tilde{R}(\hat{\theta}, g_0) - \tilde{R}(\theta^*, g_0) \geq~& \mu (V(\theta^*) - V(\hat{\theta})) - \frac{\beta}{n}(K(\hat{\theta})-K(\theta^*)) + \left(\frac{(1-\nu)\sigma}{2} - \mu\right)\|f_{\hat{\theta}} - f_{\theta^*}\|^2\\
    \geq~& \mu (V(\theta^*) - V(\hat{\theta})) -\frac{\beta}{n}(K(\hat{\theta})-K(\theta^*))+ 8\mu  \|f_{\hat{\theta}} - f_{\theta^*}\|^2  \tag{Since: $\sigma=2$ and $\mu \leq \frac{(1-\nu)}{18}\implies \frac{(1-\nu)\sigma}{2} \geq 9\mu$}
\end{align}
Combining the latter upper and lower bound on $\tilde{R}(\hat{\theta}, g_0) - \tilde{R}(\theta^*, g_0)$, we have:
\begin{align}
    4\mu \|f_{\hat{\theta}} - f_{\theta^*}\|^2 \leq 4\mu V(\hat{\theta}) + Z_n + \frac{3}{s}K(\hat{\theta})+ \frac{1}{\mu} \error(\hat{g})^2  
\end{align}
Plugging the bound on $4\mu\|f_{\hat{\theta}} - f_{\theta^*}\|^2$ back into Equation~\eqref{eqn:first-upper-no-split} we have:
\begin{align}
    \tilde{R}(\hat{\theta}, g_0) - \tilde{R}(\theta^*, g_0) \leq 7\mu V(\hat{\theta}) + \mu V(\theta^*) + 2\, Z_n -\frac{\beta}{n}(K(\hat{\theta})-K(\theta^*)) +\frac{6}{s}K(\hat{\theta}) + \frac{2}{\mu} \error(\hat{g})^2  
\end{align}
By Lemma~\ref{lem:v-versus-tilde}, we have:
\begin{align}
    R(\hat{\theta}, g_0) \leq~& \tilde{R}(\theta^*, g_0) + \mu V(\theta^*)+ \left(7\mu - \frac{\nu \sigma}{2}\right) V(\hat{\theta}) + 2 Z_n + \frac{2}{\mu} \error(\hat{g})^2 \\
    ~& + \left(\frac{6}{s}-\frac{\beta}{n}\right)K(\hat{\theta}) + \frac{\beta}{n}K(\theta^*)\\
    \leq~& \tilde{R}(\theta^*, g_0) + \mu V(\theta^*)+ \frac{\beta}{n}K(\theta^*) +2 Z_n + \frac{2}{\mu} \error(\hat{g})^2   \tag{since, $\sigma=2$, $\mu\leq \frac{\nu}{14}$ and $V(\hat{\theta})>0$ and $\beta\geq \frac{6n}{s}$ and $K(\hat{\theta})\geq0$}\\
    =~& \min_{\theta\in \Theta} \left[\tilde{R}(\theta, g_0) + \mu V(\theta) + \frac{\beta}{n}K(\theta^*)\right] + 2 Z_n + \frac{2}{\mu} \error(\hat{g})^2 \tag{Since $\theta_*$ minimizes $H(\theta, g_0)$}\\
    \leq~& \min_{j} \left[\tilde{R}(e_j, g_0) + \mu V(e_{j}) + \frac{\beta}{n}\log(\frac{1}{\pi_j})\right]+ 2 Z_n + \frac{2}{\mu} \error(\hat{g})^2 \\
    =~& \min_{j} \left[R(e_j, g_0) + \frac{\beta}{n}\log(\frac{1}{\pi_j})\right]+ 2 Z_n + \frac{2}{\mu} \error(\hat{g})^2 \tag{Since $\tilde{R}(e_j, g_0)=R(e_j, g_0)$ and $V(e_j)=0$ for all $e_j$}
\end{align}
\end{proof}

\begin{lemma}
Assuming that $f\in F$ are uniformly and absolutely and bounded in $[-b,b]$, for any 
$s\leq \frac{n \mu}{2b^2 (1+\mu)}$:
\begin{align}
    {\Pr\left(Z_{n,3} \geq \frac{\log(1/\delta)}{s}\right) \leq \delta}
\end{align}    
\end{lemma}
\begin{proof}
It suffices to show that:
\begin{align}
    I := \E\left[\exp s Z_{n,3} \right] \leq 1
\end{align}
Note first that:
\begin{align}
    \E\left[\exp s Z_{n,3} \right] =~& \E\left[\exp\left\{s
    \max_{j} \left[(P-P_n) (f_{j}(Z) - f_{\theta^*}(Z))^2 - 
    \mu P (f_j(Z) - f_{\theta^*}(Z))^2 + \frac{1}{s}\log(\pi_j)\right]\right\}\right]\\
    \leq~& \E\left[\max_j \exp\left\{s \left( (P-P_n) (f_{j}(Z) - f_{\theta^*}(Z))^2 - 
    \mu P (f_j(Z) - f_{\theta^*}(Z))^2+ \frac{1}{s}\log(\pi_j)\right)\right\}\right]\\
    \leq~&\sum_j \pi_j \E\left[\exp\left\{s \left( (P-P_n) (f_{j}(Z) - f_{\theta^*}(Z))^2 - 
    \mu P (f_j(Z) - f_{\theta^*}(Z))^2\right)\right\}\right]\\
    \leq~& \max_j \E\left[\exp\left\{s \left( (P-P_n) (f_{j}(Z) - f_{\theta^*}(Z))^2 - 
    \mu P (f_j(Z) - f_{\theta^*}(Z))^2\right)\right\}\right]
\end{align}
    Since $f_j(Z)\in [-b, b]$, we have that $f_j(Z) - f_{\theta^*}(Z)\in [-2b, 2b]$. Thus:
\begin{align}
(f_j(Z)-f_{\theta^*}(Z))^2\geq \frac{(f_j(Z)-f_{\theta^*}(Z))^2}{4b^2} (f_j(Z)-f_{\theta^*}(Z))^2 = \frac{1}{4b^2} (f_j(Z)-f_{\theta^*}(Z))^4.
\end{align}
Thus:
\begin{align}
    I \leq~& \max_j \E_{Z_{1:n}}\left[ \exp\left\{s \left(\left(P_n - P\right)(f_{j}(Z) - f_{\theta^*}(Z))^2 - \frac{\mu}{4b^2} P (f_{j}(X) - f_{\theta^*}(X))^4 \right)\right\}\right]
\end{align}
It suffices to show that each term $j$ is upper bounded by $1$.

Applying Lemma~\ref{lem:bernstein} with:
\begin{align}
    V_i =~& (f_j(Z_i) - f_{\theta^*}(Z_i))^2\\
    c =~& 4b^2\\
    \lambda =~& s\\
    c_0 =~& \frac{\mu}{4b^2}
\end{align}
we have that as long as:
\begin{align}
    \lambda := \frac{s}{n} \leq \frac{2c_0}{1+2c_0 c} =: \frac{2\cdot \mu/4b^2}{1 + \mu} = \frac{\mu}{2b^2(1+\mu)}
\end{align}
Then:
\begin{align}
    \E_{Z_{1:n}}\left[ \exp\left\{s \left(\left(P_n - P\right)(f_{j}(Z) - f_{\theta^*}(Z))^2 - \frac{\mu}{4b^2} P (f_{j}(X) - f_{\theta^*}(X))^4 \right)\right\}\right]\leq 1
\end{align}
Thus we need that:
\begin{align}
    s \leq \frac{n \mu}{2b^2 (1+\mu)}
\end{align}
which is satisfied by our assumptions on $s$. 
\end{proof}

To bound $Z_{n,4}$ we will use the following concentration inequality of based on critical radius (e.g. \cite{wainwright2019high}). See \citep{foster2023orthogonal} for a proof of this theorem:
\begin{lemma}[Localized Concentration, \citep{foster2023orthogonal}]\label{lem:concentration}
For any $h\in \Hcal := \prod_{i=1}^d \Hcal_i$ be a multi-valued outcome function, that is almost surely absolutely bounded by a constant. Let $\ell(Z; h(W))\in \R$ be a loss function that is $L$-Lipschitz in $h(W)$, with respect to the $\ell_2$ norm. Fix any $h_*\in \Hcal$ and let $\delta_n=\Omega\left(\sqrt{\frac{d\,\log\log(n) + \log(1/\zeta)}{n}}\right)$ be an upper bound on the critical radius of $\shull(\Hcal_i-h_{i,*})$ for $i\in [d]$. Then w.p. $1-\zeta$: 
\begin{align}
    \forall h\in \Hcal: \left|(P_n - P)\left[\ell(Z; h(W)) - \ell(Z; h_*(W))\right]\right| = 18\,L\,\left(d\, \delta_n \sum_{i=1}^d \|h_i - h_{i,*}\|_{L_2} + d\,\delta_n^2\right)
\end{align}
If the loss is linear in $h(W)$, i.e. $\ell(Z; h(W) + h'(W)) = \ell(Z; h(W)) + \ell(Z; h'(W))$ and $\ell(Z;\alpha h(W)) = \alpha \ell(Z;h(W))$ for any scalar $\alpha$, then it suffices that we take $\delta_n=\Omega\left(\sqrt{\frac{\log(1/\zeta)}{n}}\right)$ that upper bounds the critical radius of $\shull(\Hcal_i-h_{i,*})$ for $i\in [d]$.
\end{lemma}

\begin{lemma}
Let $r_n$ denote the critical radius of the function space $\operatorname{star}(Y_G)$ where $Y_G := \{ (\hat{Y}(g) - \hat{Y}(g_0))^2 | g \in G\}$ and assume that $\hat{Y}(g)$ is uniformly bounded by $U$ for all $g \in G$.%
Then for $\delta_n \geq \max\left\{r_n,\Omega\left(\sqrt{\frac{\log(1/\delta)}{n}}\right)\right\}$, w.p. $1-\delta$
\begin{align}
Z_{n,4} \leq~& 18\,\left(\delta_n \sqrt{\E\left[(\hat{Y}(g) - \hat{Y}(g_0))^4\right]} + \delta_n^2\right) \label{eqn:lemma19}
\end{align}    
\end{lemma}
\begin{proof}
    We can apply Lemma~\ref{lem:concentration} with:
    \begin{align}
        h_g(W) :=~& (\hat{Y}(g) - \hat{Y}(g_0))^2, g\in G &
        h_*(W) :=~& 0 &
        \ell(Z;h(W)) :=~& h(W)
    \end{align}
    where $W$ is assumed to contain all the random variables that go into the calculation of the proxy labels $\hat{Y}(g)$.
    Note that the loss function is $1$-Lipschitz in $h$ and we have a scalar-valued function $h$, i.e. $d=1$ and the loss function is linear in $h$. Thus we get that if $\delta_n = \Omega\left(\sqrt{\frac{\log(1/\zeta)}{n}}\right)$ upper bounds the critical radius $r_n$ of $\shull(Y_G)$, then w.p. $1-\zeta$:
    \begin{align}
        \left|(P_n - P) h_{\hat{g}}(W)\right| = 18\,\left(\delta_n \|h_{\hat{g}}\|_{L_2} + \delta_n^2\right) \leq 18\,\left(\delta_n \sqrt{\E\left[(\hat{Y}(g) - \hat{Y}(g_0))^4\right]} + \delta_n^2\right)
    \end{align}
\end{proof}

\subsection{Putting it all together}
  
Applying the AM-GM inequality to Equation \eqref{eqn:lemma19}, we get:
\begin{align}
    Z_{n,4} \leq~& \E[(\hat{Y}(\hat{g}) - \hat{Y}(g_0))^4] + 23\delta_n^2
    = \E[(\hat{Y}(\hat{g}) - \hat{Y}(g_0))^4] + O\left(r_n^2+\frac{\log(1/\delta)}{n}\right)
\end{align}

Assuming $\hat{Y}(g)$ be uniformly bounded by $U$ and $f_j(X)$ is absolutely and uniformly bounded by $U$, the square loss is $4U$-Lipschitz. Choosing $\nu=1/2$ and combining all the above Lemmas, we have that w.p. $1-\delta$:
\begin{align}
    R(\hat{\theta}) \leq~& \min_{j}\left( R(e_j)  +\frac{\beta}{n}\log(1/\pi_j)\right) + O\left(  r_n^2 + \frac{\log(\log(n))+ \log(1/\delta) \max\{\frac{1}{\mu} U^2, b\,U \}}{n}\right)
     + \frac{2}{\mu} \error(\hat{g})^2
\end{align}

\section{Greedy Q-aggregation for Square Losses}\label{app:greedy}
In this section, we show that we can get comparable theoretical guarantees when using a k-step greedy approximation of the Q-aggregation algorithm. 

\subsection{Preliminaries}
Consider the following k-step greedy algorithm to find an approximate minimizer of a functional $Q(\theta)$, proposed in \cite{greedy}:
\begin{algorithm}[H]
\normalsize
	\caption{Greedy model averaging GMA-0$_{+}$}\label{alg:greedy}
	\begin{algorithmic}
	\State Let $\theta^{(0)}=0$, $f_{\theta^{(0)}} = 0$
	\For{$k = 1, 2, \cdots$}
	\normalsize
	\State Set $a_k = \frac{2}{k+1}$
	\State $J^{(k)} = \argmin_{j}Q((1-a_{k})\theta^{(k-1)} +a_{k} e_{j})$
	\State $\theta^{(k)} = \argmin_{\theta \in \Theta}Q(\theta)$ s.t. $\theta_{j} = 0$ for $j \not\in \{J^{(1)}, \cdots, J^{(k)}\}$
 \EndFor
	\end{algorithmic}
\end{algorithm}

\begin{theorem}[\cite{greedy}]\label{thm:approximate}
Suppose that $\ell_{\theta,\hat{g}}(Z)$ is of the form $(Y - f_\theta(X))^2$ for some label $Y\equiv \hat{Y}(\hat{g})$. Denote with:
\begin{align}
    Q(\theta) :=~& P_n \tilde{\ell}_{\theta, \hat{g}}(Z) + \frac{\beta}{n}K(\theta) &
    V_n(\theta) :=~&\sum_{j=1}^M \theta_j P_n (f_j(X) - f_\theta(X))^2
\end{align}
Then the estimator $f_{\theta^{(k)}}$, where $\theta^{(k)}$ is the output of GMA-$0_{+}$ after k steps satisfies:
\begin{align}
    Q(\theta^{(k)}) \leq \min_{\theta \in \Theta}\{Q(\theta)+ C_Q V_{n}(\theta)\}
\end{align} where $C_Q=\frac{4(1-\nu)}{k+3}$.
\end{theorem}

\begin{proof}
We first introduce some short-hand notation that will be convenient for the proof. Let $x_\theta = (f_\theta(X_1), \ldots, f_\theta(X_n))\in \R^n$, denote the vector of values of the function $f_\theta$ on the $n$ samples and $y=(Y_1, \ldots, Y_n)\in \R^n$ denote the vector of samples. 
For any vectors $x_{\theta},x_{\bar{\theta}} \in \mathbb{R}^n$ define the empirical norm and associated inner product
\begin{align}
    \|x_{\theta}-x_{\bar{\theta}}\|_n^2 :=~& \frac{1}{n} \sum_{i=1}^{n} (x_{\theta_{i}}-x_{\bar{\theta}_{i}})^2 &
\langle x_{\theta}, x_{\bar{\theta}} \rangle_{n} :=~& \frac{1}{n}\sum_{i=1}^{n}x_{\theta_{i}} x_{\bar{\theta}_{i}}
\end{align}
Then note that:
\begin{align}
    V_n(\theta) =~&\sum_{j=1}^M \theta_j \|x_j - x_\theta\|_n^2 &
    Q(\theta) =~& (1-\nu)\|x_\theta - y\|_n^2 + \nu\,\sum_{j=1}^M \theta_j \|x_j - y\|_n^2 + \frac{\beta}{n}K(\theta)
\end{align}

From Taylor expansion on the Q-functional, we have for any $x_\theta, x_{\theta'} \in \mathbb{R}^n$:
\begin{align}
    Q(\theta) = Q(\theta') + \langle \nabla_{x}Q({\theta'}), x_{\theta}-x_{\theta'} \rangle_{n} + (1-\nu)
    \|x_{\theta}-x_{\theta'}\|_{n}^2
\end{align}
Observe that for any $j \in \{1,\cdots, M\}$:
\begin{align}
    Q(\theta^{(k)}+a_{k+1}(e_{J^{(k+1)}})-\theta^{(k)}) \leq Q(\theta^{(k)}+a_{k+1}(e_{j}-\theta^{(k)}))
\end{align}
Multiplying by $\theta_{j}$ and summing over all $\theta$ we get that for any $\theta \in \Theta$:
\begin{align}
        & \sum_{j=1}^{M}\theta_{j}Q(\theta^{(k)}+a_{k+1}(e_{J^{(k+1)}})-\theta^{(k)}) \leq \sum_{j=1}^{M}\theta_{j}Q(\theta^{(k)}+a_{k+1}(e_{j}-\theta^{(k)})) \\
        \iff & Q(\theta^{(k)}+a_{k+1}(e_{J^{(k+1)}})-\theta^{(k)}) \leq \sum_{j=1}^{M}\theta_{j}Q(\theta^{(k)}+a_{k+1}(e_{j}-\theta^{(k)})) \\
        \implies & Q(\theta^{(k+1)}) \leq \sum_{j=1}^{M}\theta_{j}Q(\theta^{(k)}+a_{k+1}(e_{j}-\theta^{(k)}))
\end{align}
Using the Taylor expansion formula with $\theta = \theta^{(k)}+a_{k+1}(e_{j}-\theta^{(k)})$ and $\theta'=\theta^{(k)}$ on the RHS:
\begin{align}
   Q(\theta^{(k+1)}) \leq~& \sum_{j=1}^M \theta_{j} \left[ Q(\theta^{(k+1)}) +\left[a_{k+1}(e_{j}-\theta^{(k)})\right]^T \nabla Q(\theta^{(k)}) +(1-\nu)\norm{\sum_{j=1}^{M}a_{k+1}(e_{j}-\theta^{(k)})x_{j}}_{n}^2\right] \\
   \leq~& Q(\theta^{(k)}) + a_{k+1}\sum_{j=1}^M \theta_{j} \langle e_{j}, \nabla Q(\theta^{(k)}) \rangle_{n} -a_{k+1}\langle \theta^{(k)}, \nabla Q(\theta^{(k)}) \rangle_{n} + (1-\nu)a_{k+1}^{2} \sum_{j=1}^{M}\theta_{j}  \norm{x_{j}-x_{\theta^{(k)}}}_{n}^2  \\
   \leq~& Q(\theta^{(k)}) + a_{k+1} \langle \theta - \theta^{(k)} , \nabla Q(\theta^{(k)}) \rangle_{n} + (1-\nu)a_{k+1}^{2} \sum_{j=1}^{M}\theta_{j}  \norm{x_{j}-x_{\theta^{(k)}}}_{n}^2
   \\\leq~& Q(\theta^{(k)})+ a_{k+1}[Q(\theta)-Q(\theta^{(k)})] -(1-\nu) a_{k+1} \|x_{\theta^{(k)}}-x_{\theta}\|_{n}^2 + (1-\nu)a_{k+1}^{2}\sum_{j=1}^{M}\theta_{j}  \norm{x_{j}-x_{\theta^{(k)}}}_{n}^2 
   \\ \leq~& Q(\theta^{(k)})+ a_{k+1}[Q(\theta)-Q(\theta^{(k)})]-(1-\nu)a_{k+1}\|x_{\theta^{(k)}}-x_{\theta}\|_{n}^2 
   \\ & + (1-\nu)a_{k+1}^{2}\left(\sum_{j=1}^{M}\theta_{j}  \norm{x_{j}-x_{\theta}}_{n}^2 +\norm{x_{\theta^{(k)}}-x_{\theta}}_{n}^2 \right)\\
   \leq~& Q(\theta^{(k)})+ a_{k+1}[Q(\theta)-Q(\theta^{(k)})] + (1-\nu)a_{k+1}^{2}\sum_{j=1}^{M}\theta_{j}  \norm{x_{j}-x_{\theta}}_{n}^2
\end{align}
\begin{align}
Q(\theta^{(k+1)})-Q(\theta) \leq~& (1-a_{k+1})[Q(\theta^{(k)})-Q(\theta)] +(1-\nu)a_{k+1}^2V(\theta)
\end{align}
For simplicity, denote $\delta_{k}=Q(\theta^{(k)})-Q(\theta)$ and $B=(1-\nu)\sum_{j=1}^{M}\theta_{j}\|x_{j}-x_{\theta}\|_{n}^2$

Using an induction on k, it is straightforward to prove the following:
\begin{align}
    \delta_{k}\leq \frac{4B}{k+3}
\end{align}
\end{proof}

\subsection{Excess Risk Guarantees for the k-step Greedy Q-aggregation Estimator}

\begin{theorem}\label{thm:approximate-excess}
Assume that for some function $\error(\hat{g})$:
\begin{align}
    P(\tilde{\ell}_{\theta,g_0}(Z) - \tilde{\ell}_{\theta',g_0}(Z)) - P(\tilde{\ell}_{\theta,\hat{g}}(Z) - \tilde{\ell}_{\theta',\hat{g}}(Z))
    \leq \error(\hat{g})\, \|f_{\theta} - f_{\theta'}\|^{\gamma}
\end{align}
for some $\gamma\in [0, 1]$ and square loss $\ell_{f,g}$. Assume that $\hat{Y}(\hat{g})$ and all functions $f_j$ are absolutely uniformly bounded by $U$. Let $\theta^k$ be the k-step greedy approximate of $\hat{\theta}$. {For any $k\geq 2$, $\nu\geq 1/2$ and $\mu \leq \frac{1}{60} (1 - \nu)$}
 and $\beta \geq \max\{\frac{64 U^2(1-\nu)^2}{\mu}, \frac{16\nu U^2(\nu + 8 \mu U)}{\mu}\}$, w.p. $1-\delta$:
\begin{align}
    R(\theta^k, g_0) \leq~& \min_{j=1}^M\left( R(e_j, g_0)  +\frac{3\beta}{2n}\log(1/\pi_j)\right) + O\left( \frac{\log(1/\delta) \frac{1}{\mu}U^2}{n}\right) + \left(\frac{1}{\mu}\right)^{\frac{\gamma}{2-\gamma}} \error(\hat{g})^{\frac{2}{2-\gamma}}
\end{align}
\end{theorem}

{Here, we prove an analogue of  Lemma \ref{lem:main} for the K-step greedy approximate of $\hat{\theta}$.}

\begin{lemma}\label{lem:main_greedy_half}
{Let $\theta^*=\argmin_{\theta \in \Theta} \tilde{R}(\theta, g_0) + C_V V(\theta) + C_K K(\theta)$ and let $\theta^k$ be the k-step greedy approximate of $\hat{\theta}$.}

Define the centered and offset empirical processes:
\begin{align}
    Z_{n,1} =~& (1-\nu)(P-P_n) (\ell_{\theta^k, \hat{g}}(Z) - \ell_{\theta^*, \hat{g}}(Z)) - \mu \sum_{j} \theta^k_j \|f_j - f_{\theta^*}\|^2 {-\frac{1}{s}K(\theta^k)}\\
    Z_{n,2} =~& \nu(P-P_n) \sum_{j} (\theta^k_j-\theta_j^*)\ell_{e_j, \hat{g}}(Z) - \mu (\theta^k)^\top H \theta^*{-\frac{1}{s}K(\theta^k)} & H_{ij} = \|f_i - f_j\|^2 \\
    { Z_{n,5} =}~& {C_Q(V_n(\theta^*) - V(\theta^*)) - \mu V(\theta^*) - \frac{1}{s}K(\theta^*)}\\
\end{align}
where $s$ is a positive number. 
Assume that:
\begin{align}
    {\cal E}(\hat{g}) := P(\tilde{\ell}_{\theta^k,g_0}(Z) - \tilde{\ell}_{\theta^*,g_0}(Z)) - P(\tilde{\ell}_{\theta^k,\hat{g}}(Z) - \tilde{\ell}_{\theta^*,\hat{g}}(Z)) \leq \error(\hat{g})\, \|f_{\theta^k} - f_{\theta^*}\|
\end{align}
and $\ell_{f, g} = (\hat(Y)(g) - f)$ where $\hat{Y}(\hat{g})$ and all functions $f_j$ are absolutely uniformly bounded by $U$. 

Choose {$\beta \geq \frac{8n}{s}$}, {$C_K = \frac{3\beta}{2n}$}, and {$C_V< \min\{1-\nu,\nu\}$}. Furthermore, choose $k$ such that $C_Q<C_V$. Then, assuming {$\mu \leq \min\left\{\frac{1}{6}\left(1-\nu - C_V \right), {\frac{1}{4} (\nu - C_V)}, \frac{1}{2}(C_V -C_Q)\right\}$}, we get:
\begin{align}
    R(\theta^k) \leq \min_{j} [R(e_j, g_0){+ \frac{3\beta}{2n} \log(1/\pi_j)}] + 2\, (Z_{n,1} + Z_{n,2} + Z_{n,5}) + \frac{1}{\mu}\error(\hat{g})^2 \label{eqn:lemma9-greedy}
\end{align}
\end{lemma}

\begin{proof}

For k-step greedy approximate, Theorem \ref{thm:approximate} implies that:
\begin{align}
P_n \tilde{\ell}_{\theta^k, \hat{g}}(Z) + \frac{\beta}{n}K(\theta^k) \leq~& \min_{\theta \in \Theta} \{ P_n \tilde{\ell}_{\theta, \hat{g}}(Z) + \frac{\beta}{n}K(\theta)+ \frac{4(1-\nu)}{k+3}V_n(\theta)\}  \\
\leq~&P_n \tilde{\ell}_{\theta^*, \hat{g}}(Z) + \frac{4(1-\nu)}{k+3}V_n(\theta^*) + \frac{\beta}{n}K(\theta^*) \label{eqn:greedy_opt}
\end{align}

\begin{align}
    \tilde{R}(\theta^k, g_0) - \tilde{R}(\theta^*, g_0) =~& 
    P(\tilde{\ell}_{\theta^k,g_0}(Z) - \tilde{\ell}_{\theta^*,g_0}(Z))\\
    =~& 
    P(\tilde{\ell}_{\theta^k,\hat{g}}(Z) - \tilde{\ell}_{\theta^*,\hat{g}}(Z)) + {\cal E}(\hat{g})\\
    \leq~& 
    (P - P_n)(\tilde{\ell}_{\theta^k, \hat{g}}(Z) - \tilde{\ell}_{\theta^*, \hat{g}}(Z)) + {\cal E}(\hat{g}) 
    {-\frac{\beta}{n}(K(\theta^k)-K(\theta^*))} {+ C_Q V_n(\theta^*)}
    \tag{By Eqn \ref{eqn:greedy_opt}}\\
    =~& \mu\sum_j \theta^k_j \|f_j - f_{\theta^*}\|^2 + \mu (\theta^k)^\top H\theta^*+ Z_{n,1} + Z_{n,2} {+Z_{n,5}} + {\cal E}(\hat{g})\\
    ~& {-\frac{\beta}{n}(K(\theta^k)-K(\theta^*)) + \frac{2}{s}K(\theta^k)} {+ (C_Q+\mu) V(\theta^*) + \frac{1}{s}K(\theta^*)}
    \tag{definition of $Z_{n,1}, Z_{n,2}, Z_{n,5}$}
\end{align}
By simple algebraic manipulations note that for any $\theta, \theta' \in \Theta$:
\begin{align}
    \sum_j \theta_j \|f_j - f_{\theta'}\|^2 = V(\theta) + \|f_{\theta} - f_{\theta'}\|^2 \label{eqn:v}\\
    \theta^\top H \theta'= V(\theta) + V(\theta') + \|f_{\theta} - f_{\theta'}\|^2
\end{align}
Letting $Z_n=Z_{n,1} + Z_{n,2} {+Z_{n,5}}$, we have:
\begin{align}
    \tilde{R}(\theta^k, g_0) - \tilde{R}(\theta^*_k, g_0) \leq~& 2\mu V(\theta^k) + ({2\mu+C_Q}) V(\theta^*) + 2\mu \|f_{\theta^k} - f_{\theta^*}\|^2 + Z_n + {\cal E}(\hat{g})\\
    ~& -\frac{\beta}{n}(K(\theta^k)-K(\theta^*)) + \frac{2}{s}K(\theta^k) {+\frac{1}{s}K(\theta^*)}\\
    \leq~& 2\mu V(\theta^k) + ( {2\mu+C_Q}) V(\theta^*) + 2\mu \|f_{\theta^k} - f_{\theta^*}\|^2 + Z_n + \error(\hat{g})\, \|f_{\theta^k} - f_{\theta^*}\|^{\gamma}\\
    ~& {-\frac{\beta}{n}(K(\theta^k)-K(\theta^*)) + \frac{2}{s}K(\theta^k)}  {+\frac{1}{s}K(\theta^*)} \label{eqn:}
\end{align}
By Young's inequality, for any $a\geq 0$ and $b\geq 0$, we have $a\,b\leq \frac{a^p}{p} + \frac{b^q}{q}$, such that $\frac{1}{p}+\frac{1}{q}=1$. Taking $a= \left(\frac{\gamma}{2\mu}\right)^{\gamma/2} \error(\hat{g})$ and $b=\left(\frac{2\mu}{\gamma}\right)^{\gamma/2}\|f_{\theta^k}-f_{\theta^*}\|^\gamma$ and $q=2/\gamma$ and $p=\frac{1}{1-1/q}=\frac{1}{1-\gamma/2}=\frac{2}{2-\gamma}$
\begin{align}
    \error(\hat{g})\, \|f_{\theta^k} - f_{\theta^*}\|^{\gamma} \leq \underbrace{\frac{2}{2-\gamma}\left(\frac{\gamma}{2\mu}\right)^{\frac{\gamma}{2-\gamma}} \error(\hat{g})^{\frac{2}{2-\gamma}}}_{{\cal A}(\hat{g})}+ \mu \|f_{\theta^k} - f_{\theta^*}\|^2
\end{align}
Thus we can derive:
\begin{align}
    \tilde{R}(\theta^k, g_0) - \tilde{R}(\theta^*, g_0) \leq~& 2\mu V(\theta^k) + ({2\mu+C_Q}) V(\theta^*) + 3\mu \|f_{\theta^k} - f_{\theta^*}\|^2 + Z_n + {\cal A}(\hat{g}) \\
    ~& {-\frac{\beta}{n}(K(\theta^k)-K(\theta^*)) + \frac{2}{s}K(\theta^k)} {+\frac{1}{s} K(\theta^*)}\label{eqn:greedy_first_upper}
\end{align}

{
To obtain a lower bound:
\begin{align}
    \tilde{R}(\theta^k, g_0) - \tilde{R}(\theta^*, g_0) =~&
    H(\theta_k, g_0) - H(\theta^*,g_0) + C_V (V(\theta^*) - V(\theta^k)) +C_K (K(\theta^*)-K(\theta^k))\\
    \geq~& \left(1-\nu - C_V\right)\|f_{\theta^k} - f_{\theta^*}\|^2 + C_V (V(\theta^*) - V(\theta^k)) +C_K (K(\theta^*)-K(\theta^k)) \tag{By Lemma \ref{lem:strong-conv-tilde}}\\
\end{align}
}
Combining the latter upper and lower bound on $\tilde{R}(\theta^k, g_0) - \tilde{R}(\theta^*, g_0)$, we have:
{
\begin{align}
   \left(1-\nu - C_V - 3\mu \right)  \|f_{\theta^k} - f_{\theta^*}\|^2 \leq~& (2\mu + C_V)V(\theta^k) + Z_n {
    + {\cal A}(\hat{g})} +  (2\mu + C_Q - C_V)V(\theta^*)\\
    -~& \left(\frac{\beta}{n} - C_K\right)(K(\theta^k) - K(\theta^*)) + \frac{2}{s}K(\theta^k) + \frac{1}{s}K(\theta^*)\\
\end{align}
Choosing:
\begin{align}
     \left(1-\nu - C_V - 3\mu \right) \geq 3\mu ~~~~~\Leftrightarrow~~~~~
     \frac{1}{6}\left(1-\nu - C_V \right) \geq \mu
\end{align}
We get that:
\begin{align}
  3\mu \|f_{\theta^k} - f_{\theta^*}\|^2 \leq~& {
  Z_n 
  + {\cal A}(\hat{g})} + (2\mu + C_V)V(\theta^k) +  (2\mu + C_Q - C_V)V(\theta^*)\\
  -~& \left(\frac{\beta}{n} - C_K\right)(K(\theta^k) - K(\theta^*)) + \frac{2}{s}K(\theta^k) + \frac{1}{s}K(\theta^*)\\
\end{align}
}
Plugging the bound on $\|f_{\theta^k} - f_{\theta^*}\|^2$ back into Equation~\eqref{eqn:greedy_first_upper} we have:
\begin{align}
    \tilde{R}(\theta^k, g_0) - \tilde{R}(\theta^*, g_0) \leq ~& (4\mu + C_V) V(\theta^k) + (4\mu + 2C_Q - C_V) V(\theta^*) + 2\, Z_n + 2\,{\cal A}(\hat{g}) \\
    ~& - \left(\frac{2\beta}{n} - C_K\right)(K(\theta^k)-K(\theta^*)) + \frac{4}{s}K(\theta^k) + \frac{2}{s}K(\theta^*)\\
\end{align}
By Lemma~\ref{lem:v-versus-tilde}, we have:
\begin{align}
    R(\theta^k, g_0) \leq~& \tilde{R}(\theta^*, g_0) + (4\mu + 2C_Q - C_V) V(\theta^*)+ \left(4\mu + C_V - {\nu}\right) V(\theta^k) + 2 Z_n + 2{\cal A}(\hat{g}) \\
    ~& {+\left(\frac{2\beta}{n}-C_K+\frac{2}{s}\right)K(\theta^*) + \left(\frac{4}{s} - \frac{2\beta}{n} + C_K\right)K(\theta^k)}\\
    \leq~& \tilde{R}(\theta^*, g_0) + (4\mu + 2C_Q - C_V) V(\theta^*) + 2 Z_n + 2{\cal A}(\hat{g}) \tag{Since $\mu \leq \frac{1}{4}(\nu - C_V)$}\\
    ~& {+(\frac{2\beta}{n}-C_K+\frac{2}{s})K(\theta^*) + \left(\frac{4}{s} - \frac{2\beta}{n} + C_K\right)K(\theta^k)}\\
\end{align}
{
Let $C_K = \frac{3\beta}{2n}$
Assuming:
\begin{align}
    \frac{1}{s} \leq~& \frac{\beta}{8n}\\
    C_V \geq~& C_Q\\
    \mu \leq~& \frac{1}{2}(C_V - C_Q)
\end{align}
We get that:
\begin{align}
    4\mu + 2C_Q - C_V \leq C_V\\
    \frac{2\beta}{n} - C_K +\frac{2}{s} \leq~& C_K\\
    \frac{4}{s} - \frac{2\beta}{n} + C_K \leq~&0\\
\end{align}
Thus we get:
\begin{align}
    R(\theta^k, g_0)\leq~& \tilde{R}(\theta^*, g_0) + C_V V(\theta^*) + C_K K(\theta^*) + 2 Z_n + 2{\cal A}(\hat{g})\\
    \leq~& \min_{j} [\tilde{R}(e_j, g_0) + C_V V(e_{j}) {+  C_K\log(1/\pi_j)}]+ 2 Z_n + 2{\cal A}(\hat{g})\\
    =~& \min_{j} [R(e_j, g_0){+  \frac{3\beta}{2n} \log(1/\pi_j)}] + 2 Z_n +2{\cal A}(\hat{g}) \tag{Since $\tilde{R}(e_j, g_0)=R(e_j, g_0)$ and $V(e_j)=0$ for all $e_j$}
\end{align}
}
\end{proof}

\begin{lemma}
Assuming that $f\in F$ are uniformly and absolutely and bounded in $[-U,U]$, for any positive 
$s\leq \frac{n \mu}{2 C_Q^2 U^2 (1+\mu)}$:
\begin{align}
    {\Pr\left(Z_{n,5} \geq \frac{\log(1/\delta)}{s}\right) \leq \delta}
\end{align}    
\end{lemma}
\begin{proof}
It suffices to show that:
\begin{align}
    I := \E\left[\exp s Z_{n,5} \right] \leq 1
\end{align}
Note first that:
\begin{align}
    Z_{n,5} =~& {C_Q(V_n(\theta^*) - V(\theta^*)) - \mu V(\theta^*) - \frac{1}{s}K(\theta^*)}\\
    \leq~& \max_j \left[ C_Q(P_n - P)(f_{j}(Z) - f_{\theta^*}(Z))^2 - \mu P(f_{j}(Z) - f_{\theta^*}(Z))^2 + \frac{1}{s}\log(\pi_j)\right]
\end{align}
Thus:
\begin{align}
    \E\left[\exp s Z_{n,5} \right] \leq ~& \E\left[\exp\left\{s
    \max_{j} \left[C_Q(P_n - P)(f_{j}(Z) - f_{\theta^*}(Z))^2 - 
    \mu P (f_{j}(Z) - f_{\theta^*}(Z))^2 + \frac{1}{s}\log(\pi_j)\right]\right\}\right]\\
    \leq~& \E\left[\max_j \exp\left\{s \left( C_Q(P_n - P)(f_{j}(Z) - f_{\theta^*}(Z))^2 - 
     \mu  P (f_{j}(Z) - f_{\theta^*}(Z))^2+ \frac{1}{s}\log(\pi_j)\right)\right\}\right]\\
    \leq~&\sum_j \pi_j \E\left[\exp\left\{s \left( C_Q(P_n - P)(f_{j}(Z) - f_{\theta^*}(Z))^2 - 
     \mu  P(f_{j}(Z) - f_{\theta^*}(Z))^2\right)\right\}\right]\\
    \leq~& \max_j \E\left[\exp\left\{s \left( C_Q(P_n - P)(f_{j}(Z) - f_{\theta^*}(Z))^2 - 
     \mu P(f_{j}(Z) - f_{\theta^*}(Z))^2\right)\right\}\right]
\end{align}
    Since $f_j(Z)\in [U, U]$, we have that $f_j(Z) - f_{\theta^*}(Z)\in [-2U, 2U]$. Thus:
\begin{align}
(f_j(Z)-f_{\theta^*}(Z))^2\geq \frac{(f_j(Z)-f_{\theta^*}(Z))^2}{4U^2} (f_j(Z)-f_{\theta^*}(Z))^2 = \frac{1}{4U^2} (f_j(Z)-f_{\theta^*}(Z))^4.
\end{align}
Thus:
\begin{align}
    I \leq~& \max_j \E_{Z_{1:n}}\left[ \exp\left\{s \left(C_Q\left(P_n - P\right)(f_{j}(Z) - f_{\theta^*}(Z))^2 - \frac{\mu}{4 C_Q^2 U^2} P C_Q^2(f_{j}(Z) - f_{\theta^*}(Z))^4 \right)\right\}\right]
\end{align}
It suffices to show that each term $j$ is upper bounded by $1$.

Applying Lemma~\ref{lem:bernstein} with:
\begin{align}
    V_i =~& C_Q (f_j(Z_i) - f_{\theta^*}(Z_i))^2\\
    c =~& 4 C_Q^2 U^2\\
    \lambda =~& s\\
    c_0 =~& \frac{\mu}{4C_Q^2U^2}
\end{align}
we have that as long as:
\begin{align}
    \lambda := \frac{s}{n} \leq \frac{2c_0}{1+2c_0 c} =: \frac{2\cdot \mu/4 C_Q^2 U^2}{1 + \mu} = \frac{\mu}{2 C_Q^2 U^2(1+\mu)}
\end{align}
Then:
\begin{align}
    \E_{Z_{1:n}}\left[ \exp\left\{s \left(C_Q\left(P_n - P\right)(f_{j}(Z) - f_{\theta^*}(Z))^2 - \frac{\mu}{4 C_Q^2 U^2} P C_Q^2(f_{j}(X) - f_{\theta^*}(X))^4 \right)\right\}\right]\leq 1
\end{align}
Thus we need that:
\begin{align}
    s \leq \frac{n \mu}{2 C_Q^2 U^2 (1+\mu)}
\end{align}
which is satisfied by our assumptions on $s$. 
\end{proof}

\paragraph{Putting it All Together}
{Choose $\nu\geq 1/2$, such that $1-\nu \leq \nu$. Note that for any $k\geq 2$, we have $C_Q = \frac{4}{k+3}(1-\nu)\leq \frac{4}{5}(1-\nu)$. Choosing $C_V = \frac{5}{6} (1-\nu) \leq \frac{5}{6} \nu$, then we have that for $k\geq 2$, we have $C_Q < C_V$ and we can take $\mu \leq \frac{1}{60}(1-\nu)$ in Lemma~\ref{lem:main_greedy_half}.}
Combining all the above Lemmas, we have that w.p. $1-\delta$, the k-step greedy Q-aggregation estimator $\theta_k$ achieves:
\begin{align}
    R(\theta_k) \leq \min_{j}\left( R(e_j)  +\frac{3\beta}{2n}\log(1/\pi_j)\right) + O\left(  \frac{\log(1/\delta) \frac{1}{\mu}U^2}{n}\right) + \frac{2}{\mu} \error(\hat{g})^2
\end{align}

\section{Proof of Orthogonal Statistical Learning Theorem}\label{sec:osl_proof}
\subsection{Fast Rates}
\begin{proof}[Proof of Theorem \ref{thm:osl}]
The proof follows the proof of Theorem 1 in \cite{foster2023orthogonal} after replacing $f_*$ by $f$, and we recover the original theorem by assuming $F' = F$.

By a second order Taylor expansion around $f'$, there exist $\bar{f}\in\text{Star}(F',f')$ such that:
\begin{align}
    ( P\ell_{\hat{f},\hat{g}} - P\ell_{f',\hat{g}}) - D_{f}P\ell(f',\hat{g})[\hat{f}-f'] =~& \frac{1}{2}D_{f}^2P\ell_{\bar{f},\hat{g}}[\hat{f} - f',\hat{f}-f']\\
     \geq~& \frac{\lambda}{2}\|\hat{f}-f'\|^2 - \frac{\kappa}{2} d(\hat{g},g_0)^{\frac{4}{1+r}} \tag{By Assumption \ref{assum:4}}
\end{align}

Rearranging, we get:
\begin{align}
    \frac{\lambda}{2}\|\hat{f}-f'\|^2 \leq~& (P\ell_{\hat{f},\hat{g}} - P\ell_{f',\hat{g}})-D_{f}P\ell(f',\hat{g})[\hat{f}-f'] + \frac{\kappa}{2} d(\hat{g},g_0)^{\frac{4}{1+r}}\\
\end{align}

Applying a second-order Taylor expansion, we get:
\begin{align}
    -D_{f}P\ell(f',\hat{g})[\hat{f}-f'] =~& -D_{f}P\ell(f',g_0)[\hat{f}-f'] - D_{g}D_{f}P\ell_{f',g_0}[\hat{f} - f',\hat{g}-g_0] \\
    -~& \frac{1}{2}D_{g}^{2}D_{f}P\ell_{f,\bar{g}}[\hat{f} - f',\hat{g}-g_0, \hat{g}-g_0]\\
    =~& - D_{f}P\ell(f',g_0)[\hat{f}-f']- \frac{1}{2}D_{g}^{2}D_{f}P\ell_{f,\bar{g}}[\hat{f} - f',\hat{g}-g_0, \hat{g}-g_0] \tag{By Assumption \ref{assum:1}}\\
    \leq ~& - D_{f}P\ell(f',g_0)[\hat{f}-f'] + \beta_2\|\hat{f}-f'\|^{1-r}d(\hat{g},g_0)^2 \tag{By Assumption \ref{assum:3}(b)}\\
\end{align}

Alternatively, we can get the above result directly by \ref{assum:for_square_loss} instead of Assumptions \ref{assum:1} and \ref{assum:3}.

Invoking Young's inequality and using $r\in [0,1)$, we have that for any $\eta>0$:
\begin{align}
     -D_{f}P\ell(f',\hat{g})[\hat{f}-f'] \leq - D_{f}P\ell(f',g_0)[\hat{f}-f']+\frac{\beta_2\eta}{4}\|\hat{f}-f'\|^2 + \frac{\beta_2}{2\eta^{(1-r)/(1+r)}}d(\hat{g},g_0)^{\frac{4}{1+r}}
\end{align}
Choosing $\eta = \frac{\lambda}{\beta_2}$, we get that:
\begin{align}
    \|\hat{f}-f'\|^2 \leq - \frac{4}{\lambda}D_{f}P\ell(f',g_0)[\hat{f}-f']+ (P\ell_{\hat{f},\hat{g}} - P\ell_{f',\hat{g}}) +O\left( d(\hat{g},g_0)^{\frac{4}{1+r}} \right) \label{eqn:norm_f}
\end{align}

Taking Taylor expansion at $f'$, we have that there exist $\bar{f}\in\text{Star}(F',f')$ such that:
\begin{align}
    P\ell_{\hat{f},g_0} -  P\ell_{f',g_0} = ~& D_{f}P\ell_{f',g_0}[\hat{f}-f'] + \frac{1}{2}D_{f}^2P\ell_{\bar{f},g_0}[\hat{f}-f',\hat{f}-f']\\ 
    \leq~& D_{f}P\ell_{f',g_0}[\hat{f}-f'] + \frac{\beta_1}{2}\|\hat{f}-f'\|^2 \tag{By Assumption \ref{assum:3}(a)}\\
    \leq~& -(\frac{2\beta_1}{\lambda}-1)  D_{f}P\ell_{f',g_0}[\hat{f}-f']+  O\left((P\ell_{\hat{f},\hat{g}} - P\ell_{f',\hat{g}}) + d(\hat{g},g_0)^{\frac{4}{1+r}}\right)
\end{align}
Since $\beta_1$ is the upper bound, we may assume $\frac{2\beta_1}{\lambda}>1$. Thus, we get:
\begin{align}
    P\ell_{\hat{f},g_0} -  P\ell_{f',g_0} \leq O\left((P\ell_{\hat{f},\hat{g}} - P\ell_{f',\hat{g}}) + d(\hat{g},g_0)^{\frac{4}{1+r}}\right)\label{eqn:fast_rate}
\end{align}
Moreover, if 
\begin{align}
    P\ell_{\hat{f},\hat{g}}(Z) - P\ell_{f',\hat{g}}(Z) \leq \epsilon_n(\delta) \| \hat{f} - f'\| + \alpha_n(\delta)
\end{align}
Plugging this into Equation \ref{eqn:norm_f} yields:
\begin{align}
    \|\hat{f}-f'\|^2 \leq~& - \frac{4}{\lambda}D_{f}P\ell(f',g_0)[\hat{f}-f']+ \epsilon_n(\delta) \| \hat{f} - f'\| + \alpha_n(\delta) +O\left( d(\hat{g},g_0)^{\frac{4}{1+r}} \right)\\
    \leq~& - \frac{4}{\lambda}D_{f}P\ell(f',g_0)[\hat{f}-f']+ \frac{1}{2}\| \hat{f} - f'\|^2 +O\left(\epsilon_n(\delta)^2 + \alpha_n(\delta)+ d(\hat{g},g_0)^{\frac{4}{1+r}} \right)
\end{align}
Thus, we get:
\begin{align}
     \|\hat{f}-f'\|^2 \leq-\frac{8}{\lambda}D_{f}P\ell(f',g_0)[\hat{f}-f'] +O\left(\epsilon_n(\delta)^2 + \alpha_n(\delta)+ d(\hat{g},g_0)^{\frac{4}{1+r}} \right)
\end{align}
Plugging this into the Taylor expansion at $f'$, we get:
\begin{align}
    P\ell_{\hat{f},g_0} -  P\ell_{f',g_0} = ~& D_{f}P\ell_{f',g_0}[\hat{f}-f'] + \frac{1}{2}D_{f}^2P\ell_{\bar{f},g_0}[\hat{f}-f',\hat{f}-f']\\ 
    \leq~& D_{f}P\ell_{f',g_0}[\hat{f}-f'] + \frac{\beta_1}{2}\|\hat{f}-f'\|^2 \tag{By Assumption \ref{assum:3}(a)}\\
    \leq~& -(\frac{4\beta_1}{\lambda}-1)  D_{f}P\ell_{f',g_0}[\hat{f}-f']+  O\left(\epsilon_n(\delta)^2 + \alpha_n(\delta)+ d(\hat{g},g_0)^{\frac{4}{1+r}}\right)
\end{align}
Again, we may assume $\frac{4\beta_1}{\lambda}>1$. Finally, we get:
\begin{align}
    P\ell_{\hat{f},g_0} -  P\ell_{f',g_0} \leq O\left(\epsilon_n(\delta)^2 + \alpha_n(\delta)+ d(\hat{g},g_0)^{\frac{4}{1+r}}\right)
\end{align}

\end{proof}

\subsection{Slow Rates}
\subsubsection{With Universal Orthogonality}
\begin{proof}[Proof of Theorem \ref{thm:osl_slow}]
The proof follows the proof of Theorem 2 in \cite{foster2023orthogonal}, we reproduce it here for completeness.
\begin{align}
    P\ell_{\hat{f},g_0}-P\ell_{f',g_0} =~& (P\ell_{\hat{f},g_0}-P\ell_{\hat{f},\hat{g}}) +(P\ell_{f',\hat{g}}-P\ell_{f',g_0}) + (P\ell_{\hat{f},\hat{g}}-P\ell_{f',\hat{g}})\\
\end{align}
Taking a second order Taylor expansion yields that there exist $g,\bar{g} \in \text{Star}(G,g_0)$ such that:
\begin{align}
(P\ell_{\hat{f},g_0}-P\ell_{\hat{f},\hat{g}}) +(P\ell_{f',\hat{g}}-P\ell_{f',g_0}) =~& -D_{g}P\ell_{\hat{f},g_0}[\hat{g} - g_0] - \frac{1}{2}D_{g}^{2}P\ell_{\hat{f},g_0}[\hat{g}-g-0,\hat{g}-g_0]\\
+~& D_{g}P\ell_{f',g_0}[\hat{g} - g_0] + \frac{1}{2}D_{g}^{2}P\ell_{f',g_0}[\hat{g}-g-0,\hat{g}-g_0]\\
\leq~& -D_{g}P\ell_{\hat{f},g_0}[\hat{g} - g_0] + D_{g}P\ell_{f',g_0}[\hat{g} - g_0] + \beta d(\hat{g},g_0)^2 \tag{By Assumption \ref{assum:6}}
\end{align}
Next, we take another second-order Taylor expansion on the derivative terms:
\begin{align}
D_{g}P\ell_{\hat{f},g_0}[\hat{g} - g_0] - D_{g}P\ell_{f',g_0}[\hat{g} - g_0] =~& D_{f}D_{g}P\ell_{f',g_0}[\hat{g}-g_0,\hat{f}-f'] + \frac{1}{2}D_{f}^{2}D_{g}P\ell_{\bar{f},g_0}[\hat{g}-g_0,\hat{f}-f',\hat{f}-f'] \\
=~& \frac{1}{2}D_{f}^{2}D_{g}P\ell_{\bar{f},g_0}[\hat{g}-g_0,\hat{f}-f',\hat{f}-f'] \tag{By Assumption \ref{assum:5}}
\end{align} for some $\bar{f}\in\text{Star}(F,f')$.
Furthermore, we can rewrite:
\begin{align}
    D_{f}^{2}~&D_{g}P\ell_{\bar{f},g_0}[\hat{g}-g_0,\hat{f}-f',\hat{f}-f'] \\=~& \lim_{t \rightarrow0}\frac{D_{f}D_{g}P\ell_{\bar{f}+t(\hat{f}-f'),g_0}[\hat{g}-g_0,\hat{f}-f'] - D_{f}D_{g}P\ell_{\bar{f},g_0}[\hat{g}-g_0,\hat{f}-f']}{t}
\end{align}

Since $\bar{f}+t(\hat{f}-f') \in \text{Star}(F,f') + \text{Star}(F-f',0)$ for all $t\in [0,1]$, Assumption \ref{assum:5} implies that both terms in the numerator are zero. Thus, we get:
\begin{align}
    D_{g}P\ell_{\hat{f},g_0}[\hat{g} - g_0] = D_{g}P\ell_{f',g_0}[\hat{g} - g_0]
\end{align}
Combining, we arrive at the desired result:
\begin{align}
    P\ell_{\hat{f},g_0}(Z) - P\ell_{f',g_0}(Z) \leq O\left(  P\ell_{\hat{f},\hat{g}}(Z) - P\ell_{f',\hat{g}}(Z) + d(g, g_0)^2\right)
\end{align}
\end{proof}

\subsubsection{Without Universal Orthogonality}
\begin{proof}[Proof of Theorem \ref{thm:osl_slow}]
\begin{align}
    P\ell_{\hat{f},g_0}-P\ell_{f',g_0} =~& (P\ell_{\hat{f},g_0}-P\ell_{f',g_0}) - (P\ell_{\hat{f},\hat{g}}-P\ell_{f',\hat{g}}) + (P\ell_{\hat{f},\hat{g}}-P\ell_{f',\hat{g}})\\
    =~& D_f P\ell(\bar{f}, g_0)[\hat{f}-f'] - D_f P\ell(\bar{f}, \hat{g})[\hat{f}-f'] + P\ell_{\hat{f},\hat{g}}-P\ell_{f',\hat{g}})\tag{By first order Taylor expansion}\\
    \leq~& \beta_4 d(\hat{g},g_0)^2 + (P\ell_{\hat{f},\hat{g}}-P\ell_{f',\hat{g}}) \tag{By Assumption \ref{assum:for_square_loss_slow}}\\
     =~& O\left(P\ell_{\hat{f},\hat{g}}-P\ell_{f',\hat{g}} + d(\hat{g},g_0)^2\right)
\end{align}
\end{proof}

\end{document}